%% file: main.tex
\documentclass[english,10pt]{article}
\include{macros}

\begin{document}

\title{Efficient Clustering for Stretched Mixtures: Landscape and Optimality
}
\author{Kaizheng Wang\thanks{Department of Industrial Engineering and Operations Research, Columbia University. Email: \texttt{kaizheng.wang@columbia.edu}} \and Yuling Yan\thanks{Department of Operations Research and Financial Engineering, Princeton University. Email: \texttt{yulingy@princeton.edu}}
\and Mateo D\'iaz\thanks{Center for Applied Mathematics, Cornell University. Email: \texttt{md825@cornell.edu}}}
\date{December 2020}

\maketitle

\begin{abstract}
This paper considers a canonical clustering problem where one receives unlabeled samples drawn from a balanced mixture of two elliptical distributions and aims for a classifier to estimate the labels. Many popular methods including PCA and k-means require individual components of the mixture to be somewhat spherical, and perform poorly when they are stretched. To overcome this issue, we propose a non-convex program seeking for an affine transform to turn the data into a one-dimensional point cloud concentrating around $-1$ and $1$, after which clustering becomes easy. Our theoretical contributions are two-fold: (1) we show that the non-convex loss function exhibits desirable geometric properties when the sample size exceeds some constant multiple of the dimension, and (2) we leverage this to prove that an efficient first-order algorithm achieves near-optimal statistical precision without good initialization. We also propose a general methodology for clustering with flexible choices of feature transforms and loss objectives.
\end{abstract}


\input{intro}

\input{problem_setup}

\input{main_results}

\input{numerical_experiments}

\input{discussion}

\section*{Acknowledgements}

We thank Philippe Rigollet and Damek Davis for insightful and stimulating discussions. Kaizheng Wang acknowledges support from the Harold W. Dodds Fellowship at Princeton University where part of the work was done. Yuling Yan is supported in part by the AFOSR grant FA9550-19-1-0030. Mateo D\'iaz would like to thank his advisor, Damek Davis, for research funding during the completion of this work.


\appendix

\input{appendix_numerical_experiments.tex}

\input{appendix_proof_outlines.tex}
\input{appendix_general_landscape_population.tex}

\input{appendix_general_landscape_empirical.tex}

\input{appendix_misclassification.tex}
\input{appendix_technical.tex}

{
\bibliographystyle{ims}
\bibliography{bib}
}

\end{document}

%% file: intro.tex
\section{Introduction} \label{sec:intro} 

Clustering is a fundamental problem in data science, especially in the early stages of knowledge discovery. Its wide applications include genomics~\citep{eisen1998cluster,remm2001automatic}, imaging~\citep{filipovych2011semi}, linguistics~\citep{di2013clustering}, networks~\citep{AGl05}, and finance~\citep{arnott1980cluster,ZBJ20}, to name a few. They have motivated numerous characterizations for ``clusters'' and associated learning procedures.

In this paper, we consider a binary clustering problem where the data come from a mixture of two elliptical distributions. 
Suppose that we observe i.i.d.~samples $\{\bX_i\}_{i=1}^n \subseteq \RR^d$ from the latent variable model
\begin{align}
\bX_i = \bmu_0 + \bmu Y_i + \bSigma^{1/2} \bZ_i , \qquad i \in [n].
\label{eqn-location-scale-intro}
\end{align}
Here $\bmu_0,~\bmu \in \RR^d$ and $\bSigma \succ 0$ are deterministic; $Y_i \in \{ \pm 1 \}$ and $\bZ_i \in \RR^d$ are independent random quantities; $\PP (Y_i = -1) = \PP (Y_i=1) = 1/2$, and $\bZ_i$ is an isotropic random vector whose distribution is spherically symmetric with respect to the origin. 
$\bX_i$ is elliptically distributed \citep{Fan18} given $Y_i$.
The goal of clustering is to estimate $\{Y_i\}_{i=1}^n$ from $\{\bX_i\}_{i=1}^n$. Moreover, it is desirable to build a classifier with straightforward out-of-sample extension that predicts labels for future samples.

As a warm-up example, assume for simplicity that $\bZ_i$ has density and $\bmu_0 = \mathbf{0}$. The Bayes-optimal classifier is 
\begin{equation*}
	\varphi_{\bbeta^\star} (\bx) = \sgn ( \bbeta^{\star\top} \bx ) =
	\begin{cases}
		1 & \text{if } \bbeta^{\star\top} \bx \geq 0 \\
		-1 & \text{otherwise}
	\end{cases},
\end{equation*}
with any $\bbeta^\star \propto \bSigma^{-1} \bmu$. A natural strategy for clustering is to learn a linear classifier $\varphi_{\bbeta} (\bx) = \sgn ( \bbeta^{\top} \bx ) $ with discriminative coefficients $\bbeta \in \RR^d$ estimated from the samples. Note that
\begin{align*}
\bbeta^{\top} \bX_i = (\bbeta^{\top} \bmu) Y_i + \bbeta^{\top} \bSigma^{1/2} \bZ_i
\overset{d}{=} (\bbeta^{\top} \bmu) Y_i + \sqrt{ \bbeta^{\top} \bSigma \bbeta } Z_i,
\end{align*}
where $Z_i = \be_1^{\top} \bZ_i$ is the first coordinate of $\bZ_i$. The transformed data $\{ \bbeta^{\top} \bX_i \}_{i=1}^n$ are noisy observations of scaled labels $\{ (\bbeta^{\top} \bmu) Y_i  \}_{i=1}^n$. A discriminative feature mapping $\bx \mapsto \bbeta^{\top} \bx$ results in high signal-to-noise ratio $(\bbeta^{\top} \bmu)^2 / \bbeta^{\top} \bSigma \bbeta $, turning the data into two well-separated clusters in $\RR$.

When the clusters are almost spherical ($\bSigma \approx \bI$) or far apart ($\| \bmu \|_2^2 \gg \| \bSigma \|_2$)
, the mean vector $\bmu$ has reasonable discriminative power and the leading eigenvector of the overall covariance matrix $\bmu \bmu^{\top} + \bSigma$ roughly points that direction. This helps develop and analyze various spectral methods \citep{vempala2004spectral,
	ndaoud2018sharp
} based on Principal Component Analysis (PCA). $k$-means \citep{lu2016statistical} and its semidefinite relaxation \citep{mixon2017clustering, royer2017adaptive, fei2018hidden, giraud2018partial, chen2018hanson} are also closely related. As they are built upon the Euclidean distance, a key assumption is the existence of well-separated balls each containing the bulk of one cluster. Existing works typically require $\| \bmu \|_2^2 / \| \bSigma \|_2$ to be large under models like \eqref{eqn-location-scale-intro}. Yet, the separation is better measured by $\bmu^{\top} \bSigma^{-1} \bmu$, which always dominates $\| \bmu \|_2^2 / \| \bSigma \|_2$. Those methods may fail when the clusters are separated but ``stretched''. As a toy example, consider a Gaussian mixture $\frac{1}{2} N(\bmu, \bSigma) + \frac{1}{2} N(-\bmu, \bSigma)$ in $\RR^2$ where $\bmu = (1,0)^{\top}$ 
and the covariance matrix $\bSigma = \diag(0.1, 10)$ is diagonal. Then the distribution consists of two separated but stretched ellipses. PCA returns the direction $(0,1)^{\top}$ that maximizes the variance but is unable to tell the clusters apart. 

To get high discriminative power under general conditions, we search for $\bbeta$ that makes $\{ \bbeta^{\top} \bX_i \}_{i=1}^n$ concentrate around the label set $\{ \pm 1 \}$, through the following optimization problem:
\begin{equation}\label{eq:1}
\min_{\bbeta \in \RR^d }\sum_{i = 1}^n f(\bbeta^{\top} \bX_i).
\end{equation}
Here $f:\RR \to \RR$ attains its minimum at $\pm 1$, e.g. $f(x) = (x^2 - 1)^2$. We name this method as ``Clustering via Uncoupled REgression'', or \textsc{CURE} for short. Here $f$ penalizes the discrepancy between predictions $\{ \bbeta^{\top} \bX_i \}_{i=1}^n$ and labels $\{ Y_i \}_{i=1}^n$. In the unsupervised setting, we have no access to the one-to-one correspondence but can still enforce proximity on the distribution level, i.e.
\begin{align}
\frac{1}{n} \sum_{i=1}^{n} \delta_{\bbeta^{\top} \bX_i} \approx \frac{1}{2} \delta_{-1} + \frac{1}{2} \delta_{1}.
\label{eqn-ur-1}
\end{align}
A good approximate solution to (\ref{eq:1}) leads to $|\bbeta^{\top} \bX_i| \approx 1$. That is, the transformed data form two clusters around $\pm 1 $. The symmetry of the mixture distribution automatically ensures balance between the clusters. Thus (\ref{eq:1}) is an uncoupled regression problem based on (\ref{eqn-ur-1}). 
Above we focus on the centered case $(\bmu_0 = \mathbf{0})$ merely to illustrate main ideas. Our general methodology
\begin{align}
\min_{\alpha \in \RR,~ \bbeta \in \RR^d }
\bigg\{ 
\frac{1}{n}
\sum_{i = 1}^n f( \alpha + \bbeta^{\top} \bX_i)
+ \frac{1}{2} (\alpha + \bbeta^{\top} \hat\bmu_0)^2
\bigg\},
\label{eqn-cure-0}
\end{align}
where $\hat\bmu_0 = \frac{1}{n} \sum_{i = 1}^n \bX_i $, deals with arbitrary $\bmu_0$ by incorporating an intercept term $\alpha$.



 \paragraph*{Main contributions.} We propose a clustering method through (\ref{eqn-cure-0}) and study it under the model (\ref{eqn-location-scale-intro}) without requiring the clusters to be spherical. Under mild assumptions, we prove that an efficient algorithm achieves near-optimal statistical precision even in the absence of a good initialization. 
  \begin{itemize} 
  \item {(\bf Loss function design)} We construct an appropriate loss function $f$ by clipping the growth of the quartic function $(x^2 - 1)^2 / 4$ outside some interval centered at $0$. As a result, $f$ has two ``valleys'' at $\pm 1$ and does not grow too fast, which is beneficial to statistical analysis and optimization.
    
  \item ({\bf Landscape analysis}) We characterize the geometry of the empirical loss function when $n / d$ exceeds some constant. In particular, all second-order stationary points, where the smallest eigenvalues of Hessians are not significantly negative, are nearly optimal in the statistical sense.
  
 \item ({\bf Efficient algorithm with near-optimal statistical property}) We show that with high probability, a perturbed version of gradient descent algorithm starting from $\mathbf{0}$ yields a solution with near-optimal statistical property after $\tilde{O}( n/ d + d^2 / n )$ iterations (up to polylogarithmic factors).
  \end{itemize}

The formulation (\ref{eqn-cure-0}) is uncoupled linear regression for binary clustering. Beyond that, we introduce a unified framework which learns feature transforms to identify clusters with possibly non-convex shapes. That provides a principled way of designing flexible unsupervised learning algorithms.
 
We introduce the model and methodology in Section~\ref{sec:problem-setup}, conduct theoretical analysis in Section~\ref{sec:main_results}, present numerical results in Section~\ref{sec:numerical}, and finally conclude the paper with a discussion in Section~\ref{sec:discussion}.
 
  \paragraph{Related work.} Methodologies for clustering can be roughly categorized as generative and discriminative ones. Generative approaches fit mixture models for the joint distribution of features $\bX$ and label $Y$ to make predictions  \citep{
  	MVa10, kannan2005spectral, anandkumar2014tensor}. Their success usually hinges on well-specified models and precise estimation of parameters. Since clustering is based on the conditional distribution of $Y$ given $\bX$, it only involves certain functional of parameters. Generative approaches often have high overhead in terms of sample size and running time.
 On the other hand, discriminative approaches directly aim for predictive classifiers. A common strategy is to learn a transform to turn the data into a low-dimensional point cloud that facilitates clustering. Statistical analysis of mixture models lead to information-based methods \citep{BHM92,
  KPG10}, analogous to the logistic regression 
  for supervised classification. 
  Geometry-based methods uncover latent structures in an intuitive way, similar to the support vector machine. 
  Our method \textsc{CURE} belongs to this family. Other examples include projection pursuit \citep{FTu74, PPr01}, margin maximization \citep{BHS01, XNL05
  }, discriminative $k$-means \citep{YZW08, BHa08}, graph cut optimization by spectral methods \citep{SMa00, NJW02} and semidefinite programming \citep{WSa06}, correlation clustering \cite{BGL20,JMW20}.  
  Discriminative methods are easily integrated with modern tools such as deep neural networks \citep{Spr15, XGF16
  }. The list above is far from exhaustive. 
  
%
The formulation \eqref{eqn-cure-0} is invariant under invertible affine transforms of data and thus tackles stretched mixtures which are catastrophic for many existing approaches.
A recent paper \cite{KJS19} uses random projections to tackle such problem but requires the separation between two clusters to grow at the order of $\sqrt{d}$, where $d$ is the dimension. There have been provable algorithms dealing with general models with multiple classes and minimal separation conditions \citep{BVe08, KMV10, BSi15}. However, their running time and sample complexity are large polynomials in the dimension and desired precision. In the class of two-component mixtures we consider, \textsc{CURE} has near-optimal (linear) sample complexity and runs fast in practice. 
Another relevant area of study is clustering under sparse mixture models \citep{azizyan2015efficient,VAr17}, where additional structures help handle non-spherical clusters efficiently.

The vanilla version of \textsc{CURE} in \eqref{eq:1} is closely related to the Projection Pursuit (PP) \citep{FTu74} and Independent Component Analysis (ICA) \citep{HOj00}.
PP and ICA find the most nontrivial direction by maximizing the deviation of the projected data from some null distribution (e.g. Gaussian). Their objective functions are designed using key features of that. Notably, \cite{PPr01} propose clustering algorithms based on extreme projections that maximize and minimize the kurtosis; \cite{VAr17} use the first absolute moment and skewness to construct objective functions in pursuit of projections for clustering. On the contrary, \textsc{CURE} stems from uncoupled regression and minimizes the discrepancy between the projected data and some target distribution. This makes it generalizable beyond linear feature transforms with flexible choices of objective functions. Moreover, \textsc{CURE} has nice computational guarantees while only a few algorithms for PP and ICA do.
The formulation \eqref{eq:1} with double-well loss $f$ also appears in the real version of Phase Retrieval (PR) \citep{CLS15} for recovering a signal $\bbeta$ from noisy quadratic measurements $Y_i \approx (\bX_i^{\top} \bbeta )^2$. In both \textsc{CURE} and PR, one observes the magnitudes of labels/outputs without sign information. However, algorithmic study of PR usually require $\{ \bX_i \}_{i=1}^n$ to be isotropic Gaussian; most efficient algorithms need good initializations by spectral methods. Those cannot be easily adapted to clustering. Our analysis of \textsc{CURE} could provide a new way of studying PR under more general conditions.



  \paragraph{Notation.} Let $[n] = \{ 1, 2, \cdots, n \}$. Denote by $| \cdot |$ the absolute value of a real number or cardinality of a set. For real numbers $a$ and $b$, let $a \wedge b = \min \{ a, b \}$ and $a \vee b = \max \{ a, b \}$. For nonnegative sequences $\{ a_n \}_{n=1}^{\infty}$ and $\{ b_n \}_{n=1}^{\infty}$, we write $a_n \lesssim b_n$ or $a_n = O(b_n)$ if there exists a positive constant $C$ such that $a_n \leq C b_n$. In addition, we write $a_n = \tilde O(b_n)$ if $a_n = O( b_n )$ holds up to some logarithmic factor; $a_n \asymp b_n$ if $a_n \lesssim b_n$ and $b_n \lesssim a_n$. We let $\mathbf{1}_{ S }$ be the indicator function of a set $S$. We equip $\RR^d$ with the inner product $\dotp{\bx}{\by} =\bx^\top \by$, Euclidean norm $\| \bx \|_2 =  \sqrt{\dotp{\bx}{\bx}}$ and canonical bases $\{ \be_j \}_{j=1}^d$. Let $\SSS^{d-1} = \{ \bx \in \RR^d:~ \| \bx \|_2 = 1 \}$, $B(\bm{x},r) = \{ \by \in \RR^d:~\| \by - \bx \|_2 \leq r \}$, and $\dist(\bx,S) = \inf_{\by \in S} \|\bx - \by\|_2$ for $S \subseteq \RR^d$. For a matrix $\bA $, we define its spectral norm $\| \bA \|_2 = \sup_{\|\bx\|_2 = 1 } \| \bA \bx \|_2$. For a symmetric matrix $\bA$, we use $\lambda_{\max}(\bA)$ and $\lambda_{\min}(\bA)$ to represent its largest and smallest eigenvalues, respectively. For a positive definite matrix $\bA \succ 0$, let $\| \bx\|_{\bA} = \sqrt{ \bx^\top \bA \bx}$. 
Denote by $\delta_{\bx}$ the point mass at $\bx$. Define $\| X \|_{\psi_2} = \sup_{p \geq 1} p^{-1/2} \EE^{1/p} |X|^p$ for random variable $X$ and $\|\bX\|_{\psi_2} = \sup_{\| \bu \|_2 = 1} \|\dotp{\bu}{\bX}\|_{\psi_2} $ for random vector $\bX$.


%% file: problem_setup.tex
\section{Problem setup} \label{sec:problem-setup}

\subsection{Elliptical mixture model}

\begin{model}\label{model:1}  Let $\bX \in \RR^d$ be a random vector with the decomposition
\begin{align*}
\bX = \bmu_0 + \bmu Y + \bSigma^{1/2} \bZ.
\end{align*}
Here $\bmu_0, \bmu \in \RR^d$ and $\bSigma \succ 0$ are deterministic; $Y \in \{ \pm 1 \}$ and $\bZ \in \RR^d$ are random and independent. Let $Z = \be_1^{\top} \bZ$, $\rho$ be the distribution of $\bX$ and $\{\bX_i\}_{i=1}^n $ be i.i.d.~samples from $\rho$.
\begin{itemize}
\item {\bf (Balanced classes)} $\PP(Y=-1)=\PP(Y=1)=1/2$;
\item {\bf (Elliptical sub-Gaussian noise)} $\bZ$ is sub-Gaussian with $\Vert \bZ\Vert_{\psi_2}$ bounded by some constant $M$, $\EE \bZ = \mathbf{0}$ and $\EE ( \bZ \bZ^{\top} ) = \bI_d$; its distribution is spherically symmetric with respect to $\mathbf{0}$;
\item {\bf (Leptokurtic distribution)} $\EE Z^4 - 3 > \kappa_0$ holds for some constant $\kappa_0 > 0$;
\item {\bf (Regularity)} $\| \bmu_0 \|_2$, $\| \bmu \|_2$, $\lambda_{\max} (\bSigma)$ and $\lambda_{\min} (\bSigma)$ are bounded away from 0 and $\infty$ by constants.
\end{itemize}
  \end{model}

We aim to build a classifier $\RR^{d} \to \{ \pm 1 \}$ based solely on the samples $\{ \bX_i \}_{i=1}^n$ from a mixture of two elliptical distributions. For simplicity, we assume that the two classes are balanced and focus on the well-conditioned case where the signal strength and the noise level are of constant order. This is already general enough to include stretched clusters incapacitating many popular methods including PCA, $k$-means and semi-definite relaxations \citep{BVe08}. One may wonder whether it is possible to transform the data into what they can handle. While multiplication by $\bSigma^{-1/2}$ yields spherical clusters, precise estimation of $\bSigma^{-1/2}$ or $\bSigma$ is no easy task under the mixture model. Dealing with those $d\times d$ matrices causes overhead expenses in computation and storage. 
The assumption on positive excess kurtosis prevents the loss function from having undesirable degenerate saddle points and facilitates the proof of algorithmic convergence. It rules out distributions whose kurtoses do not exceed that of the normal distribution, and it is not clear whether there exists an easy fix for that.
The last assumption in Model \ref{model:1} makes the loss landscape regular, helps avoid undesirable technicalities, and is commonly adopted in the study of parameter estimation in mixture models. The Bayes optimal classification error is of constant order, and we want to achieve low excess risk. 


\subsection{Clustering via Uncoupled Regression}\label{sec:clust-via-unco}



Under Model \ref{model:1}, the Bayes optimal classifier for predicting $Y$ given $\bX$ is
\begin{align*}
	\hat{Y}^{\mathrm{Bayes}}(\bX)=\sgn \big( \alpha^{\mathrm{Bayes}}+\bbeta^{\mathrm{Bayes}\top}\bm{X} \big),
\end{align*}
where $\big(\alpha^{\mathrm{Bayes}},\bbeta^{\mathrm{Bayes}}\big) = (- \bmu_0^\top\bSigma^{-1}\bmu, \bSigma^{-1}\bmu)$. On the other hand, it is easily seen that the following (population-level) least squares problem $ \EE [ (\alpha + \bbeta^{\top} \bX) - Y ]^2$
has a unique solution $(\alpha^{\mathrm{LR}} , \bbeta^{\mathrm{LR}}) = (- c \bmu_0^\top\bSigma^{-1}\bmu, c \bSigma^{-1}\bmu)$ for some $c>0$. For the supervised classification problem where we observe $\{ (\bX_i , Y_i ) \}_{i=1}^n$, the optimal feature transform can be estimated via linear regression
\begin{align}
\frac{1}{n} \sum_{i=1}^{n} [ (\alpha + \bbeta^{\top} \bX_i) - Y_i ]^2 .
\label{eqn-lda-reg}
\end{align}
This is closely related to Fisher's Linear Discriminant Analysis \citep{FHT01}. 

In the unsupervised clustering problem, we no longer observe individual labels $\{ Y_i \}_{i=1}^n$ associated with $\{ \bX_i \}_{i=1}^n$ but have population statistics of labels, as the classes are balanced. While \eqref{eqn-lda-reg} directly forces $\alpha + \bbeta^{\top}\bX_i \approx Y_i$ thanks to supervision, here we relax such proximity to the population level:
\begin{align}
\frac{1}{n} \sum_{i=1}^{n} \delta_{ \alpha + \bbeta^{\top} \bX_i } \approx \frac{1}{2} \delta_{-1} + \frac{1}{2} \delta_{1}.
\label{eqn-ur-2}
\end{align}
Thus the regression should be conducted in an uncoupled manner using marginal information about $\bX$ and $Y$. We seek for an affine transformation $\bx \mapsto \alpha + \bbeta^{\top} \bx$ to turn the samples $\{ \bX_i \}_{i=1}^n$ into two balanced clusters around $\pm 1$, after which $\sgn( \alpha + \bbeta^{\top} \bX )$ predicts $Y$ up to a global sign flip. It is also supported by the geometric intuition in Section \ref{sec:intro} based on projections of the mixture distribution.

Clustering via Uncoupled REgression (\textsc{CURE}) is formulated as an optimization problem:
\begin{align}
	\min_{\alpha \in \RR,~ \bbeta \in \RR^d }
	\bigg\{ 
	\frac{1}{n} \sum_{i = 1}^n f( \alpha + \bbeta^{\top} \bX_i)
	+ \frac{1}{2} (\alpha + \bbeta^{\top} \hat\bmu_0)^2
	\bigg\},
\label{eqn-cure}
\end{align}
where $\hat\bmu_0 = \frac{1}{n} \sum_{i = 1}^n \bX_i $. $f$ attains its minimum at $\pm 1$. Minimizing $\frac{1}{n} \sum_{i = 1}^n f( \alpha + \bbeta^{\top} \bX_i)$ makes the transformed data $\{ \alpha + \bbeta^{\top} \bX_i \}_{i=1}^n$ concentrate around $\{ \pm 1 \}$. 
However, there are always two trivial minimizers $(\alpha , \bbeta ) = ( \pm 1 , \mathbf{0} )$, each of which maps the entire dataset to a single point. What we want are two balanced clusters around $-1$ and $1$. The centered case ($\bmu_0 = \mathbf{0}$) discussed in Section \ref{sec:intro} does not have such trouble as $\alpha$ is set to be $0$ and the symmetry of the mixture automatically balance the two clusters. For the general case, we introduce a penalty term $(\alpha + \bbeta^{\top} \hat\bmu_0)^2 / 2$ in \eqref{eqn-cure} to drive the center of the transformed data towards $0$. The idea comes from moment-matching and is similar to that in \cite{FPB17}. If $ \frac{1}{n} \sum_{i=1}^{n} f ( \alpha + \bbeta^{\top} \bX_i )$ is small, then $| \alpha + \bbeta^{\top} \bX_i | \approx 1$ and
\begin{align*}
\frac{1}{n} \sum_{i = 1}^n \delta_{\alpha + \bbeta^{\top} \bX_i} \approx \frac{ | \{ i:~ \alpha + \bbeta^{\top} \bX_i \geq 0 \} | }{ n } \delta_{1} + \frac{ | \{ i:~ \alpha + \bbeta^{\top} \bX_i < 0 \} | }{ n } \delta_{-1}.
\end{align*}
Then, in order to get \eqref{eqn-ur-2}, we simply match the expectations on both sides. This gives rise to the quadratic penalty term in \eqref{eqn-cure}. The same idea generalizes beyond the balanced case. When the two classes $1$ and $-1$ have probabilities $p$ and $(1-p)$, we can match the mean of $\{ \alpha + \bbeta^{\top} \bX_i \}_{i=1}^n$ with that of a new target distribution $p \delta_{1} + (1-p) \delta_{-1}$, and change the quadratic penalty to $[ (\alpha + \bbeta^{\top} \hat\bmu_0) - (2p-1) ]^2$.
When $p$ is unknown, (\ref{eqn-cure}) can always be a surrogate as it seeks for two clusters around $\pm 1$ and uses the quadratic penalty to prevent any of them from being vanishingly small.

The function $f$ in \eqref{eqn-cure} requires careful design. To facilitate statistical and algorithmic analysis, we want $f$ to be twice continuously differentiable and grow slowly. That makes the empirical loss smooth and concentrate well around its population counterpart. In addition, the coercivity of $f$, i.e. $\lim_{|x| \to \infty} f(x) = + \infty$, confines all minimizers within some ball of moderate size. Similar to the construction of Huber loss \citep{Hub64}, we start from $h(x) = (x^2 - 1)^2 / 4$, keep its two valleys around $\pm 1$, clip its growth using linear functions and interpolate in between using cubic splines: 
\begin{align}
	f(x) = \begin{cases}
		h(x), & |x| \leq a \\
		h(a) + h'(a) (|x|-a) + \frac{h''(a)}{2} (|x|-a)^2 - \frac{h''(a)}{6(b-a)} (|x|-a)^3, & a < |x| \leq b \\
		f(b) + [ h'(a) + \frac{b-a}{2} h''(a) ] (|x| - b), & |x| > b
	\end{cases}.
	\label{eqn-test-function}
\end{align}
Here $b>a>1$ are constants to be determined later. $f$ is clearly not convex, and neither is the loss function in (\ref{eqn-cure}). Yet we can find a good approximate solution efficiently by taking advantage of statistical assumptions and recent advancements in non-convex optimization \citep{jin2017escape}.

\subsection{Generalization}\label{sec-cure-general}

The aforementioned procedure seeks for a one-dimensional embedding of the data that facilitates clustering. It searches for the best affine function such that the transformed data look like a two-point distribution. The idea of uncoupled linear regression can be easily generalized to any suitable target probability distribution $\nu$ over a space $\cY$, class of feature transforms $\cF$ from the original space $\cX$ to $\cY$, discrepancy measure $D$ that quantifies the difference between the transformed data distribution and $\nu$, and classification rule $g:~\cY \to [K]$. \textsc{CURE} for Model~\ref{model:1} above uses $\cX = \RR^d$, $\cY = \RR$, $\nu = \frac{1}{2} \delta_{-1} + \frac{1}{2} \delta_{1}$, $\cF = \{ \bx \mapsto \alpha + \bbeta^{\top} \bx :~ \alpha \in \RR,~ \bbeta \in \RR^d \}$, $g(y) = \sgn(y)$ and
\begin{align}
D ( \mu, \nu ) = | \EE_{X \sim \mu } f(X) - \EE_{X \sim \nu } f(X) | + \frac{1}{2} | \EE_{X \sim \mu } X - \EE_{X \sim \nu } X |^2
\label{eqn-D}
\end{align}
for any probability distribution $\mu$ over $\RR$. Here we briefly show why (\ref{eqn-D}) is true. Fix any $f :~\bx\mapsto \alpha + \bbeta^{\top} \bx$ in $\cF$ and let $\mu = \frac{1}{n} \sum_{i=1}^{n} \delta_{\alpha + \bbeta^{\top} \bX_i}$ be the transformed data distribution. From $f(-1) = f(1) = 0$ and $\EE_{X \sim \nu } X = 0$ we see
\begin{align*}
& | \EE_{X \sim \mu } f(X) - \EE_{X \sim \nu } f(X) | =  \EE_{X \sim \mu } f(X) = \frac{1}{n} \sum_{i=1}^{n} f(\alpha + \bbeta^{\top} \bX_i) , \\
& | \EE_{X \sim \mu } X - \EE_{X \sim \nu } X |^2 = \bigg( \frac{1}{n} \sum_{i=1}^{n} (\alpha + \bbeta^{\top} \bX_i) \bigg)^2 = (\alpha + \bbeta^{\top} \hat\bmu_0), \\
& D(\mu, \nu) = \frac{1}{n} \sum_{i=1}^{n} f(\alpha + \bbeta^{\top} \bX_i) + \frac{1}{2}  (\alpha + \bbeta^{\top} \hat\bmu_0).
\end{align*}

On top of that, we propose a general framework for clustering (also named as \textsc{CURE}) and describe it at a high level of abstraction in Algorithm \ref{alg-cure}. Here $\hat{\rho}_n = \frac{1}{n} \sum_{i=1}^{n} \delta_{\bX_i}$ is the empirical distribution of data and $\varphi_{\#}\hat{\rho}_n = \frac{1}{n} \sum_{i=1}^{n} \delta_{\varphi(\bX_i)}$ is the push-forward distribution.
The general version of \textsc{CURE} is a flexible framework for clustering based on uncoupled regression \citep{PWe19}. For instance, we may set $\cY = \RR^K$ and $\nu = \frac{1}{K} \sum_{k=1}^{n} \delta_{\be_k}$ when there are $K$ clusters; choose $\cF$ to be the family of convolutional neural networks for image clustering; let $D$ be the Wasserstein distance or some divergence. \textsc{CURE} is easily integrated with other tools, see Section \ref{appendix:numerical-2} in the supplementary material.

	\begin{algorithm}
	{\bf Input:} Data $\{ \bX_i \}_{i=1}^n$ in a feature space $\cX$, embedding space $\cY$, target distribution $\nu$ over $\cY$, discrepancy measure $D$, function class $\cF$, classification rule $g$.\\
	{\bf Embedding:} find an approximation solution $\hat\varphi$ to $\min_{ \varphi \in \cF} D ( \varphi_{\#} \hat{\rho}_n, \nu )$.\\
	{\bf Output:} $\hat{Y}_i = g [ \hat\varphi(\bX_i) ]$ for $i \in [n]$.
	\caption{Clustering via Uncoupled REgression (meta-algorithm)}
	\label{alg-cure}
\end{algorithm}

%% file: main_results.tex
\section{Theoretical analysis}\label{sec:main_results}

\subsection{Main results}

	\begin{algorithm}
		{\bf Initialize} $\bgamma^0=\bm{0}$. \\
	{\bf For $t = 0,1,\ldots$ do}\\
	\hspace*{.5cm}{\bf If} perturbation condition holds: 
	\hspace*{1cm}Perturb $\displaystyle \bgamma^t \leftarrow \bgamma^t + \bxi^t$ with $\bxi^t\sim \cU (B(\bm{0},r))$ \\
	\hspace*{.5cm}{\bf If} termination condition holds: 
	\hspace*{1cm} {\bf Return} $\bgamma^{t}$ \\	
	\hspace*{.5cm}{\bf Update} $\bgamma^{t+1} \leftarrow \bgamma^{t} - \eta\nabla \hat{L}_1 (\bgamma^t)$.
	\caption{Perturbed gradient descent}
	\label{alg:PGD-simple}
\end{algorithm}

Let $\hat L_1 (\alpha, \bbeta)$ denote the objective function of \textsc{CURE} in (\ref{eqn-cure}). Our main result (Theorem \ref{thm:main}) shows that with high probability, a perturbed version of gradient descent (Algorithm \ref{alg:PGD-simple}) applied to $\hat L_1$ returns an approximate minimizer that is nearly optimal in the statistical sense, within a reasonable number of iterations. Here $\cU (B(\bm{0},r))$ refers to the uniform distribution over $B(\bm{0},r)$. We omit technical details of the algorithm and defer them to Appendix~\ref{sec:algorithm}, see Algorithm~\ref{alg:PGD} and Theorem~\ref{theorem:algorithmic} therein. For notational simplicity, we write $\bgamma = (\alpha , \bbeta) \in \RR \times \RR^d$ and $\bgamma^{\mathrm{Bayes}}=(\alpha^{\mathrm{Bayes}},\bbeta^{\mathrm{Bayes}})=(-\bmu^\top\bSigma^{-1}\bmu_0,\bSigma^{-1}\bmu)$. $\bgamma^{\mathrm{Bayes}}$ defines the Bayes-optimal classifier $\bx \mapsto \sgn ( \alpha^{\mathrm{Bayes}} + \bbeta^{\mathrm{Bayes}\top}  \bx )$ for Model~\ref{model:1}.

\begin{theorem}[Main result]\label{thm:main}
	Let $\bgamma_0, \bgamma_1,\cdots$ be the iterates of Algorithm \ref{alg:PGD-simple} starting from $\mathbf{0}$. Under Model~\ref{model:1} there exist constants $c, C, C_0,C_1,C_2>0$ independent of $n$ and $d$ such that if $n\geq C d$ and $b \geq 2 a \geq C_0$, then with probability at least $1-C_1 [ (d/n)^{C_2d} + e^{ -C_2n^{1/3} } + n^{-10} ]$, Algorithm~\ref{alg:PGD-simple} terminates within $\tilde{O}( n/d+d^2/n )$ iterations and the output $\hat{\bgamma}$ satisfies
	\[
\min_{s=\pm1}\bigl\Vert s \hat{\bm{\gamma}} -c\bgamma^{\mathrm{Bayes}}\bigr\Vert_2\lesssim \sqrt{\frac{d}{n} \log \left( \frac{n}{d} \right)} .
	\]
\end{theorem}


Up to a $\sqrt{ \log (n/d) }$ factor, this matches the optimal rate of convergence $O(\sqrt{d/n})$ for the supervised problem with $\{ Y_i \}_{i=1}^n$ observed, which is even easier than the current one. Theorem \ref{thm:main} asserts that we can achieve a near-optimal rate efficiently without good initialization, although the loss function is non-convex. The two terms $n/d$ and $d^2 / n$ in the iteration complexity have nice interpretations. When $n$ is large, we want a small computational error in order to achieve statistical optimality. The term $n/d$ reflects the cost for this. When $n$ is small, the empirical loss function does not concentrate well and is not smooth enough either. Hence we choose a conservative step-size and pay the corresponding price $d^2/n$. A byproduct of Theorem \ref{thm:main} is the following corollary which gives a tight bound for the excess risk. Here $\| g \|_{\infty} = \sup_{x \in \RR} |g(x)|$ for any $g:~\RR \to \RR$. The proof is deferred to Appendix \ref{appendix:misclassification}.
\begin{corollary}[Misclassification rate]\label{cor:misclassification}
Consider the settings in Theorem~\ref{thm:main} and suppose that $Z = \be_1^{\top} \bZ$ has density $p \in C^1(\mathbb{R})$ satisfying $\Vert p\Vert_\infty\leq C_3$ and $\Vert p^\prime\Vert_\infty\leq C_3$ for some constant $C_3>0$. For $\bgamma = (\alpha, \bbeta) \in \RR \times \RR^d$, define its misclassification rate (up to a global sign flip) as
\begin{align*}
\cR ( \bgamma ) = \min_{s=\pm 1}\mathbb{P}\left(s\sgn\big( \alpha + \bm{\beta}^\top\bm{X}\big)\neq Y\right).
\end{align*}
There exists a constant $C_4$ such that
	\[
	\PP \bigg(
	\cR ( \hat\bgamma ) \leq \cR ( \bgamma^{\mathrm{Bayes}}  ) +  \frac{C_4 d \log (n/d) }{n}  \bigg)
	\geq 1-C_1 [ (d/n)^{C_2d} + e^{-C_2n^{1/3}} + n^{-10} ].
	\]
\end{corollary}

\subsection{Sketch of proof}

The loss function $\hat{L}_1$ is non-convex in general. To find an approximate minimizer efficiently without good initialization, we need $\hat{L}_1$ to exhibit benign geometric properties that can be exploited by a simple algorithm. Our choice is the perturbed gradient descent algorithm in \cite{jin2017escape}, see Algorithm \ref{alg:PGD} in Appendix~\ref{sec:algorithm} for more details. Provided that the function is smooth enough, it provably converges to an approximate second-order stationary point where the norm of gradient is small and the Hessian matrix does not have any significantly negative eigenvalue. Then it boils down to landscape analysis of $\hat{L}_1$ with precise characterizations of approximate stationary points.
To begin with, define the population version of $\hat{L}_1$ as
\begin{equation*}
L_1 \left( \alpha, \bbeta \right) = \mathbb{E}_{\bm{X}\sim\rho}f(\alpha+\bm{\beta}^{\top}\bm{X})+\frac{1}{2}(\alpha+\bbeta^\top\bmu_0)^2.
\end{equation*}

\begin{prop}\label{prop-population-landscape}
There exist positive constants $c, \varepsilon,\delta, \eta$ and a set $S \subseteq \RR\times \RR^d$ such that
	\begin{enumerate}
		\item The only two local minima of $L_1$ are $\pm \bgamma^\star$ with $\bgamma^\star = - c \bgamma^{\mathrm{Bayes}}$;
		\item All the other first-order critical points (i.e. with zero gradient) are within $\delta$ distance to $S$;
		\item $\Vert\nabla L_1(\bm{\gamma})\Vert_{2}\geq\varepsilon$ if $\mathrm{dist}(\bm{\gamma}, \{ \pm \bgamma^\star \} \cup S)\geq\delta$;
		\item $\nabla^{2}L_1(\bm{\gamma})\succeq\eta\bm{I}$ if $\mathrm{dist}(\bm{\gamma}, \{ \pm \bgamma^\star \} )\leq\delta$,
		and $\lambda_{\min} [ \nabla^{2}L_1(\bm{\gamma}) ] \leq-\eta$ if
		$\mathrm{dist}(\bm{\gamma},S)\leq\delta$.
	\end{enumerate}
\end{prop}

Proposition \ref{prop-population-landscape} shows that all of the approximate second-order critical points of $L_1$ are close to that corresponding to the Bayes-optimal classifier. 
Then we will prove similar results for the empirical loss $\hat{L}_1$ using concentration inequalities, which leads to the following proposition translating approximate second-order stationarity to estimation error.

\begin{prop}\label{prop-conversion}
There exists a constant $C$ such that the followings happen with high probability: for any $\bm{\gamma} \in \RR\times \RR^d $ satisfying
	$\Vert\nabla \hat{L}_1(\bm{\gamma} )\Vert_{2}\leq\varepsilon / 2$ and $\lambda_{\min}[\nabla^{2}\hat{L}_1(\bm{\gamma} )]>-\eta / 2$,
	\[
	\min_{s=\pm1}\left\Vert s\bm{\gamma} -\bm{\gamma}^{\star}\right\Vert _{2}\leq
	C \bigg(
	\big\Vert \nabla \hat{L}_1\left(\bm{\gamma} \right)\big\Vert _{2} + \sqrt{\frac{d}{n}\log\left( \frac{n}{d} \right)}
	\bigg).
	\]
\end{prop}

To achieve near-optimal statistical error (up to a $\sqrt{\log(n/d)}$ factor), Proposition \ref{prop-conversion} asserts that it suffices to find any $\hat\bgamma$ such that $\| \nabla \hat{L}_1 (\hat\bgamma) \|_2 \lesssim \sqrt{d/n}$ and $\lambda_n [ \nabla^{2} \hat L_1(\hat\bgamma) ] > - \eta / 2$. Here the perturbed gradient descent algorithm comes into play, and we see the light at the end of the tunnel. It remains to estimate the Lipschitz smoothness of $\nabla \hat{L}_1$ and $\nabla^2 \hat{L}_1$ with respect to the Euclidean norm. Once this is done, we can directly apply the convergence theorem in \cite{jin2017escape} for the perturbed gradient descent. A more comprehensive outline of the proof and all the details are deferred to the Appendix.

%% file: numerical_experiments.tex
\section{Numerical experiments} \label{sec:numerical}

\begin{figure}[t]
	\centering
	\includegraphics[width=0.7\textwidth]{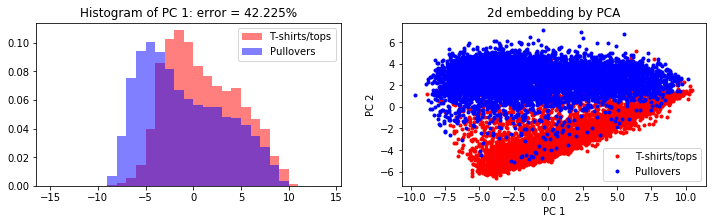}
	\caption{Visualization of the dataset via PCA. The left plot shows the transformed data via PCA. The right polt is a 2-dimensional visualization of the dataset using PCA.}
	\label{fig:shape}
\end{figure}

	\begin{figure}[t]
	\centering
	\includegraphics[width=0.9\textwidth]{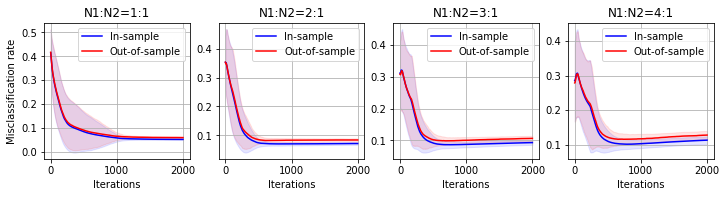}
	\caption{In sample and out-of-sample misclassification rate (with error bar quantifying one standard deviation) vs.~iteration count for \textsc{CURE} over 50 independent trials. The four plots corresponds to $N_2=6000$, $3000$, $2000$ and $1500$ respectively, while $N_1$ is always fixed to be $6000$.}
	\label{fig:err_iter}
\end{figure}

\begin{figure}[t]
	\centering
	\includegraphics[width=0.8\textwidth]{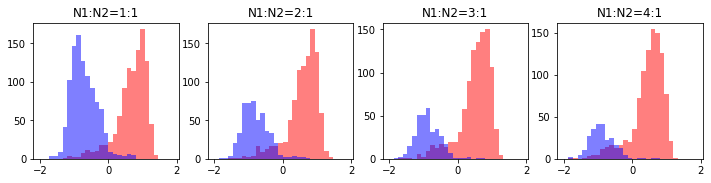}
	\caption{Histograms of transformed out-of-sample data for \textsc{CURE}. The red bins correspond to T-shirts/tops, and the blue bins correspond to pullovers.}
	\label{fig:hist}
\end{figure}

\begin{table}[t]
	\caption{Misclassification rate of \textsc{CURE} and other methods.  }\label{table:comparison}
	\vspace{0.8em}
	\centering
	\begin{tabular}{|c|c|c|c|c|}
		\hline 
		\multirow{1}{*}{} \backslashbox{Method}{$N_1:N_2$}& $1:1$ & $2:1$ & $3:1$ & $4:1$\tabularnewline
		\hline 
		CURE & $5.2\pm0.2\%$ & $7.1\pm0.4\%$ & $9.3\pm0.7\%$ & $11.3\pm1.1\%$\tabularnewline
		\hline 
		K-means & $45.1\%$ & $49.7\%$ & $46.8\%$ & $45.1\%$\tabularnewline
		\hline 
		Spectral method (vanilla) & $42.2\%$ & $46.9\%$ & $49.7\%$ & $49.0\%$\tabularnewline
		\hline 
		Spectral method (Gaussian kernel) & $49.9\%$ & $33.4\%$ & $25.0\%$ & $20.0\%$\tabularnewline
		\hline 
	\end{tabular}
	
\end{table}

	In this section, we conduct numerical experiments on a real dataset. We randomly select $N_1$ (resp.~$M_1$) T-shirts/tops and $N_2$ (resp.~$M_2$) pullovers from the Fashion-MNIST \citep{XRV17} training (resp.~testing) dataset, each of which is a $28\times 28$ grayscale image represented by a vector in $[0,1]^{28 \times 28}$. The goal is clustering, i.e. learning from those $N=N_1+N_2$ unlabeled images to predict the class labels of both $N$ training samples and $M=M_1+M_2$ testing samples. The inputs for \textsc{CURE} and other methods are raw images and their pixel-wise centered versions, respectively. To get a sense why this problem is difficult, we set $N_1=N_2=6000$ and plot the transformed data via PCA in the left panel of Figure \ref{fig:shape}: the transformation does not give meaningful clustering information, and the misclassification rate is $42.225\%$. A 2-dimensional visualization of the dataset using PCA (right panel of Figure \ref{fig:shape}) shows two stretched clusters, which cause the PCA to fail.
	In this dataset, the bulk of a image corresponds to the belly part of clothing with different grayscales, logos and hence contributes to the most of variability. However, T-shirts and Pullovers are distinguished by sleeves. Hence the two classes can be separated by a linear function that is not related to the leading principle component of data. \textsc{CURE} aims for such direction onto which the projected data exhibit cluster structures.
	

	To show that \textsc{CURE} works beyond our theory, we set $N_1$ to be $6000$ and choose $N_2$ from $\{6000, 3000, 2000, 1500\}$ to include unbalanced cases. We set $M_1$ to be $1000$ and choose $M_2$ from $\{1000,500,333,250\}$. We use gradient descent with random initialization from the unit sphere and learning rate $10^{-3}$ (instead of perturbed gradient descent) to solve \eqref{eqn-cure} as that requires less tuning. Figure \ref{fig:err_iter} shows the learning curves of \textsc{CURE} over 50 independent trials. Even when the classes are unbalanced, CURE still reliably achieves low misclassification rates. Figure \ref{fig:hist} presents histograms of testing data under the feature transform learned by the last (50th) trial of \textsc{CURE}, showing two seperated clusters around $\pm1$ corresponding to the two classes.
	To demonstrate the efficacy of \textsc{CURE}, we compare its misclassification rates with those of K-means and spectral methods on the training sets. We include the standard deviation over 50 independent trials for \textsc{CURE} due to its random initializations; other methods use the default settings (in Python) and thus are regarded as deterministic algorithms. As is shown in Table \ref{table:comparison}, \textsc{CURE} has the best performance under all settings.
	
	
	

%% file: discussion.tex
\section{Discussion}\label{sec:discussion}

Motivated by the elliptical mixture model (Model~\ref{model:1}), we propose a discriminative clustering method \textsc{CURE} and establish near-optimal statistical guarantees for an efficient algorithm.
It is worth pointing out that CURE learns a classification rule that readily predicts labels for any new data. This is an advantage over many existing approaches for clustering and embedding 
whose out-of-sample extensions are not so straightforward.
We impose several technical assumptions (spherical symmetry, constant condition number, positive excess kurtosis, etc.) to simplify the analysis, which we believe can be relaxed. Achieving Bayes optimality in multi-class clustering is indeed very challenging. Under parametric models such as Gaussian mixtures, one may construct suitable loss functions for CURE based on likelihood functions and obtain statistical guarantees.
Other directions that are worth exploring include the optimal choice of the target distribution and the discrepancy measure, high-dimensional clustering with additional structures, estimation of the number of clusters, to name a few. We also hope to further extend our methodology and theory to other tasks in unsupervised learning and semi-supervised learning.

The general \textsc{CURE} (Algorithm \ref{alg-cure}) provides versatile tools for clustering problems. In fact, it is related to several methods in the deep learning literature \citep{Spr15, XGF16,YFS17}. When we were finishing the paper, we noticed that \cite{GDV19} develop a deep clustering algorithm based on $k$-means and use optimal transport to incorporate prior knowledge of class proportions. Those methods are built upon certain network architectures (function classes) or loss functions while \textsc{CURE} offers more choices. In addition to the preliminary numerical results
, it would be nice to see how \textsc{CURE} tackles more challenging real data problems.



%% file: appendix_numerical_experiments.tex
\section{Additional numerical experiments} \label{appendix:numerical}
\subsection{Two classes}\label{appendix:numerical-1}
	In this section, we provide additional numerical experiments to compare \textsc{CURE} in (\ref{eqn-cure}) with other clustering methods on the same real dataset as Section \ref{sec:numerical}. We focus on six methods: (i) discriminative K-means (DisKmeans) in \cite{YZW08}; (ii) a discriminative clustering formulation described in \cite{BHa08,FPB17}; (iii) Model-based clustering (Mclust) in \cite{fraley1999mclust}; (iv) Projection Pursuit (PP) in \cite{pena2001cluster}; (v) Adaptive LDA-guided K-means Clustering in \cite{ding2007adaptive}; and (vi) Minimum Density Hyperplane (MDH) in \cite{pavlidis2016minimum}.
	
	As suggested by \cite{YZW08}, the regularization parameter $\lambda$ therein has a significant impact on the performance of DisKmeans. To resolve this issue, they provide an automatic tuning framework. Here we provide a comparison between \textsc{CURE} and DisKmeans. For the DisKmeans, we consider pre-chosen $\lambda \in \{0,1,10,100\}$ as well as $\lambda$ from the automatic tuning procedure suggested by \cite{YZW08}, initialized from 1. Due to high computational cost of DisKmeans with automatic tuning (which includes eigendecomposition of $(N_1+N_2)\times(N_1+N_2)$ matrix in each iteration), we conduct the experiment on smaller dataset: we fix $N_1=1000$ and choose $N_2$ from $\{1000,500,333,250\}$. As is shown in Table \ref{table:comparison2}, \textsc{CURE} has lower misclassification rate under all settings. It is also worth mentioning that the automatic tuning procedure sends $\lambda\to\infty$, in which case DisKmeans is equivalent to classical K-means.
	
\begin{table}[htpb]
	\caption{Misclassification rate of \textsc{CURE} and disciminative K-means.  }\label{table:comparison2}
	\vspace{0.8em}
	\centering
	\begin{tabular}{|c|c|c|c|c|c|}
		\hline 
		\multicolumn{2}{|c|}{\backslashbox{Method}{$N_1:N_2$}} & $1:1$ & $2:1$ & $3:1$ & $4:1$\tabularnewline
		\hline 
		CURE & / & $5.2\pm0.3\%$ & $6.7\pm0.6\%$ & $9.1\pm0.9\%$ & $11.2\pm1.2\%$\tabularnewline
		\hline 
		& $\lambda=0$ & $49.9\%$ & $49.5\%$ & $49.5\%$ & $47.7\%$\tabularnewline
		\cline{2-6} \cline{3-6} \cline{4-6} \cline{5-6} \cline{6-6} 
		Discriminative & $\lambda=1$ & $48.8\%$ & $46.6\%$ & $49.4\%$ & $48.3\%$\tabularnewline
		\cline{2-6} \cline{3-6} \cline{4-6} \cline{5-6} \cline{6-6} 
		K-means  & $\lambda=10$ & $46.5\%$ & $44.2\%$ & $47.4\%$ & $41.8\%$\tabularnewline
		\cline{2-6} \cline{3-6} \cline{4-6} \cline{5-6} \cline{6-6} 
		\cite{YZW08} & $\lambda=100$ & $6.6\%$ & $49.4\%$ & $46.5\%$ & $27.2\%$\tabularnewline
		\cline{2-6} \cline{3-6} \cline{4-6} \cline{5-6} \cline{6-6} 
		& automatic tuning & $43.3\%$ & $49.4\%$ & $47.5\%$ & $45.8\%$\tabularnewline
		\hline 
	\end{tabular}
	
\end{table}

For experiments comparing \textsc{CURE} with other five methods, we still adopt the usual setting of sample size: we fix $N_1=6000$ and choose $N_2$ from $\{6000,3000,2000,1500\}$. Model-based clustering (Mclust) in \cite{fraley1999mclust}, Projection Pursuit (PP) in \cite{pena2001cluster} and Minimum Density Hyperplane (MDH) in \cite{pavlidis2016minimum} are implemented using open-source R packages with default settings. In addition:
\begin{enumerate}
	\item The discriminative clustering method appeared in \cite{BHa08,FPB17} stems from the optimization problem
	\begin{equation}
	\label{eq:alt}
	\underset{\bm{v}\in\mathbb{R}^d,\bm{y}\in\{\pm1\}^d}{\min}\left\Vert\bm{y}-\bm{X}\bm{v}\right\Vert_2^2,
	\end{equation}
	where $\bX$ is the centered data matrix. We adopt the alternating minimization scheme: given $\bm{v}$, the optimal $\by$ is
	obtained by $\sgn(\bm{X}\bm{v})$ (or by running K-means on $\bm{X}\bm{v}$, which has similar empirical performance) while given $\bm{y}$, the optimal $\bm{v}$ is obtained from solving a least squares problem. In the first step, $\bm{v}$ is initialized from a uniform distribution over the unit sphere. The iterative algorithm is terminated when $\bm{y}$, the predicted label, no longer changes. 
	\item  Following the instructions in \cite{ding2007adaptive}, we implement the adaptive LDA-guided K-means clustering algorithm (Algorithm 1 therein) by alternating between linear discriminant analysis and K-means until convergence.
\end{enumerate}
Table \ref{table:comparison3} shows the misclassification rate and the standard deviation of \textsc{CURE} and the other five methods over 50 independent trials. It is clear that \textsc{CURE} is more accurate and stable than these five methods under all settings.
\begin{table}[htpb]
	\caption{Misclassification rate of \textsc{CURE} and other methods.  }\label{table:comparison3}
	\vspace{0.8em}
	\centering
	\begin{tabular}{|c|c|c|c|c|}
		\hline 
		\backslashbox{Method}{$N_1:N_2$} & $1:1$ & $2:1$ & $3:1$ & $4:1$\tabularnewline
		\hline 
		CURE & $5.2\pm0.2\%$ & $7.1\pm0.4\%$ & $9.3\pm0.7\%$ & $11.3\pm1.1\%$\tabularnewline
		\hline 
		Method \eqref{eq:alt} & $31.1\pm13.8\%$ & $32.9\pm13.3\%$ & $34.7\pm12.7\%$ & $36.8\pm11.2\%$\tabularnewline
		\hline 
		Mclust & $48.7\pm1.3\%$ & $39.1\pm4.8\%$ & $34.1\pm8.0\%$ & $28.2\pm7.8\%$\tabularnewline
		\hline
		Projection Pursuit & $36.9\pm9.8\%$ & $37.4\pm9.6\%$ & $39.7\pm6.9\%$ & $40.6\pm7.3\%$\tabularnewline
		\hline
		LDA-guided K-means & $45.9\%$ & $49.0\%$ & $45.6\%$ & $44.3\%$\tabularnewline
		\hline
		MDH & $48.6\%$ & $43.1\%$ & $38.3\%$ & $35.2\%$\tabularnewline
		\hline
	\end{tabular}

\end{table}

\subsection{Multiple classes}\label{appendix:numerical-2}

To illustrate how the general \textsc{CURE} in Section \ref{sec-cure-general} works, we consider the clustering problem with the first 4 classes in Fashion-MNIST (T-shirt/top, Trouser, Pullover, Dress), each of which has 6000 training samples and 1000 testing samples. Our training process only uses features of training samples and does not touch any labels.

We let the number of classes $K$ be 4, the embedding space $\cY$ be $\RR^K$, the target distribution $\nu$ be $\frac{1}{K} \sum_{j=1}^{K} \delta_{ \be_j }$, the discrepancy measure $D$ be the Wasserstein-1 distance, and define the classification rule $g (\by) = \argmin_{j \in [K] } \| \by - \be_j \|_2$. We compare two classes $\cF$ of feature mappings: linear functions and fully-connected neural networks with one hidden layer that has 100 nodes. Initial values All of the weight parameters are initialized using i.i.d.~samples from $N(0, 0.05^2)$.

Let $f_{\btheta}$ be a feature transform in $\cF$, parametrized by $\btheta$. Denote by $\{ \bx_i \}_{i=1}^n$ the samples, where $n = 4 \times 6000 = 24000$. The loss function is
\begin{align*}
L(\btheta) =  W_1 \bigg(
\frac{1}{n} \sum_{i = 1}^n \delta_{ f_{\btheta} (\bx_i) } ,~ \nu
\bigg) 
= \min_{ \bP \in [ 0, 1 ]^{n\times K},~ \mathbf{1}_{n}^{\top} \bP = \mathbf{1}_K^{\top} / K,~\bP \mathbf{1}_{K} = \mathbf{1}_n / n } \sum_{i=1}^{n}  \sum_{j=1}^{K} p_{ij} |  f_{\btheta} (\bx_i) - \be_j |.
\end{align*}
It is natural to optimize with respect to $\bP$ and $\btheta$ in an alternating manner. We apply random sampling techniques to speedup computation. In the $t$-th iteration,
\begin{enumerate}
\item Draw $B = 200$ samples $\{ \bx_{ti} \}_{i=1}^B$ uniform at random (with replacement) from the dataset;
\item Use the Python function \texttt{ot.sinkhorn2} in library \texttt{POT} \citep{FCo17} with \texttt{reg = 0.1} to obtain the solution $\bP_t$ to an entropy-regularized version of
\begin{align*}
\min_{ \bP \in [ 0, 1 ]^{B\times K},~ \mathbf{1}_{B}^{\top} \bP = \mathbf{1}_K^{\top} / K,~\bP \mathbf{1}_{K} = \mathbf{1}_B / B } \sum_{i=1}^{B}  \sum_{j=1}^{K} p_{ij} |  f_{\btheta_t} (\bx_{ti}) - \be_j |;
\end{align*}
\item Update model parameters by $\btheta_{t+1} = \btheta_t - \eta \partial  L_t (\btheta_t)$, where $\partial$ is the sub-differential operator, $\eta = 10^{-3}$ and
\begin{align*}
L_t (\btheta) = \sum_{j=1}^{K} \hat{p}_{ij} |  f_{\btheta} (\bx_{ti}) - \be_j | , \qquad \forall \btheta.
\end{align*}
\end{enumerate}

An epoch refers to $n/B = 12$ consecutive iterations. The learning curves in Figure \ref{fig-NN} shows the advantage of neural network and demonstrates the flexibility of CURE with nonlinear function classes. 

\begin{figure}[t]
\begin{center}
\includegraphics[width=0.5\textwidth]{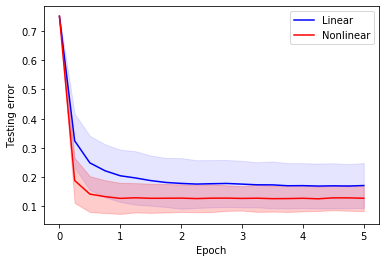}
\end{center}
\caption{4-class Fashion-MNIST: Testing errors of linear functions and neural networks, with error bar quantifying one standard deviation.
}
\label{fig-NN}
\end{figure}


%% file: appendix_proof_outlines.tex
\section{Proof sketch of Theorem \ref{thm:main}}\label{sec:theory}
\subsection{Step 1: properties of the test function $f$}
We now investigate the function $f$ defined in \eqref{eqn-test-function} and relate it to $h(x) = (x^2 - 1)^2 / 4$. As Lemma \ref{lem:f-property} suggests, $|f'|$, $|f''|$ and $|f'''|$ are all bounded by constants determined by $a$ and $b$; $|f'-h'|$ and $|f''-h''|$ are bounded by polynomials that are independent of $a$ and $b$. See Appendix \ref{appendix:f-property} for a proof.
\begin{lemma}\label{lem:f-property}
	When $a$ is sufficiently large and $b\geq 2a$, $f$ has the following properties:
	\begin{enumerate}
		\item $f^\prime$ is continuous with $F_1\triangleq\sup_{x\in\mathbb{R}}\vert f^\prime(x)\vert\leq2a^2b$ and $\vert f^\prime(x)-h^\prime(x)\vert\leq7\vert x\vert^3\mathbf{1}_{\{ \vert x\vert\geq a\} }$;
		\item  $f^{\prime\prime}$ is continuous with $F_2\triangleq\sup_{x\in\mathbb{R}}\vert f^{\prime\prime}(x)\vert\leq3a^2$ and $\vert f^{\prime\prime}(x)-h^{\prime\prime}(x)\vert\leq 9x^2 \mathbf{1}_{\{ \vert x\vert\geq a\} }$;
		\item $f^{\prime\prime\prime}$ exists in $\mathbb{R}\setminus\{\pm a,\pm b\}$ with $F_3\triangleq\sup_{x\in\mathbb{R}\setminus\{\pm a,\pm b\}}\vert f^{\prime\prime\prime}(x)\vert\leq 6a$.
	\end{enumerate}
\end{lemma}

\subsection{Step 2: landscape analysis of the population loss}
To kick off the landscape analysis we investigate the population version of $\hat L_1$, namely
\begin{equation}
L_1 \left( \alpha, \bbeta \right) = \mathbb{E}_{\bm{X}\sim\rho}f(\alpha+\bm{\beta}^{\top}\bm{X})+\frac{1}{2}(\alpha+\bbeta^\top\bmu_0)^2.
\end{equation}
One of the main obstacles is the complicated piecewise definition of $f$, which prevent us from obtaining closed form formulae. We bypass this problem by relating the population loss with $f$ to that with the quartic function $h$. See Appendix \ref{appendix:general-landscape-population} for a proof.
\begin{theorem}[Landscape of the population loss]\label{thm:landscape-general}
Consider Model~\ref{model:1} and assume that $b \geq 2 a$.
There exist positive constants $A, \varepsilon,\delta$ and $\eta$ determined by $M$, $\EE Z^4$, $\| \bmu \|_{2}$, $\lambda_{\max} ( \bSigma )$ and $\lambda_{\min} ( \bSigma )$ but independent of $d$ and $n$, such that when $a > A$,
	\begin{enumerate}
		\item The only two global minima of $L_1$ are $\pm \bgamma^\star$, where $\bgamma^\star=( - c \bm{\beta}^{h\top}\bm{\mu}_{0}, c \bm{\beta}^{h})$ for some $c\in(1/2,2)$ and
		\[
		\bm{\beta}^{h}=\left(\frac{1+1/\left\Vert \bm{\mu}\right\Vert _{\bm{\Sigma}^{-1}}^{2}}{\left\Vert \bm{\mu}\right\Vert _{\bm{\Sigma}^{-1}}^{4} +6\left\Vert \bm{\mu}\right\Vert _{\bm{\Sigma}^{-1}}^{2}+ M_Z }\right)^{1/2}\bm{\Sigma}^{-1}\bm{\mu};
		\]
		\item $\Vert\nabla L_1(\bm{\gamma})\Vert_{2}\geq\varepsilon$ if $\mathrm{dist}(\bm{\gamma}, \{ \pm \bgamma^\star \} \cup S)\geq\delta$, where $S = \{ \mathbf{0} \} \cup \{ ( -\bm{\beta}^{\top}\bm{\mu}_{0} , \bbeta) :~ \bmu^{\top} \bbeta = \mathbf{0},~ \bbeta^{\top} \bSigma \bbeta = 1 / M_Z \}
		$;
		\item $\nabla^{2}L_1(\bm{\gamma})\succeq\eta\bm{I}$ if $\mathrm{dist}(\bm{\gamma}, \{ \pm \bgamma^\star \} )\leq\delta$,
		and $\bm{u}^{\top}\nabla^{2}L_1(\bm{\gamma})\bm{u}\leq-\eta$ if
		$\mathrm{dist}(\bm{\gamma},S)\leq\delta$ with $\bm{u}=(0, \bm{\Sigma}^{-1}\bm{\mu} /\Vert\bm{\Sigma}^{-1}\bm{\mu}\Vert_{2})$.
	\end{enumerate}
\end{theorem}

Theorem \ref{thm:landscape-general} precisely characterizes the landscape of $L_1$. In particular, all of its critical points make up the set $\{\pm\bgamma^\star\}\cup S$, where $\pm\bgamma^\star$ are global minima and $S$ consists of strict saddles. The local geometry around critical points is also desirable.


\subsection{Step 3: landscape analysis of the empirical loss} \label{sec:landscape-empirical-loss}
Based on geometric properties of the population loss $L_1$, we establish similar results for the empirical loss $\hat L_1$ through concentration analysis. See Appendix \ref{appendix:theorem-landscape-sample} for a proof.

\begin{theorem}[Landscape of the empirical loss]\label{theorem-landscape-sample}
Consider Model~\ref{model:1} and assume that $b \geq 2 a \geq 4$. Let $\bgamma^\star$ and $S$ be defined as in Theorem \ref{thm:landscape-general}. 
	There exist positive constants $A, C_0, C_1, C_2, M_1, \varepsilon, \delta$ and $\eta$ determined by $M$, $M_Z$, $\| \bmu \|_{2}$, $\lambda_{\max} ( \bSigma )$ and $\lambda_{\min} ( \bSigma )$ but independent of $d$ and $n$, such that when $a \geq A$ and $n \geq C_0 d$, the followings hold with probability exceeding $1-C_1(d/n)^{C_2d}-C_1\exp(-C_2n^{1/3})$:
	\begin{enumerate}
		\item $\| \nabla \hat L_1 (\bgamma) \|_2 \geq \varepsilon$ if $\mathrm{dist} ( \bgamma , \{ \pm \bgamma^\star \} \cup S ) \geq \delta$;
		\item $ \bu^{\top}  \nabla^2 \hat L_1 (\bgamma) \bu \leq - \eta$ if $\mathrm{dist} ( \bgamma , S ) \leq \delta$, with $\bm{u}=(0, \bm{\Sigma}^{-1} \bm{\mu} /\Vert\bm{\Sigma}^{-1}\bm{\mu}\Vert_{2}) $;
		\item $\| \nabla \hat L_1 (\bgamma_1) -  \nabla \hat L_1 (\bgamma_2) \|_2 \leq M_1 \| \bgamma_1 - \bgamma_2 \|_2$ and $\| \nabla^2 \hat L_1 (\bgamma_1) -  \nabla^2 \hat L_1 (\bgamma_2) \|_2 \leq M_1 [ 1 \vee (d\log(n/d) / \sqrt{n}) ] \| \bgamma_1 - \bgamma_2 \|_2$ hold for all $ \bgamma_1,\bgamma_2 \in \RR \times \RR^{d}$.
	\end{enumerate}
\end{theorem}

Theorem \ref{theorem-landscape-sample} shows that a sample of size $n \gtrsim d$ suffices for the empirical loss to inherit nice geometric properties from its population counterpart. The corollary below illustrates that as long as we can find an approximate second-order stationary point, then the statistical estimation error can be well controlled by the gradient. We defer the proof of this to Appendix \ref{appendix:approx_beta_dist1}.

\begin{corollary}\label{cor:approx_beta_dist1}
	Under the settings in Theorem \ref{theorem-landscape-sample}, there exist constants $C,C_1',C_2'$ such that the followings happen with probability exceeding $1-C_1'(d/n)^{C_2'd}-C_1'\exp(-C_2'n^{1/3})$: for any $\bm{\gamma} \in \RR\times \RR^d $ satisfying
	$\Vert\nabla \hat{L}_1(\bm{\gamma} )\Vert_{2}\leq\varepsilon$ and $\lambda_{\min}[\nabla^{2}\hat{L}_1(\bm{\gamma} )]>-\eta$,
	\[
	\min_{s=\pm1}\left\Vert s\bm{\gamma} -\bm{\gamma}^{\star}\right\Vert _{2}\leq
	C \bigg(
	\big\Vert \nabla \hat{L}_1\left(\bm{\gamma} \right)\big\Vert _{2} + \sqrt{\frac{d}{n}\log\left( \frac{n}{d} \right)}
	\bigg).
	\]
As a result, when the event above happens, any local minimizer $\tilde\bgamma$ of $\hat{L}_1$ satisfies
\[
\min_{s=\pm1}\left\Vert s\tilde\bgamma -\bm{\gamma}^{\star}\right\Vert _{2} \leq C \sqrt{\frac{d}{n}\log \left( \frac{n}{d} \right) } .
\]
\end{corollary}

\subsection{Step 4: convergence guarantees for perturbed gradient descent}\label{sec:algorithm}

The landscape analysis above shows that all local minimizers of $\hat{L}_1$ are statistically optimal (up to logarithmic factors), and all saddle points are non-degenerate. Then it boils down to finding any $\bgamma$ whose gradient size is sufficiently small and Hessian has no significantly negative eigenvalue. Thanks to the Lipschitz smoothness of $\nabla \hat{L}_1$ and $\nabla^2 \hat{L}_1$, this can be efficiently achieved by the perturbed gradient descent algorithm (see Algorithm~\ref{alg:PGD}) proposed by \cite{jin2017escape}. Small perturbation is occasionally added to the iterates, helping escape from saddle points efficiently and thus converge towards local minimizers. Theorem \ref{theorem:algorithmic} provides algorithmic guarantees for \textsc{CURE} on top of that. We defer the proof to Appendix \ref{appendix:theorem-algorithmic}. 

Implementation of the algorithm requires specification of hyperparameters $a$, $b$, $M_1$, $\varepsilon$ and $\eta$. Under the regularity assumptions in Model \ref{model:1}, many structural parameters are well-behaved constants and that helps choose hyperparameters at least in a conservative way. In theory, we can let $b = 2a$; $a$ and $M_1$ be sufficiently large; $\varepsilon$ and $\eta$ be sufficiently small. In our numerical experiments, the algorithm does not appear to be sensitive to choices of hyperparameters. 
We do not go into much details to avoid distractions.

	\begin{algorithm}
	$\chi\leftarrow3\max\{\log(d\ell\Delta_{\mathrm{pgd}}/(c_{\mathrm{pgd}}\varepsilon_{\mathrm{pgd}}^2\delta_{\mathrm{pgd}})),4\}$, $\eta_{\mathrm{pgd}}\leftarrow c_{\mathrm{pgd}}/\ell$, $r\leftarrow \sqrt{c_{\mathrm{pgd}}}\varepsilon_{\mathrm{pgd}}/(\chi^2\ell)$, $g_{\mathrm{thres}}\leftarrow\sqrt{c_{\mathrm{pgd}}}\varepsilon_{\mathrm{pgd}}/\chi^2$, $f_{\mathrm{thres}}\leftarrow c_{\mathrm{pgd}}\varepsilon_{\mathrm{pgd}}^{1.5}/(\chi^3\sqrt{\rho})$, $t_{\mathrm{thres}}\leftarrow\chi\ell/(c_{\mathrm{pgd}}^2\sqrt{\rho\varepsilon_{\mathrm{pgd}}})$, $t_{\mathrm{noise}}\leftarrow-t_{\mathrm{thres}}-1$.
	
	{\bf Initialize} $\bgamma^0=\bgamma_{\mathrm{pgd}}$.\\
	{\bf For $t = 0,1,\ldots$ do}\\
	\hspace*{.5cm}{\bf If} $\Vert \nabla \hat{L}_1(\bgamma^t)\Vert_2 \leq g_{\mathrm{thres}}$ {\bf and} $t - t_{\mathrm{noise}} > t_{\mathrm{thres}}$: \\
	\hspace*{1cm}Update $t_{\mathrm{noise}}\leftarrow t$, \\
	\hspace*{1cm}Perturb $\displaystyle \bgamma^t \leftarrow \bgamma^t + \bxi^t$ with $\bxi^t\sim \cU (B(\bm{0},r))$ \\
	\hspace*{.5cm}{\bf If} $t-t_{\mathrm{noise}}=t_{\mathrm{thres}}$ {\bf and} $\hat{L}_1(\bgamma^t) - \hat{L}_1(\tilde{\bgamma}^{t_{\mathrm{noise}}})>-f_{\mathrm{thres}}$: \\
	\hspace*{1cm} {\bf Return} $\tilde{\bgamma}^{t_{\mathrm{noise}}}$ \\
	\hspace*{.5cm} {\bf Update} $\bgamma^{t+1} \leftarrow \bgamma^{t} - \eta_{\mathrm{pgd}}\nabla \hat{L}_1(\bgamma^t)$.
	\caption{Perturbed gradient descent~~ $\mathrm{PerturbedGD}( \bgamma_{\mathrm{pgd}},  \ell,\rho,\varepsilon_{\mathrm{pgd}},c_{\mathrm{pgd}},\delta_{\mathrm{pgd}},\Delta_{\mathrm{pgd}})$}
	\label{alg:PGD}
\end{algorithm}

\begin{theorem}[Algorithmic guarantees] \label{theorem:algorithmic}
	Consider the settings in Theorem \ref{theorem-landscape-sample} and adopt the constants $M_1$, $\varepsilon$ and $\eta$ therein. With
	probability exceeding $1-C_1 [ (d/n)^{C_2d} + e^{-C_2n^{1/3}}+n^{-10}]$, Algorithm~\ref{alg:PGD} with parameters $\bgamma_{\mathrm{pgd}} = \mathbf{0}$, $\ell=M_1$, $\delta_{\mathrm{pgd}}=n^{-11}$, $\rho=M_1\max\{1,d\log(n/d)/\sqrt{n}\}$, $\varepsilon_{\mathrm{pgd}}=\min\{\sqrt{d\log(n/d)/n},\ell^2/\rho,\eta^2/\rho,\varepsilon\}$ and $\Delta_{\mathrm{pgd}}=1/4$
	terminates within $\tilde O ({n}/{d}+{d^2}/{n})$ iterations and the output $\hat{\bm{\gamma}}$ satisfies
	\[
	\bigl\Vert\nabla\hat{L}_{1}(\hat{\bm{\gamma}})\bigr\Vert_{2}\leq\sqrt{\frac{d}{n}\log
\left(\frac{n}{d}\right)	
}\leq\varepsilon\qquad\text{and}\qquad\lambda_{\min}\bigl(\nabla^{2}\hat{L}_{1}(\hat{\bm{\gamma}})\bigr)\geq-\eta.
	\]
\end{theorem}

Theorem \ref{theorem:algorithmic} and Corollary \ref{cor:approx_beta_dist1} immediately lead to
\[
\min_{s=\pm1}\left\Vert s\hat{\bgamma} -\bm{\gamma}^{\star}\right\Vert_{2}
\lesssim
\big\Vert \nabla \hat{L}_1(\hat{\bm{\gamma}})\big\Vert _{2} + \sqrt{\frac{d}{n}\log\left(\frac{n}{d}\right)}
\lesssim \sqrt{\frac{d}{n}\log\left(\frac{n}{d}\right)},
\]
which finishes the proof of Theorem \ref{thm:main}.

%% file: appendix_general_landscape_population.tex
\section{Preliminaries}

Before we start the proof, let us introduce some notations. Recall the definition of the random vector $\bX=\bmu_0+\bmu Y+\bSigma^{1/2}\bZ$ and the i.i.d.~samples $\bX_1,\ldots,\bX_n\in\mathbb{R}^d$. Let $\bar{\bX}=(1,\bX)$, $\bar{\bX}_i=(1,\bX_i)$ and $\bar{\bmu}_0=(1,\bmu_0)$. For any $\bgamma=(\alpha,\bbeta)\in\mathbb{R}\times\RR^d$, define
\begin{equation*}
L_\lambda(\bgamma)=L(\bgamma)+\lambda R(\bgamma)\qquad\text{and}\qquad \hat{L}_\lambda(\bgamma)=\hat{L}(\bgamma)+\lambda \hat{R}(\bgamma),
\end{equation*}
where
\begin{align*}
& L(\bgamma)=\mathbb{E}f(\bgamma^\top\bar{\bX})=\mathbb{E}f(\alpha+\bbeta^\top\bX),\qquad \hat{L}(\bgamma)=\frac{1}{n}\sum_{i=1}^{n}f(\bgamma^\top\bar{ \bX}_i)=\frac{1}{n}\sum_{i=1}^{n}f(\alpha+\bbeta^\top\bX_i), \\
& R(\bgamma)=\frac{1}{2}(\alpha+\bbeta^\top\bmu_0)^2=\frac{1}{2}(\bgamma^\top\bar{\bmu}_0)^2,\qquad \hat{R}(\bgamma)=\frac{1}{2}(\alpha+\bbeta^\top n^{-1}\sum_{i=1}^{n}\bX_i)^2=\frac{1}{2}(\bgamma^\top n^{-1}\sum_{i=1}^{n}\bar{\bX}_i)^2.
\end{align*}
Note that the results stated in Section \ref{sec:main_results} and \ref{sec:theory} focus on the special case when $\lambda=1$. The proof in the appendices allows for general choices of $\lambda \geq 1$.

\section{Proof of Lemma \ref{lem:f-property}} \label{appendix:f-property}
By direct calculation, one has
\begin{align*}
&f'(x) = \begin{cases}
h'(x), & |x| \leq a \\
[ h'(a) + h''(a) (|x|-a) - \frac{h''(a)}{2(b-a)} (|x|-a)^2 ]  \sgn(x), & a < |x| \leq b \\
[ h'(a) + \frac{b-a}{2} h''(a) ] \sgn(x) , & |x| > b
\end{cases}, \\
&f''(x) = \begin{cases}
h''(x), & |x| \leq a \\
h''(a) ( 1 - \frac{|x|-a}{b-a} ) , & a < |x| \leq b \\
0 , & |x| > b.
\end{cases}, \\
&f'''(x) = \begin{cases}
h'''(x), & |x| < a \\
- \frac{h''(a)}{b-a} \sgn (x), & a < |x| < b \\
0 , & |x| > b
\end{cases}.
\end{align*}
When $a$ is sufficiently large and $b\geq2a$, we have $F_1\triangleq\sup_{x \in \RR} |f' (x)| = h'(a) + \frac{b-a}{2} h''(a)\leq2a^2b$, $F_2\triangleq\sup_{x \in \RR} |f'' (x)| = h''(a)\leq3a^2$, and $F_3\triangleq\sup_{|x| \neq a, b} |f''' (x)| = h'''(a) \vee \frac{h''(a)}{b-a}\leq6a$. 

In addition, one can also check that when $a<\vert x\vert\leq b$, we have $\vert h^\prime(a)\vert\leq\vert x\vert^3$ and $\vert h^{\prime\prime}(a)\vert\leq3\vert x\vert^2$, thus
\begin{align*}
\vert f^\prime(x)-h^\prime(x)\vert&\leq\left\vert f^\prime(x)\right\vert+\left\vert h^\prime(x)\right\vert\leq \vert h^\prime(a)\vert+\vert h^{\prime\prime}(a)(\vert x\vert-a)\vert+\vert h''(a)(|x|-a)^2 /(2a)\vert+\vert x^3-x\vert\\
&\leq \vert x\vert^3+3\vert x\vert^2+\frac{3}{2}\vert x\vert^2+\vert x\vert^3\leq 7\vert x\vert^3
\end{align*}
provided that $b\geq 2a \geq 2$. When $\vert x\vert\geq b$, we have
\begin{align*}
\vert f^\prime(x)-h^\prime(x)\vert&\leq\left\vert f^\prime(x)\right\vert+\left\vert h^\prime(x)\right\vert\leq 
\vert h^\prime(a)\vert+\vert (b-a) h''(a) /2 \vert+\vert x^3-x\vert\\
&\leq \vert x\vert^3+\frac{3}{2}\vert x\vert^2+\vert x\vert^3\leq 4\vert x\vert^3.
\end{align*}
This combined with $f^\prime(x)=h^\prime(x)$ when $\vert x\vert\leq a$ gives $\vert f^\prime(x)-h^\prime(x)\vert\leq
\mathbf{1}_{ \{ |x| \geq a \} } 7 \vert x\vert^3$. Similarly we have $\vert f^{\prime\prime}(x)-h^{\prime\prime}(x)\vert\leq \mathbf{1}_{ \{ |x| \geq a \} } 9x^2$.
\section{Proof of Theorem \ref{thm:landscape-general}}\label{appendix:general-landscape-population}


It suffices to focus on the special case $\bmu_0 = \bm{0}$ and $\bSigma = \bI_d$. We first give a theorem that characterizes the landscape of an auxiliary population loss, which serves as a nice starting point of the study of the actual loss functions that we use.

\begin{theorem}[Landscape of the auxillary population loss]\label{thm:landscape-general-quartic}
	Consider model \eqref{model:1} with $\bmu_0 = \mathbf{0}$ and $\bSigma = \bI_d$. Suppose that $M_Z > 3$. 
	Let $h(x)= (x^{2}-1)^{2} / 4$ and $\lambda \geq 1$. The stationary points of the population loss
	\[
	L_\lambda^{h}\left( \alpha, \bbeta \right)
	= \mathbb{E} h\left(\alpha+\bm{\beta}^{\top}\bm{X}\right)+\frac{\lambda}{2} \alpha^2 
	\]
	are $\{ (\alpha, \bbeta) :\nabla L_\lambda^{h}( \alpha, \bbeta )=\bm{0}\}=S_{1}^h \cup S_{2}^h $,
	where
	\begin{enumerate}
		\item $S_{1}^h=\{  (0, \pm \bbeta^h ) \}$ consists of global minima, with 
		\[
		\bm{\beta}^{h}=\left(\frac{1+1/\left\Vert \bm{\mu}\right\Vert _{2}^{2}}{\left\Vert \bm{\mu}\right\Vert _{2}^{4} +6\left\Vert \bm{\mu}\right\Vert _{2}^{2}+ M_Z }\right)^{1/2} \bm{\mu};
		\]
		\item $S_{2}^h=\{ (0,\bm{\beta}):~ \bm{\mu}^{\top}\bm{\beta}=0,~\| \bm{\beta} \|_2^2=1/ M_Z \} \cup \{ \mathbf{0} \}$
		consists of saddle points whose Hessians have negative eigenvalues.
	\end{enumerate}
	We also have the following quantitative results: there exist positive constants $\varepsilon^h, \delta^h$ and $\eta^h$ determined by $M_Z$, $\| \bmu \|_2$ and $\lambda$ such that
	\begin{enumerate}
		\item $\Vert\nabla L^{h}_\lambda( \bgamma )\Vert_{2}\geq\varepsilon^h$ if $\mathrm{dist}( \bgamma ,S_{1}^h\cup S_{2}^h )\geq\delta^h $;
		\item $\nabla^{2}L_\lambda^{h}(\bgamma)\succeq\eta^h \bm{I}$ if $\mathrm{dist}( \bgamma ,S_{1}^h)\leq 3 \delta^h$,
		and $\bm{u}^{\top}\nabla^{2}L_\lambda^{h}( \bgamma )\bm{u}\leq-\eta^h$ if
		$\mathrm{dist}( \bgamma , S_{2}^h )\leq 3 \delta^h$ where $\bm{u}=(0,\bmu / \| \bmu \|_2 ) $.
	\end{enumerate}
\end{theorem}
\begin{proof}
	See Appendix \ref{appendix:general-quartic-landscape-population}.
\end{proof}


The following Lemma \ref{lemma-perturbation-population} controls the difference between the landscape of $L_\lambda$ and $L^h_\lambda$ within a compact ball.

\begin{lemma}\label{lemma-perturbation-population}
	Let $\bX$ be a random vector in $\RR^{d+1}$ with $\Vert\bm{X}\Vert_{\psi_{2}}\leq M$, $f$ be defined in (\ref{eqn-test-function}) with $b\geq 2a \geq 4$, $h(x) = (x^2 - 1)^2 / 4$ for $x \in \RR$, $L_\lambda(\bgamma) =  \EE f (\bgamma^{\top} \bX) +\lambda\alpha^2/2$ and $L_\lambda^h (\bgamma) = \EE h(\bgamma^{\top} \bX)+\lambda\alpha^2/2 $ for $\bgamma \in \RR^{d+1}$. There exist constants $C_1,C_2 > 0$ such that for any $R>0$,
	\begin{align*}
	&\sup_{ \| \bm{\gamma} \|_2 \leq R }\left\Vert \nabla L_\lambda \left(\bm{\gamma}\right)-\nabla L_\lambda^h \left(\bm{\gamma}\right)\right\Vert _{2}\leq C_2 R^{3}M^{4}\exp\left(-\frac{C_1 a^{2}}{R^{2}M^{2}}\right),\\
	&\sup_{  \| \bm{\gamma} \|_2 \leq R }\left\Vert \nabla^{2}L_\lambda \left(\bm{\gamma}\right)-\nabla^{2}L_\lambda^h \left(\bm{\gamma}\right)\right\Vert_2 \leq
	C_2 R^{2}M^{4}\exp\left(-\frac{C_1 a^{2}}{R^{2}M^{2}}\right).
	\end{align*}
	In addition, when $ \EE (\bX \bX^{\top}) \succeq \sigma^2 \bI$ holds for some $\sigma > 0$, there exists $m > 0$ determined by $M$ and $\sigma$ such that $\inf_{ \| \bgamma \|_2 \geq 3 / m } \| \nabla L_\lambda(\bgamma) \|_2 \geq m$ and $\inf_{ \| \bgamma \|_2 \geq 3 / m } \| \nabla L^h_\lambda(\bgamma) \|_2 \geq m$.
\end{lemma}
\begin{proof}
	See Appendix \ref{appendix:perturbation-population}.
\end{proof}

	On the one hand,  Lemma \ref{lemma-perturbation-population} implies that $\inf_{ \| \bgamma \|_2 \geq 3 / m } \| \nabla L_\lambda(\bgamma) \|_2 \geq m$ for some constant $m > 0$. Suppose that
	\begin{align}
	\varepsilon^h < m	
	\label{ineq-proof-thm-landscape-general-0-1}
	\end{align}
	and define $r = 3 / \varepsilon^h$. Then
	\begin{align}
	\| \nabla L_1 (\bgamma) \|_2 > \varepsilon^h \qquad \text{if} \qquad \| \bgamma \|_2 \geq r.
	\label{ineq-proof-thm-landscape-general-1}
	\end{align}
	Moreover, we can take $a$ to be sufficiently large such that
	\begin{align}
	\sup_{  \| \bm{\gamma} \|_2 \leq r }\big\Vert \nabla L_1 \left(\bm{\gamma}\right)-\nabla L_1^h \left(\bm{\gamma}\right)\big\Vert _{2}\leq\varepsilon^h/2.
	\label{ineq-proof-thm-landscape-general-2}
	\end{align}
	On the other hand, from Theorem \ref{thm:landscape-general-quartic} we know that 
	\begin{align}
	\Vert\nabla L^{h}_\lambda(\bm{\gamma})\Vert_{2}\geq\varepsilon^h
	\qquad \text{if} \qquad 
	\mathrm{dist}(\bm{\gamma},S_{1}^h\cup S_2^h )\geq\delta^h.
	\label{ineq-proof-thm-landscape-general-3}
	\end{align}
	Taking (\ref{ineq-proof-thm-landscape-general-1}), (\ref{ineq-proof-thm-landscape-general-2}) and (\ref{ineq-proof-thm-landscape-general-3}) collectively gives
	\begin{align}
	\Vert\nabla L_\lambda(\bm{\gamma})\Vert_{2}\geq\varepsilon^h/2
	\qquad \text{if} \qquad 
	\mathrm{dist}(\bm{\gamma},S_{1}^h\cup S_2^h )\geq\delta^h.
	\label{ineq-proof-thm-landscape-general-3.2}
	\end{align}
	Hence $\{\bgamma:~ \nabla L_\lambda(\bgamma)=\bm{0}\} \subseteq \{ \bgamma:~
	\mathrm{dist}( \bgamma ,S_1^h\cup S_2^h )\leq\delta^h \}$ and it yields a decomposition $\{\bgamma: ~ \nabla L_\lambda(\bgamma)=\bm{0}\}=S_1\cup S_2 $, where
	\begin{align}
	S_j \subseteq \{ \bgamma:~ \mathrm{dist}(\bgamma, S_j^h) \leq \delta^h \} ,
	\qquad
	\forall j = 1,~2.
	\label{ineq-proof-thm-landscape-general-3.5}
	\end{align}	
Consequently, for $j = 1, 2$ we have
	\begin{align}
	& \{ \bgamma:~ \mathrm{dist}(\bm{\gamma},S_{j}) \leq 2 \delta^h \} \subseteq \{ \bgamma:~ \mathrm{dist}(\bm{\gamma},S_{j}^h) \leq 3 \delta^h \}
	\subseteq
	\{ \bgamma:~ \| \bm{\gamma} \|_2 \leq 3 \delta^h + \max_{\bgamma' \in S_1^h\cup S_2^h} \| \bgamma' \|_2 \}.
	\label{ineq-proof-thm-landscape-general-4}
	\end{align}
	
	{\bf Now we work on the first proposition in Theorem \ref{thm:landscape-general} by characterizing $S_1$.}
	\begin{lemma}\label{lemma-population-existence}
		Consider the model in (\ref{model:1}) with $\bmu_0 = \mathbf{0}$ and $\bSigma = \bI_d $. Suppose that $f \in C^2 (\RR)$ is even, $\lim_{x \to +\infty} x f'(x) = + \infty$ and $f''(0) < 0$. Define
		\begin{align*}
		L_\lambda ( \alpha , \bbeta ) = \EE f( \alpha + \bbeta^{\top} \bX ) + \frac{\lambda}{2} \alpha^2, \qquad
		\forall \alpha \in \RR,~~ \bbeta \in \RR^d.
		\end{align*}
		\begin{enumerate}
			\item There exists some $c > 0$ determined by $\| \bmu \|_2$, the function $f$, and the distribution of $Z$, such that $ (0, \pm c  \bmu )$ are critical points of $L_\lambda$;
			\item In addition, if $f''$ is piecewise differentible and $|f'''(x)| \leq F_3 < \infty$ almost everywhere, we can find $c_0 > 0$ determined by $\|  \bmu \|_2$, $f''(0)$, $F_3$ and $M$ such that $c > c_0$. 
		\end{enumerate}
	\end{lemma}
	\begin{proof}
		See Appendix \ref{appendix:lemma-existence}.
	\end{proof}
	
	Lemma \ref{lemma-population-existence} asserts the existence of two critical points $\pm \bgamma^\star = (0, \pm c \bbeta^h )$ of $L_1$, for some $c$ bounded from below by a constant $c_0 > 0$. If
	\begin{align}
	\delta^h < c_0  \| \bbeta^h \|_2 / 4 ,
	\label{ineq-proof-thm-landscape-general-0-2}
	\end{align}
	then the property of $S_2^h$ forces 
\begin{align}
	\dist(  \pm \bgamma^\star  , S_2^h ) \geq \| \bgamma^\star \|_2 = c \| \bbeta^h \|_2 \geq c_0  \| \bbeta^h \|_2 > 4 \delta^h > 3 \delta^h.
		\label{ineq-proof-thm-landscape-general-4.2}
	\end{align}
	It is easily seen from (\ref{ineq-proof-thm-landscape-general-4}) with $j = 2$ that $\dist( \pm \bgamma^\star , S_2 ) > 2\delta^h$ and $ \pm \bgamma^\star \notin S_2$. 
	Then $\{ \bgamma:~ \nabla L_1(\bgamma) = \mathbf{0} \} = S_1 \cup S_2$ forces
	\begin{align}
	\{  \bgamma^\star , - \bgamma^\star \} \subseteq S_1. 
	\label{ineq-proof-thm-landscape-general-4.5}
	\end{align}
	Let us investigate the curvature near $S_1$. Lemma \ref{lemma-perturbation-population} and (\ref{ineq-proof-thm-landscape-general-4}) with $j = 1$ allow us to take $a$ to be sufficiently large such that
	\begin{align}
	\sup_{  \mathrm{dist}(\bm{\gamma},S_{1}) \leq 2 \delta^h }\big\Vert \nabla^{2}L_\lambda \left(\bm{\gamma}\right)-\nabla^{2}L_\lambda^h \left(\bm{\gamma}\right)\big\Vert_2 \leq\eta^h/2.
	\label{ineq-proof-thm-landscape-general-5}
	\end{align}
	Theorem \ref{thm:landscape-general-quartic} asserts that 
	$\nabla^{2}L^h_\lambda(\bm{\gamma})\succeq \eta^h \bm{I} $ if $\mathrm{dist}(\bm{\gamma},S_{1}^h) \leq 3\delta^h$. 
	By this, (\ref{ineq-proof-thm-landscape-general-4}) with $j = 1$ and (\ref{ineq-proof-thm-landscape-general-5}),
	\begin{align}
	\nabla^{2}L_\lambda(\bm{\gamma})\succeq (\eta^h/2) \bm{I}
	\qquad\text{if}\qquad
	\mathrm{dist}(\bm{\gamma},S_{1})\leq 2 \delta^h.
	\label{ineq-proof-thm-landscape-general-6}
	\end{align}
	Hence $L_1$ is strongly convex in $\{ \bgamma:~ \mathrm{dist}(\bm{\gamma},S_{1})\leq 2 \delta^h \}$. Combined with (\ref{ineq-proof-thm-landscape-general-4.5}), it leads to $S_1 = \{ \pm \bgamma^\star  \}$, and both points therein are local minima.
	
	Let $\bgamma^h = (0, \bbeta^h)$. The fact $S_1^h = \{ \pm \bgamma^h \}$ and (\ref{ineq-proof-thm-landscape-general-3.5}) yields
	\begin{align}
	|c-1| \cdot \| \bbeta^h \|_2 = \|  \bgamma^\star - \bgamma^h \|_2 = \dist ( \bgamma^\star , S_1^h ) \leq \delta^h.
		\label{ineq-proof-thm-landscape-general-6.5}
	\end{align}
	When
	\begin{align}
	\delta^h < \| \bbeta^h \|_2 / 2,
	\label{ineq-proof-thm-landscape-general-0-3}
	\end{align}
	we have $1/2 < c < 3/2$ as claimed. The global optimality of $\pm \bgamma^\star $ is obvious. 
	Without loss of generality, in Theorem \ref{thm:landscape-general-quartic} we can always take $\delta^h < \| \bbeta^h \|_2 \min \{ c_0 / 3, 1/2 \}$ and then find $\varepsilon^h< m$. In that case,
	(\ref{ineq-proof-thm-landscape-general-0-1}), (\ref{ineq-proof-thm-landscape-general-0-2}) and (\ref{ineq-proof-thm-landscape-general-0-3}) imply the first proposition in Theorem \ref{thm:landscape-general}.

	{\bf Next, we study the second proposition in Theorem \ref{thm:landscape-general}.} 
	Let $S = S_2^h$. Given $S_1=\{\pm\bgamma^h\}$ and $S_1=\{\pm\bgamma^\star\}$, from (\ref{ineq-proof-thm-landscape-general-6.5}) we know that $\mathrm{dist}(\bgamma,\{ \pm \bgamma^\star  \} \cup S )\geq 2\delta^h$ implies $\mathrm{dist}(\bgamma,S_1^h \cup S_2^h )\geq 2\delta^h$. This combined with (\ref{ineq-proof-thm-landscape-general-3.2}) immediately gives
	\begin{align*}
	\Vert\nabla L_\lambda(\bm{\gamma})\Vert_{2}\geq\varepsilon^h/2
	\qquad \text{if} \qquad 
	\mathrm{dist}(\bm{\gamma},\{ \pm \bgamma^\star  \} \cup S )\geq 2\delta^h.
	\end{align*}
	Hence the second proposition in Theorem \ref{thm:landscape-general} holds if
	\begin{align}
\varepsilon = \varepsilon^h / 2 \qquad \text{and} \qquad \delta = 2 \delta^h.
\label{ineq-proof-thm-landscape-general-0-4}
	\end{align}
	
	{\bf Finally, we study the third proposition in Theorem \ref{thm:landscape-general}.} By (\ref{ineq-proof-thm-landscape-general-6}), the first part of that proposition holds when
	\begin{align}
\eta = \eta^h / 2  \qquad \text{and} \qquad \delta = 2 \delta^h.
\label{ineq-proof-thm-landscape-general-0-5}
	\end{align}
	It remains to prove the second part. Lemma \ref{lemma-perturbation-population} and (\ref{ineq-proof-thm-landscape-general-4}) with $j = 2$ allow us to take $a$ to be sufficiently large such that
	\begin{align}
	\sup_{  \mathrm{dist}(\bm{\gamma},S) \leq 3 \delta^h }\big\Vert \nabla^{2}L_\lambda \left(\bm{\gamma}\right)-\nabla^{2}L_\lambda^h \left(\bm{\gamma}\right)\big\Vert_2 \leq\eta^h/2.
	\label{ineq-proof-thm-landscape-general-11}
	\end{align}
	Theorem \ref{thm:landscape-general-quartic} asserts that 
	$\bm{u}^{\top} \nabla^{2}L^h_\lambda(\bm{\gamma}) \bm{u} \leq - \eta^h $ for $\bu = (0, \bmu / \| \bmu \|_2 )$ if $\mathrm{dist}(\bm{\gamma},S) \leq 3\delta^h$. By this, (\ref{ineq-proof-thm-landscape-general-4}) with $j = 2$ and (\ref{ineq-proof-thm-landscape-general-11}),
	\begin{align}
	\nabla^{2}L_\lambda(\bm{\gamma}) \leq  - \eta^h/2 
	\qquad\text{if}\qquad
	\mathrm{dist}(\bm{\gamma},S )\leq 3 \delta^h.
	\end{align}
	Hence (\ref{ineq-proof-thm-landscape-general-0-4}) suffice for the second part of the third proposition to hold.
	
	According to (\ref{ineq-proof-thm-landscape-general-0-4}) and (\ref{ineq-proof-thm-landscape-general-0-5}), Theorem \ref{thm:landscape-general} holds with $\varepsilon = \varepsilon^h / 2$, $\delta = 2 \delta^h$ and $\eta = \eta^h / 2$.

\subsection{Proof of Theorem \ref{thm:landscape-general-quartic}}\label{appendix:general-quartic-landscape-population}
	\subsubsection{Part 1: Characterization of stationary points} 
Note that
	\begin{align*}
	\nabla L_{\lambda}^h ( \alpha, \bbeta ) & = \EE \left[
	\begin{pmatrix}
1 \\
\bX
	\end{pmatrix}
 h' ( \alpha + \bbeta^{\top} \bX ) \right] + 
 	\begin{pmatrix}
\lambda \\
\mathbf{0}
 \end{pmatrix}
\\
	& = 
	\begin{pmatrix}
	\EE h' ( \alpha + \bbeta^{\top} \bX ) + \lambda \\
	\mathbf{0}
	\end{pmatrix}
+ \begin{pmatrix}
0 \\
\EE [ Y h' ( \alpha + \bbeta^{\top} \bX )] \bmu
\end{pmatrix}
+ \begin{pmatrix}
0 \\
\EE [ \bZ h' ( \alpha + \bbeta^{\top} \bX )]
\end{pmatrix}.
	\end{align*} 
Now we will expand individual expected values in this sum. For the first term,
	\begin{align*}
	\EE h'(\alpha + \bbeta^\top\bX) &= \EE (\alpha + \bbeta^\top\bmu Y +  \bbeta^\top\bZ)^3 - \EE (\alpha + \bbeta^\top\bmu Y +  \bbeta^\top\bZ )\\
	&= \alpha^3 + 3 \alpha \EE (\bbeta^\top\bmu Y)^2 + 3 \alpha \EE( \bbeta^\top\bZ)^2 + \EE (\bbeta^\top\bmu Y +  \bbeta^\top\bZ)^3 - \alpha  \\
	&= \alpha [ \alpha^2+ 3(\bbeta^\top\bmu)^2  + 3\|\bbeta\|_2^2 - 1],
	\end{align*}
	where the first line follows since $h'(x) = x^3 - x$, the other two follows from $\EE ( \bZ\bZ^\top ) = \bI$ plus the fact that $Y$ and $\bZ$ are independent, with zero odd moments due to their symmetry.
	
	Using similar arguments,
	\begin{align*}
	\EE [ Y h' (\alpha + \bbeta^{\top} \bX ) ] &= \EE [ Y (\alpha + \bbeta^{\top} \bmu  Y + \bbeta^{\top} \bZ )^{3} ] - \EE [ Y ( \alpha + \bbeta^{\top} \bmu  Y + \bbeta^{\top} \bZ ) ] \notag \\
	&  =3\alpha^2 \EE \left[ Y( \bbeta^{\top} \bmu  Y + \bbeta^{\top} \bZ)\right] + \EE [ Y (\bbeta^{\top} \bmu  Y + \bbeta^{\top} \bZ )^{3} ] - \bbeta^\top \bmu \notag \\
	&  =3\alpha^2\bbeta^\top \bmu + \EE [ Y ( \bbeta^{\top} \bmu  Y )^3 ] + 3 \EE [ Y ( \bbeta^{\top} \bmu  Y)] \EE [ ( \bbeta^{\top} \bZ )^{2} ] - \bbeta^\top \bmu \notag \\
	& = \left[ 3\alpha^2 + ( \bbeta^{\top} \bmu )^2 \mathbb{E}Y^4+  3 \| \bbeta \|_2^2 - 1 \right] \bbeta^\top \bmu.
	\end{align*}
	
	To work on $\EE [ \bZ h' (\alpha + \bbeta^{\top} \bX ) ] = \EE [ \bZ h'(\alpha +  \bbeta^{\top} \bmu Y + \bbeta^{\top} \bZ ) ]$, we
	define $\bar{\bbeta} = \bbeta / \| \bbeta \|_2$ for $\bbeta \neq \mathbf{0}$ and $\bar{\bbeta} = \mathbf{0}$ otherwise. Observe that $(  Y ,  \bar{\bbeta} \bar{\bbeta}^{\top} \bZ, (\bI -  \bar{\bbeta} \bar{\bbeta}^{\top}) \bZ )$ and $(  Y ,  \bar{\bbeta} \bar{\bbeta}^{\top} \bZ, - (\bI -  \bar{\bbeta} \bar{\bbeta}^{\top}) \bZ )$ have exactly the same joint distribution. As a result,
	\begin{align*}
	\EE [ (\bI -  \bar{\bbeta} \bar{\bbeta}^{\top}) \bZ h' (\alpha + \bbeta^{\top} \bX ) ] 
	= \EE [ (\bI -  \bar{\bbeta} \bar{\bbeta}^{\top}) \bZ h' ( \alpha + \bbeta^{\top} \bmu  Y + \bbeta^{\top} \bZ ) ] 
	= \mathbf{0}.
	\end{align*} 
	Hence,
	\begin{align*}
	\EE [ \bZ h' ( \bbeta^{\top} \bX ) ] 
	& = \EE [  \bar{\bbeta} \bar{\bbeta}^{\top} \bZ h' ( \alpha + \bbeta^{\top} \bX ) ] 
	= \EE [  \bar{\bbeta}^{\top} \bZ h' (  \alpha + \bbeta^{\top} \bmu  Y + \bbeta^{\top} \bZ ) ]  \bar{\bbeta} \notag \\
	&  \overset{ }{=} \EE [  \bar{\bbeta}^{\top} \bZ ( \alpha +  \bbeta^{\top} \bmu  Y + \bbeta^{\top} \bZ )^3 ]  \bar{\bbeta} 
	-  \EE [  \bar{\bbeta}^{\top} \bZ (\alpha + \bbeta^{\top} \bmu  Y + \bbeta^{\top} \bZ ) ]  \bar{\bbeta} \notag \\
	&  = 3\alpha^2\EE [  \bar{\bbeta}^{\top} \bZ(\bbeta^\top \bmu Y + \bbeta^\top \bZ)] \bar \bbeta  + \EE [\bar \bbeta^\top\bZ (\bbeta^{\top} \bmu  Y + \bbeta^{\top} \bZ )^3 ]  \bar{\bbeta} 
	-  \bbeta \notag \\
	&  \overset{ }{=} ( 3\alpha^2  - 1) \bbeta + 3 \EE (  \bbeta^{\top} \bmu  Y)^2  \bbeta
	+ \EE [ \bar{\bbeta}^{\top} \bZ  ( \bbeta^{\top} \bZ )^3 ]  \bar{\bbeta} \notag \\
	&  \overset{}{=} [ 3\alpha^2 + 3 ( \bmu^{\top}  \bbeta )^2 + M_Z \| \bbeta \|_2^2 - 1 ] \bbeta, 
	\end{align*} 
	where besides the arguments we have been using we also employed identities $\| \bbeta \|_2 \bar\bbeta = \bbeta$ and $\EE(\bgamma^{\top} \bZ)^4 = M_Z$ for any unit-norm $\bgamma $. 
	Combining all these together, we get
	\begin{align}
	\nabla_\alpha L_{\lambda}^h(\alpha, \bbeta) & =  \alpha(\alpha^2 +3 (\bbeta^\top \bmu)^2 + 3\|\bbeta\|^2 + \lambda - 1) ,  \label{eqn-1129-gradient-alpha}\\
	\nabla_{\bbeta} L_{\lambda}^h(\alpha, \bbeta) & = [ 3\alpha^2 + (\bbeta^\top \bmu)^2 +3 \| \bbeta \|_2^2 - 1 ]  
	( \bmu^{\top} \bbeta ) \bmu
	+ [ 3\alpha^2 + 3 ( \bmu^{\top}  \bbeta )^2 + M_Z \| \bbeta \|_2^2 - 1 ] \bbeta .
	\label{eqn-1129-gradient-beta}
	\end{align}
	Taking second derivatives,
	\begin{align}
	\nabla_{\alpha \alpha}^2 L_{\lambda}^h(\alpha, \bbeta) & = 3\alpha^2 + 3(\bbeta^\top \bmu)^2 + 3\|\bbeta\|_2^2 + \lambda - 1, \label{eqn-1129-hessian-1} \\
	\nabla_{\bbeta \alpha}^2 L_{\lambda}^h(\alpha, \bbeta) & = 6\alpha [ (\bbeta^\top \bmu) \bmu + \bbeta ] ,
	\label{eqn-1129-hessian-2} \\
	\nabla_{\bbeta \bbeta}^2 L_{\lambda}^h(\alpha, \bbeta) & = 3 ( \bbeta^{\top} \bmu )^2  \bmu \bmu^{\top}  +
	( 3\alpha^2 + 3 \| \bbeta \|_2^2 - 1 ) \bmu\bmu^{\top} + 6 \bmu\bmu^{\top} \bbeta \bbeta^{\top} 
		\notag \\
	& \qquad \qquad 
	+ [ 3\alpha^2 + 3 ( \bmu^{\top}  \bbeta )^2 + M_Z \| \bbeta \|_2^2 - 1 ] \bI
	+  \bbeta [ 6 (\bmu^{\top} \bbeta)  \bmu^{\top}+ 2 M_Z \bbeta^{\top} ]  \notag \\
	& = 
	[ 3 \alpha^2 + 3 ( \bmu^{\top}  \bbeta )^2 + M_Z \| \bbeta \|_2^2 - 1 ] \bI
	+
	[ 3 \alpha^2 +
	3 ( \bbeta^{\top} \bmu )^2 
	+ (3 \| \bbeta \|_2^2 - 1)
	]
	\bmu \bmu^{\top}
	\notag \\& \qquad \qquad 
	+ 6 (\bmu^{\top} \bbeta)  ( \bmu \bbeta^{\top} +  \bbeta \bmu^{\top} )
	+ 2 M_Z  \bbeta  \bbeta^{\top}.
	\label{eqn-1129-hessian}
	\end{align}
	
	Now that we have derived the gradient and Hessian in closed form, we will characterize the lanscape. Let $(\alpha, \bbeta)$ be an arbitrary stationary point, we start by proving that it must satisfy $\alpha = 0$.
	\begin{claim}
	If $\lambda \geq 1$ then $\alpha = 0$ holds for any critical point $(\alpha, \bbeta)$. 
	\end{claim}
	\begin{proof}
		Seeking a contradiction assume that $\alpha \neq 0$. We start by assuming $\bbeta = c \bmu$ for some $c \in \RR$, then the optimality condition $\nabla_\alpha L_{\lambda}^h(\alpha , \bbeta) = 0$ gives $0 < \alpha^2 + 3 c^2 \|\bmu\|_2^2\left(\|\bmu\|_2^2 + 1\right) = 1 - \lambda \leq 0,$ yielding a contraction.
		
		Now, let us assume that $\bmu$ and $\bbeta$ are linearly independent, this assumption together with \eqref{eqn-1129-gradient-alpha} and \eqref{eqn-1129-gradient-beta} imply that
		\begin{align}
		\alpha^2 +3 (\bbeta^\top \bmu)^2 + 3\|\bbeta\|_2^2 + \lambda - 1 & = 0 , \notag \\
		[ 3\alpha^2 + (\bbeta^\top \bmu)^2+3 \| \bbeta \|_2^2 - 1 ]  
		\bmu^{\top} \bbeta  & = 0 , \notag \\
		3\alpha^2 + 3 ( \bmu^{\top}  \bbeta )^2 + M_Z \| \bbeta \|_2^2 - 1 & = 0. \label{eqn-1129-opt-ind}
		\end{align}
		There are only two possible cases:
		\begin{enumerate}
			\item[] \textit{Case 1}. If $\bbeta^\top \bmu = 0$, then the optimality condition for $\alpha$ gives
			$\alpha^2+3\|\bbeta\|_2^2 = 1-\lambda \leq 0$, which is a contradiction.
			\item[] \textit{Case 2}. If $\bbeta^\top \bmu \neq 0$, then $ 3\alpha^2 + (\bbeta^\top \bmu)^2+3 \| \bbeta \|_2^2 - 1 = 0$ and by substracting it from \eqref{eqn-1129-opt-ind} we get $0 < 2(\bbeta^\top \bmu)^2+(M_Z - 3)\|\bbeta\|_2^2 = 0$,
			yielding a contradiction again.
		\end{enumerate}
		This completes the proof of the claim.
	\end{proof}
	
	This claim directly implies that the Hessian $\nabla^2 L_{\lambda}^h$, evaluated at any critical point, is a block diagonal matrix with $\nabla_{\bbeta \alpha}^2 L_{\lambda}^h(\alpha, \bbeta) = 0$. Furthermore its first block is positive if $\bbeta \neq \mathbf{0}$, as
	\[
	\nabla_{\alpha \alpha}^2 L_{\lambda}^h(\alpha, \bbeta) = 3(\bbeta^\top \bmu )^2 + 3\|\bbeta\|_2^2 + \lambda - 1 > \lambda - 1 \geq 0.
	\]
	To prove the results regarding second order information at the critical points, it suffices to look at $\nabla_{\bbeta \bbeta} L_{\lambda}^h(\alpha, \bbeta)$. 
	
	Following a similar strategy to the one we used for the claim, let us start by assuming that $\bbeta$ and $\bmu$ are linearly independent. Then, (\ref{eqn-1129-gradient-beta}) yields
	\begin{align}
	&  [ ( \bbeta^{\top} \bmu )^2 + 3 \| \bbeta \|_2^2 - 1 ] (\bmu^{\top} \bbeta) = 0 , 
	\label{eqn-1129-gradient-1}\\
	& 3 ( \bmu^{\top}  \bbeta )^2 + M_Z \| \bbeta \|_2^2 - 1  = 0.
	\label{eqn-1129-gradient-2}
	\end{align}
	Consider two cases:
	\begin{enumerate} 
		\item[]  \textit{Case 1}. If $\bmu^{\top} \bbeta = 0$, then (\ref{eqn-1129-gradient-2}) yields $\| \bbeta \|_2^2 = 1 / M_Z$ and $(0,\bbeta) \in S_2^h$.
		\item[]  \textit{Case 2}. If $\bmu^{\top} \bbeta \neq 0$, then (\ref{eqn-1129-gradient-1}) forces $ ( \bbeta^{\top} \bmu )^2  +  3 \| \bbeta \|_2^2 - 1 = 0$. Since $M_Z > 3$, this equation and (\ref{eqn-1129-gradient-2}) force $\bbeta = \mathbf{0}$ and $\bmu^{\top} \bbeta = 0$, which leads to contradiction.
	\end{enumerate}
	
	Therefore, $S_2^h \backslash \{ \mathbf{0} \}$ is the collection of all critical points that are linearly independent of $(0,\bmu)$. For any $(0,\bbeta) \in S_2^h  \backslash \{ \mathbf{0} \}$, we have
	\begin{align}
& \nabla^2_{\bbeta \bbeta} L_{\lambda}^h (0,\bbeta)  = 
(3 \| \bbeta \|_2^2 - 1) \bmu \bmu^{\top}
+ 2 M_Z  \bbeta  \bbeta^{\top}, \notag \\
& \bmu^{\top} \nabla_{\bbeta \bbeta}^2 L_{\lambda}^h ( 0 , \bbeta) \bmu = (3 \| \bbeta \|_2^2 - 1) \| \bmu \|_2^4= - (1 - 3 / M_Z) \| \bmu \|_2^4, \notag \\
&\bu^{\top} \nabla^2 L_{\lambda}^h (0 , \bbeta) \bu \leq - (1 - 3 / M_Z) \| \bmu \|_2^2 < 0,
\label{eqn-thm:landscape-general-quartic-1}
	\end{align}
	where $\bu = (0, \bmu / \| \bmu \|_2)$. Hence the points in $S_2^h \backslash \{ \mathbf{0} \}$ are strict saddles.

	Now, suppose that $\bbeta = c \bmu$ and $\nabla L_\lambda^h(0, \bbeta ) = \mathbf{0}$. By (\ref{eqn-1129-gradient-beta}),
	\begin{align*}
	\nabla L_{\lambda}^h(0, \bbeta) &= [ ( c \| \bmu \|_2^2 )^3 + ( 3 c^2 \| \bmu \|_2^2 - 1 )  
	c  \| \bmu \|_2^2 ] \bmu
	+ [ 3 ( c  \| \bmu \|_2^2 )^2 + M_Z c^2 \| \bmu \|_2^2 - 1 ] c \bmu \\
	& = [
	c^2 \| \bmu \|_2^6 + ( 3 c^2 \| \bmu \|_2^2 - 1 )  
	\| \bmu \|_2^2
	+ 3  c^2  \| \bmu \|_2^4 + M_Z c^2 \| \bmu \|_2^2 - 1
	] c \bmu \\
	& = [ (\| \bmu \|_2^4 + 6 \| \bmu \|_2^2 + M_Z ) \| \bmu \|_2^2 c^2 - (\| \bmu \|_2^2 + 1) ] c \bmu.
	\end{align*}
	It is easily seen that $\nabla L_{\lambda}^h(\mathbf{0}) = \mathbf{0}$. If $c \neq 0$, then
	\begin{align}
	(\| \bmu \|_2^4 + 6 \| \bmu \|_2^2 + M_Z ) \| \bmu \|_2^2 c^2 = \| \bmu \|_2^2 + 1.
	\label{eqn-1129-gradient-3}
	\end{align}
	Hence $S_1^h \cup \{ \mathbf{0} \}$ is the collection of critical points that live in $\mathrm{span} \{ (0, \bmu) \}$, and $S_1^h \cup S_2^h$ contains all critical points of $L_\lambda^h$.
	
	We first investigate $\{ \mathbf{0} \}$. On the one hand, 
	\begin{align}
	\nabla_{\bbeta \bbeta}^2 L_{\lambda}^h(\mathbf{0}) = - ( \bI + \bmu \bmu^{\top} ) \prec 0.
		\label{eqn-thm:landscape-general-quartic-2}
	\end{align}
	On the other hand,
	\begin{align*}
& L_{\lambda}^h (\alpha , \mathbf{0}) = h(\alpha) + \frac{\lambda}{2} \alpha^2 = \frac{1}{4} (\alpha^2 - 1)^2 + \frac{\lambda}{2} \alpha^2 , \\
 & \nabla_{\alpha} L_{\lambda}^h (\alpha , \mathbf{0}) = \alpha^3 + (\lambda - 1) \alpha = \alpha ( \alpha^2 + \lambda - 1 ).
 	\end{align*}
 	It follows from $\lambda \geq 1$ that $0$ is a local minimum of $L_{\lambda}^h (\cdot, \mathbf{0})$. Thus $\mathbf{0}$ is a saddle point of $L_{\lambda}^h$ whose Hessian has negative eigenvalues.
 	
 	Next, for $(0,\bbeta) \in S_1$, we derive from (\ref{eqn-1129-hessian}) that
	\begin{align*}
	\nabla_{\bbeta \bbeta}^2 L_{\lambda}^h (0, \bbeta) & = 
	[ 3 ( c \| \bmu \|_2^2 )^2 + M_Z c^2 \| \bmu \|_2^2 - 1 ] \bI
	+
	[
	3 ( c \| \bmu \|_2^2 )^2
	+ 3 c^2 \| \bmu \|_2^2 - 1
	]
                                                             \bmu \bmu^{\top}\\
          & \hspace*{3cm}
	+ 6 c \| \bmu \|_2^2 \cdot 2 c \bmu \bmu^{\top}
	+ 2 M_Z c^2 \bmu \bmu^{\top} \\
	&= [ ( 3 \| \bmu \|_2^2  + M_Z ) c^2 \| \bmu \|_2^2 - 1] \bI
	+ [ ( 3 \| \bmu \|_2^4 + 15 \| \bmu \|_2^2 + 2 M_Z ) c^2 - 1 ] \bmu \bmu^{\top}.
	\end{align*}
	From (\ref{eqn-1129-gradient-3}) we see that
	\begin{align*}
	& ( 3 \| \bmu \|_2^2  + M_Z ) c^2 \| \bmu \|_2^2 - 1 =  \frac{  ( 3 \| \bmu \|_2^2  + M_Z )  ( \| \bmu \|_2^2 + 1 ) }{
		\| \bmu \|_2^4  + 6 \| \bmu \|_2^2 + M_Z 
	} - 1
	= \frac{  2 \| \bmu \|_2^4  + (M_Z - 3 ) \| \bmu \|_2^2  }{
		\| \bmu \|_2^4 + 6 \| \bmu \|_2^2 + M_Z 
	} > 0 ,\\
	&  ( 3 \| \bmu \|_2^4  + 15 \| \bmu \|_2^2 + 2 M_Z ) c^2 - 1 
	\geq  2 ( \| \bmu \|_2^4  + 6 \| \bmu \|_2^2 +  M_Z ) c^2 - 1
	= \frac{ 2( \| \bmu \|_2^2 + 1 ) }{ \| \bmu \|_2^2 } - 1 > 0.
	\end{align*}
	Hence both points in $S_1$ are local minima because
	\begin{align}
	\nabla_{\bbeta \bbeta}^2 L_{\lambda}^h (0,\bbeta)  \succeq \frac{  2 \| \bmu \|_2^4  + (M_Z - 3 ) \| \bmu \|_2^2  }{
		\| \bmu \|_2^4 + 6 \| \bmu \|_2^2 + M_Z 
	} \bI \succ 0,\qquad \forall (0, \bbeta ) \in S_1,
	\label{eqn-thm:landscape-general-quartic-3}
	\end{align}
	which immediately implies global optimality and finishes the proof.
	
	\subsubsection{Part 2: Quantitative properties of the landscape}

	\begin{enumerate}

		\item Lemma \ref{lemma-perturbation-population} implies that we can choose a sufficiently small constant $\varepsilon_1^h>0$ and a constant $R>0$ correspondingly such that $\Vert\nabla L^{h}_\lambda(\bm{\gamma})\Vert_{2}\geq\varepsilon_{1}^h$
		when $\Vert\bm{\gamma}\Vert_{2} \geq R$.
		Without loss of generality, we can always take $\delta^h \leq 1$ and $R > 1 + \max_{ \bgamma \in S_1^h \cup S_2^h  } \| \bgamma \|_2$. In doing so, we have
		\[
		\cS = \{ \bm{\gamma}:~\| \bgamma \|_2 \leq R, ~ \mathrm{dist} (\bm{\gamma},S_{1}^h\cup S_{2}^h  )\geq\delta^h \} \neq \varnothing.
		\]
		 

We now establish a lower bound for $\inf_{\bgamma \in \cS}\| \nabla L_{\lambda}^h (\bgamma) \|_2$. Define
		 \begin{align*}
& \cS_{\bbeta } = \mathrm{span}\left\{ 
(0, \bmu), (0, \bbeta), (1, \mathbf{0}) \right\} \cap \cS , \qquad \forall \bbeta \perp \bmu , \\
&  \varepsilon_{\bbeta} = \inf_{\bm{\gamma}\in \cS_{\bbeta } }\big\Vert \nabla L_\lambda^{h}\left(\bm{\gamma}\right)\big\Vert _{2}  .
		 \end{align*}
By symmetry, $\varepsilon_{\bbeta}$ is the same for all $\bbeta \perp \bmu$. Denote this quantity by $\varepsilon_2^h$. Since $\cS = \cup_{\bbeta \perp \bmu} \cS_{\bbeta}$,
\begin{align*}
\inf_{\bgamma \in \cS}\| \nabla L_{\lambda}^h (\bgamma) \|_2 
= \inf_{\bbeta \perp \bmu}  \inf_{\bgamma \in \cS_{\bbeta} }\| \nabla L_{\lambda}^h (\bgamma) \|_2
= \inf_{\bbeta \perp \bmu} \varepsilon_{\bbeta} = \varepsilon_2^h.
\end{align*} 

Take any $\bbeta \perp \bmu$. On the one hand, the nonnegative function $\Vert\nabla L_\lambda^h (\cdot) \Vert_2$ is continuous and its zeros are all in $S_{1}^h\cup S_{2}^h$. On the other hand, $\cS_{\bbeta }$ is compact and non-empty. Hence $\varepsilon_2^h =  \varepsilon_{\bbeta} > 0$ and it only depends on the function $L_{\lambda}^h$ restricted to a three-dimensional subspace, i.e. $\mathrm{span}\left\{ 
(0, \bmu), (0, \bbeta), (1, \mathbf{0}) \right\} $. It is then straightforward to check using the quartic expression of $L_{\lambda}^h$ and symmetry that $ \varepsilon_2^h$ is completely determined by $\| \bmu\|_2$, $M_Z$, $\lambda$ and $\delta^h$. From now on we write $\varepsilon_2^h(\delta^h)$ to emphasize its dependence on $\delta^h$, whose value remains to be determined.

To sum up, when $\delta^h \leq 1$ and $\varepsilon^h \leq \min \{ \varepsilon_1^h , \varepsilon_2^h(\delta^h) \}$, we have the desired result in the first claim.
		

		\item Given properties (\ref{eqn-thm:landscape-general-quartic-1}), (\ref{eqn-thm:landscape-general-quartic-2}) and  (\ref{eqn-thm:landscape-general-quartic-3}) of Hessians at all critical points, it suffices to show that
		\begin{align}
		\| \nabla^{2}L_{\lambda}^{h}(\bgamma_1) - \nabla^{2}L_{\lambda}^{h} (\bgamma_2) \|_2 \leq C' \| \bgamma_1 - \bgamma_2 \|_2,\qquad \forall \bgamma_1, ~\bgamma_2 \in B(\bm{0},R)
		\label{eqn-thm:landscape-general-quartic-4}
		\end{align}
		holds for some constant $C'$ determined by $\| \bmu \|_2$ and $R$. In that case, we can take sufficiently small $\delta^h$ and $\eta^h$ to finish the proof.
		
		Based on (\ref{eqn-1129-hessian-1}), (\ref{eqn-1129-hessian-2}) and (\ref{eqn-1129-hessian}), we first decompose $\nabla^{2}L_\lambda^{h}(\bm{\gamma})$ into
		the sum of two matrices $\bm{I}(\bm{\gamma})$ and $\bm{J}(\bgamma)$
		:
\begin{align*}
	\nabla^{2}L_\lambda^{h}\left(\bm{\gamma}\right)
	&=
	\begin{pmatrix}
	3\alpha^{2}+3\left(\bm{\beta}^{\top}\bm{\mu}\right)^{2}+3\left\Vert \bm{\beta}\right\Vert _{2}^{2}+\lambda-1 & 6\alpha\left[\left(\bm{\beta}^{\top}\bm{\mu}\right)\bm{\mu}+\bm{\beta}\right]^{\top}\\
	6\alpha\left[\left(\bm{\beta}^{\top}\bm{\mu}\right)\bm{\mu}+\bm{\beta}\right] & 3\alpha^{2}\left(\bm{I}+\bm{\mu}\bm{\mu}^{\top}\right)
	\end{pmatrix}
	\\
	&\quad+\begin{pmatrix}
	0 & \bm{0}^{\top}\\
\bm{0} & \nabla_{\bm{\beta}\bm{\beta}}^{2}L^{h}\left(\bm{\gamma}\right)-3\alpha^{2}\left(\bm{I}+\bm{\mu}\bm{\mu}^{\top}\right)
	\end{pmatrix}
\\
& = \bm{I}\left(\bm{\gamma}\right) + \bm{J}\left( \bgamma \right).
\end{align*}
		For any $\bm{\gamma}_{1}=(\alpha_{1},\bm{\beta}_{1}),\bm{\gamma}_{2}=(\alpha_{2},\bm{\beta}_{2}) \in B (\bm{0},R)$,
		we have 
		\begin{align*}
		\left\Vert \bm{I}\left(\bm{\gamma}_{1}\right)-\bm{I}\left(\bm{\gamma}_{2}\right)\right\Vert_2  & \leq\left|3\alpha_{1}^{2}+3\left(\bm{\beta}_{1}^{\top}\bm{\mu}\right)^{2}+3\left\Vert \bm{\beta}_{1}\right\Vert _{2}^{2}-3\alpha_{2}^{2}-3\left(\bm{\beta}_{2}^{\top}\bm{\mu}\right)^{2}-3\left\Vert \bm{\beta}_{2}\right\Vert _{2}^{2}\right|\\
		& \quad+2\left\Vert 6\alpha_{1}\left[\left(\bm{\beta}_{1}^{\top}\bm{\mu}\right)\bm{\mu}+\bm{\beta}_{1}\right]-6\alpha_{2}\left[\left(\bm{\beta}_{2}^{\top}\bm{\mu}\right)\bm{\mu}+\bm{\beta}_{2}\right]\right\Vert_2 \\
		&\quad+\left\Vert 3\left(\alpha_{1}^{2}-\alpha_{2}^{2}\right)\left(\bm{I}+\bm{\mu}\bm{\mu}^{\top}\right)\right\Vert_2 .
		\end{align*}
		Let $\Delta=\Vert\bm{\gamma}_{1}-\bm{\gamma}_{2}\Vert_{2}$ and note that $\vert\alpha_{1}^{2}-\alpha_{2}^{2}\vert\leq2R\Delta$,
		$\vert\Vert\bm{\beta}_{1}\Vert_{2}^{2}-\Vert\bm{\beta}_{2}\Vert_{2}^{2}\vert\leq2R\Delta$,
		$\vert(\bm{\beta}_{1}^{\top}\bm{\mu})^{2}-(\bm{\beta}_{2}^{\top}\bm{\mu})^{2}\vert\leq2R\Vert\bm{\mu}\Vert_{2}^{2}\Delta$,
		$\Vert\alpha_{1}\bm{\beta}_{1}-\alpha_{2}\bm{\beta}_{2}\Vert_{2}\leq2R\Delta$
		and $\vert\alpha_{1}(\bm{\beta}_{1}^{\top}\bm{\mu})-\alpha_{2}(\bm{\beta}_{2}^{\top}\bm{\mu})\vert\leq2R\Vert\bm{\mu}\Vert_{2}\Delta$,
		we immediately have
		\[
		\Vert \bm{I}(\bm{\gamma}_{1})-\bm{I}(\bm{\gamma}_{2})\Vert_2 \lesssim (1 + \left\Vert \bm{\mu}\right\Vert _{2}+\left\Vert \bm{\mu}\right\Vert _{2}^{2})R \| \bgamma_1 - \bgamma_2 \|_2.
		\]
		
		According to (\ref{eqn-1129-hessian}), $\bm{J}(\bgamma)$ depends on $\bbeta$ but not $\alpha$. Moreover, we have the following decomposition for its bottom right block:
		\begin{align*}
&\underbrace{\left[3\left(\bm{\mu}^{\top}\bm{\beta}\right)^{2} + M_Z \left\Vert \bm{\beta}\right\Vert _{2}^{2}-1\right]\bm{I}}_{ \bm{J}_{1}\left(\bm{\beta}\right)}+\underbrace{\left[3\left(\bm{\beta}^{\top}\bm{\mu}\right)^{2} +\left(3\left\Vert \bm{\beta}\right\Vert _{2}^{2}-1\right)\right]\bm{\mu}\bm{\mu}^{\top}}_{ \bm{J}_{2}\left(\bm{\beta}\right)}\\
 &\hspace{2cm}+\underbrace{6\left(\bm{\mu}^{\top}\bm{\beta}\right)\left(\bm{\mu}\bm{\beta}^{\top}+\bm{\beta}^{\top}\bm{\mu}\right)}_{ \bm{J}_{3}\left(\bm{\beta}\right)}+\underbrace{2 M_Z \bm{\beta}\bm{\beta}^{\top}}_{ \bm{J}_{4}\left(\bm{\beta}\right)}.
		\end{align*}
		Similar argument gives $\Vert\bm{J}_{1}(\bm{\beta}_{1})-\bm{J}_{1}(\bm{\beta}_{2})\Vert\lesssim (\Vert\bm{\mu}\Vert_{2}^{2}+ M_Z )R\Delta$,
		$\Vert\bm{J}_{2}(\bm{\beta}_{1})-\bm{J}_{2}(\bm{\beta}_{2})\Vert_2 \lesssim (\Vert\bm{\mu}\Vert_{2}^{4} + \Vert\bm{\mu}\Vert_{2}^{2})R\Delta$,
		$\Vert\bm{J}_{3}(\bm{\beta}_{1})-\bm{J}_{3}(\bm{\beta}_{2})\Vert_2 \lesssim \Vert\bm{\mu}\Vert_{2}^{2}R\Delta$
		and $\Vert\bm{J}_{4}(\bm{\beta}_{1})-\bm{J}_{4}(\bm{\beta}_{2})\Vert_2 \lesssim M_Z R\Delta$.
		As a result, we have
		\[
		\Vert \bm{J} (\bgamma_{1} )-\bm{J} (\bgamma_{2} ) \Vert_2 \lesssim (\left\Vert \bm{\mu}\right\Vert _{2}^{2}+ \left\Vert \bm{\mu}\right\Vert_{2}^{4}  + M_Z )R \| \bgamma_1 - \bgamma_2 \|_2.
		\]
		Hence we finally get (\ref{eqn-thm:landscape-general-quartic-4}). 
		
	\end{enumerate}

\subsection{Proof of Lemma \ref{lemma-perturbation-population}}\label{appendix:perturbation-population}
By definition, $\nabla L_\lambda \left(\bm{\gamma}\right)-\nabla L_\lambda^h \left(\bm{\gamma}\right) =\mathbb{E}\left({\bm{X}}\left[f^{\prime}\left(\bm{\gamma}^{\top}{\bm{X}}\right)-h^{\prime}\left(\bm{\gamma}^{\top}{\bm{X}}\right)\right]\right)$. 
From Lemma \ref{lem:f-property} we obtain that $|f'(x) - h'(x)| \lesssim | x |^{3} \mathbf{1}_{ \{ | x |\geq a \} }$ when $b\geq2a$ and $a$ is sufficiently large. When $\| \bgamma \|_2 \leq R$, we have
\begin{align*}
	\big\Vert \nabla L_\lambda\left(\bm{\gamma}\right)-\nabla L_\lambda^h \left(\bm{\gamma}\right)\big\Vert _{2} & =\sup_{\bm{u}\in\SSS^{d}}\mathbb{E}\left(\bm{u}^{\top}{\bm{X}}\big[f^{\prime}\big(\bm{\gamma}^{\top}{\bm{X}}\big)-h^{\prime}\big(\bm{\gamma}^{\top}{\bm{X}}\big)\big]\right)\\
	&\lesssim \sup_{\bm{u}\in\SSS^{d}}
	\mathbb{E}\left( \big|\bm{u}^{\top}{\bm{X}}\big|\big|\bm{\gamma}^{\top}{\bm{X}}\big|^{3}
	\mathbf{1}_{ \{ |\bgamma^{\top} {\bm{X}}| \geq a \} }
	\right)\\
	& \overset{\text{(i)}}{\lesssim}\sup_{\bm{u}\in\SSS^{d}} \mathbb{E}^{1/3} \big|\bm{u}^{\top}{\bm{X}}\big|^{3}
	\mathbb{E}^{1/3} \big|\bm{\gamma}^{\top}{\bm{X}}\big|^{9}
	\mathbb{P}^{1/3}\big(\big|\bm{\gamma}^{\top}{\bm{X}}\big|\geq a\big)\\
	& \overset{\text{(ii)}}{\lesssim}\sup_{\bm{u}\in\SSS^{d}}
	\big\Vert \bm{u}^{\top}{\bm{X}}\big\Vert _{\psi_{2}}
	\big\Vert \bm{\gamma}^{\top}{\bm{X}}\big\Vert _{\psi_{2}}^{3}
	\exp\left(-\frac{C_1 a^{2}}{\left\Vert \bm{\gamma}^{\top}{\bm{X}}\right\Vert _{\psi_{2}}^{2}}\right)\\
	&\overset{\text{(iii)}}{\leq} R^{3}M^{4}\exp\left(-\frac{C_1 a^{2}}{R^{2}M^{2}}\right)
\end{align*}
for some constant $C_1>0$. Here (i) uses H\"older's inequality, (ii) comes from sub-Gaussian property \citep{Ver10}, and
(iii) uses $\| \bv^{\top} {\bX} \|_{\psi_2} \leq \| \bv \|_2 \| {\bm{X}} \|_{\psi_2} = \| \bv \|_2 M$, $\forall \bv \in \RR^{d+1}$.

To study the Hessian, we start from $\nabla^{2}L_\lambda \left(\bm{\gamma}\right)-\nabla^{2}L_\lambda^h \left(\bm{\gamma}\right)  =\mathbb{E}\left({\bm{X}}{\bm{X}}^{\top}\left[f^{\prime\prime}\left(\bm{\gamma}^{\top}{\bm{X}}\right)-h^{\prime\prime}\left(\bm{\gamma}^{\top}{\bm{X}}\right)\right]\right)$.
Again from Lemma \ref{lem:f-property} we know that $|f''(x) - h''(x)| \lesssim x^2 \mathbf{1}_{ \{ |x| \geq a \} }$. When $\| \bgamma \|_2 \leq R$, we have
\begin{align*}
	\big\Vert \nabla^{2}L_\lambda\left(\bm{\gamma}\right)-\nabla^{2}L_\lambda^h \left(\bm{\gamma}\right)\big\Vert_2  & =\sup_{\bm{u}\in\SSS^{d}}\bm{u}^{\top}\mathbb{E}\left({\bm{X}}{\bm{X}}^{\top}\big[f^{\prime\prime}\big(\bm{\gamma}^{\top}{\bm{X}}\big)-h^{\prime\prime}\big(\bm{\gamma}^{\top}{\bm{X}}\big)\big]\right)\bm{u}
	\\ &\lesssim
	\sup_{\bm{u}\in\SSS^{d}} \mathbb{E}\left(\big|\bm{u}^{\top}{\bm{X}}\big|^{2}\big|\bm{\gamma}^{\top}{\bm{X}}\big|^{2}
	\mathbf{1}_{ \{ |\bgamma^{\top} {\bX}| \geq a \} }
	\right)\\
	& \lesssim \sup_{\bm{u}\in\SSS^{d}}
	\mathbb{E}^{1/3} \big|\bm{u}^{\top}{\bm{X}}\big|^{6}
	\mathbb{E}^{1/3} \big|\bm{\gamma}^{\top}{\bm{X}}\big|^{6}
	\mathbb{P}^{1/3}\big(\big|\bm{\gamma}^{\top}{\bm{X}}\big|\geq a\big)\\
	& \lesssim \sup_{\bm{u}\in\SSS^{d}}
	\big\Vert \bm{u}^{\top}{\bm{X}}\big\Vert _{\psi_{2}}^{2}
	\big\Vert \bm{\gamma}^{\top}{\bm{X}}\big\Vert _{\psi_{2}}^{2}
	\exp\left(-\frac{C_1 a^{2}}{\left\Vert \bm{\gamma}^{\top}{\bm{X}}\right\Vert _{\psi_{2}}^{2}}\right)
	\\&\leq
	R^{2}M^{4}\exp\left(-\frac{C_1 a^{2}}{R^{2}M^{2}}\right)
\end{align*}
for some constant $C_1 >0$.

We finally work on the lower bound for $\| \nabla L_\lambda(\bgamma) \|_2$. From $b \geq 2 a \geq 4$ we get $f(x) = h(x)$ for $|x| \leq a$; $f'(x) \geq 0$ and $f''(x) \geq 0$ for all $x \geq 1$. Since $f'$ is odd,
\begin{align*}
	& \inf_{ x \in \RR }xf'(x) = \inf_{ |x| \leq 1 } xf'(x) = \inf_{ |x| \leq 1 } xh'(x) = \inf_{ |x| \leq 1 } \{ x^4 - x^2 \} \geq - 1 , \\
	& \inf_{ |x| \geq 2 } f'(x) \sgn(x)
	= \inf_{ x \geq 2 } f'(x) \geq f'(2) = h'(2) = 2^3 - 2 = 6.
\end{align*}
Taking $a = 2$, $b = 1$ and $c = 6$ in Lemma \ref{lemma-Hessian-Lip-0}, we get
\begin{align*}
	\| L_\lambda(\bgamma) \|_2 \geq 6 \inf_{\bu \in \SSS^{d}} \EE | \bu^{\top}{\bX} | - \frac{12 + 1}{\| \bgamma \|_2}
	\geq 6 \varphi ( \| {\bX} \|_{\psi_2}, \lambda_{\min} [ \EE ( {\bX} {\bX}^{\top} ) ] ) - \frac{13}{\| \bgamma \|_2}
	\geq 6 \varphi ( M, \sigma^2 ) - \frac{13}{\| \bgamma \|_2}
\end{align*}
for $\bgamma \neq \mathbf{0}$. Here $\varphi$ is the function in Lemma \ref{lemma-L1-lower}. If we let $m = \varphi ( M, \sigma^2 ) $, then $\inf_{ \| \bgamma \|_2 \geq 3 / m } \| L_\lambda(\bgamma) \|_2 \geq m$. Follow a similar argument, we can show that $\inf_{ \| \bgamma \|_2 \geq 3 / m } \| L^h_\lambda(\bgamma) \|_2 \geq m$ also holds for the same $m$.
	
\subsection{Proof of Lemma \ref{lemma-population-existence}}\label{appendix:lemma-existence}
To prove the first part, we define $\bar{\bmu} = \bmu / \| \bmu \|_2$ and seek for $c > 0$ determined by $\| \bmu \|_2$, the function $f$, and the distribution of $Z$ such that $\nabla L_1 (0, \pm c \bar\bmu) = \mathbf{0}$.

By the chain rule, for any $(\alpha, \bbeta, t) \in \RR \times \RR^d \times \RR$ we have
\begin{align*}
\nabla L_\lambda(\alpha, \bbeta) = \begin{pmatrix}
\EE f' (\alpha + \bbeta^{\top} \bX) + \lambda\alpha \\
\EE [ \bX f' (\alpha + \bbeta^{\top} \bX) ]
\end{pmatrix}
\qquad \text{and} \qquad 
\nabla L_1(0, t \bar\bmu) = \begin{pmatrix}
\EE f' ( t \bar\bmu^{\top} \bX)  \\
\EE [ \bX f' ( t \bar\bmu^{\top} \bX) ]
\end{pmatrix}.
\end{align*}
Since $f$ is even, $f'$ is odd and $t \bar\bmu^{\top} \bX $ has symmetric distribution with respect to $0$, we have $\EE f' ( t \bar\bmu^{\top} \bX) = 0$. It follows from $(\bI - \bar\bmu \bar\bmu^{\top}) \bX = (\bI - \bar\bmu \bar\bmu^{\top}) \bZ$ that
\begin{align*}
(\bI - \bar\bmu \bar\bmu^{\top}) \EE [ \bX f' ( t \bar\bmu^{\top} \bX) ] = \EE [ (\bI - \bar\bmu \bar\bmu^{\top}) \bZ f' (t \bar\bmu^{\top} \bX) ] = \EE [ (\bI - \bar\bmu \bar\bmu^{\top}) \bZ f' (t \| \bmu \|_2 Y + t \bar\bmu^{\top} \bZ) ].
\end{align*}
Thanks to the independence between $Y$ and $\bZ$ as well as the spherical symmetry of $\bZ$, $(Y,  \bar\bmu^{\top} \bZ, (\bI - \bar\bmu \bar\bmu^{\top}) \bZ )$ and $(Y,  \bar\bmu^{\top} \bZ, - (\bI - \bar\bmu \bar\bmu^{\top}) \bZ )$ share the same distribution. Then
\begin{align*}
(\bI - \bar\bmu \bar\bmu^{\top}) \EE [ \bX f' ( t \bar\bmu^{\top} \bX) ] = \mathbf{0} 
\qquad \text{ and } \qquad \EE [ \bX f' ( t \bar\bmu^{\top} \bX) ] = \bar\bmu \bar\bmu^{\top}  \EE [ \bX f' ( t \bar\bmu^{\top} \bX) ].
\end{align*}
As a result,
\begin{align*}
\nabla L_\lambda(0, t \bar\bmu) =
\EE [ \bar\bmu^{\top}\bX f' ( t \bar\bmu^{\top} \bX) ] \begin{pmatrix}
0  \\
\bar\bmu
\end{pmatrix}.
\end{align*}

Define $W = \bar\bmu^{\top} \bX = \| \bmu \|_2 Y + \bar\bmu^{\top} \bZ$ and $\varphi(t) = \EE [W f'(tW)]$ for $t \in \RR$. The fact that $f$ is even yields $f'(0) = 0$ and $\varphi(0) = \EE [W f'(0)] = 0$. On the one hand, $f''(0) < 0$ forces
\begin{align}
\varphi'(0) = \EE [ W^2 f''(tW) ] |_{t = 0} = f''(0) \EE W^2 = f''(0) ( \| \bmu \|_2^2 + 1 ) < 0.
\label{ineq-lemma-population-existence-0}
\end{align}
Hence there exists $t_1 > 0$ such that $\varphi (t_1) < 0$. On the other hand, $\lim_{x \to +\infty} xf'(x) = +\infty$ leads to $\lim_{t \to +\infty} x \varphi(x) = \EE [tW f'(tW)] = + \infty$. Then there exists $t_2 > 0$ such that $\varphi(t_2) > 0$. By the continuity of $\varphi$, we can find some $c > 0$ such that $\varphi(c) = 0$. Consequently,
\begin{align*}
\nabla L_1(0, c \bar\bmu) =
\varphi(c) \begin{pmatrix}
0  \\
\bar\bmu
\end{pmatrix} = \mathbf{0}.
\end{align*} 
In addition, from
\begin{align*}
\varphi(-c) = \EE [W f'(-cW)] = - \EE [W f'(cW)] = - \varphi(c) = 0
\end{align*}
we get $\nabla L(0, -c \bar\bmu) = \mathbf{0}$. It is easily seen that $t_1$, $t_2$ and $c$ are purely determined by properties of $f$ and $W$, where the latter only depends on $\| \bmu \|_2$ and the distribution of $Z$. This finishes the first part.

To prove the second part, we first observe that
\begin{align*}
| \varphi''(t) | = | \EE [ W^3 f''' (tW) ] | \leq F_3 \EE |W|^3 = F_3 ( 3^{-1/2} \EE^{1/3} |W|^3 )^3 \cdot 3^{3/2}
\leq 3^{3/2} F_3 M, \qquad \forall t \in \RR.
\end{align*}
Let $c_0 =  - f''(0) ( \| \bmu \|_2^2 + 1 ) / (3^{3/2} F_3 M)$. In view of (\ref{ineq-lemma-population-existence-0}),
\begin{align*}
\varphi'(t) \leq \varphi'(0) + t \sup_{s \in \RR} |\varphi''(s)| \leq f''(0) ( \| \bmu \|_2^2 + 1 ) + 3^{3/2} F_3 M t < 0, \qquad \forall t \in [0, c_0).
\end{align*}
Thus $\varphi(t) < \varphi(0) = 0$ in the same interval, forcing $c > c_0$.

%% file: appendix_general_landscape_empirical.tex
\section{Proof of Theorem \ref{theorem-landscape-sample}}\label{appendix:theorem-landscape-sample}

It suffices to prove the bound on the exceptional probability for each claim. 
\begin{enumerate}
	\item Claim 1 can be derived from Lemma \ref{lemma-landscape-sample-gradient}, Theorem~\ref{thm:landscape-general} and concentration of gradients within a ball (cf.~Lemma~\ref{lemma:gradient-concentration}). 
	\begin{lemma}\label{lemma-landscape-sample-gradient}
	Let $\{ \bm{X}_i \}_{i=1}^n$ be i.i.d.~random vectors in $\RR^{d+1}$ with $\| \bm{X}_i \|_{\psi_2} \leq 1$ and $ \EE (\bm{X}_i \bm{X}_i^{\top}) \succeq \sigma^2 \bI$ for some $\sigma > 0$, $f$ be defined in (\ref{eqn-test-function}) with $b\geq 2a \geq 4$, and
	\begin{align*}
\hat L_\lambda(\bgamma) = \frac{1}{n} \sum_{i=1}^{n} f (\bgamma^{\top}\bX_i) + \frac{\lambda}{2} (\bgamma^\top \hat\bmu)^2
	\end{align*}
	with $\hat\bmu = \frac{1}{n}\sum_{i=1}^{n}\bX_i$ and $\lambda \geq 0$.
	There exist positive constants $C, C_1, C_2, R$ and $\varepsilon_1$ determined by $\sigma$ such that when $n/d \geq C$,
	\begin{align*}
	\PP \bigg( \inf_{ \| \bgamma \|_2 \geq R } \| \nabla \hat L_\lambda(\bgamma) \|_2 > \varepsilon_1 \bigg) > 1 - C_1 (d/n)^{C_2 d}.
	\end{align*}
\end{lemma}
	\begin{proof}
		See Appendix \ref{appendix:landscape-sample-gradient}.
	\end{proof}
Let $R$ and $\varepsilon$ be the constants stated in Lemma~\ref{lemma-landscape-sample-gradient} and Theorem~\ref{thm:landscape-general}, respectively. Lemma~\ref{lemma:gradient-concentration} asserts that 
	\begin{equation*}
	\mathbb{P}\left(\sup_{\bm{\gamma}\in B\left(\bm{0},R\right)}\bigl\Vert\nabla\hat{L}_\lambda\left(\bm{\gamma}\right)-\nabla L_\lambda\left(\bm{\gamma}\right)\bigr\Vert_{2} < \frac{\varepsilon}{2}\right)> 1 - C_1(d/n)^{C_2d}
	\end{equation*}
	for some constant $C_1,C_2>0$, provided that $n/d$ is large enough. From Theorem~\ref{thm:landscape-general} we know that $\Vert\nabla L_\lambda(\bgamma)\Vert_2\geq\varepsilon$ if $\mathrm{dist}(\bgamma,\{\pm\bgamma^\star\}\cup S)\geq\delta$. The triangle inequality immediately gives
	\[
	\PP \bigg( \inf_{\bgamma:~ \mathrm{dist}(\bgamma,\{\pm\bgamma^\star\}\cup S)\geq\delta} \Vert\nabla\hat{L}_\lambda(\bgamma)\Vert_2 > \varepsilon/2 \bigg) < 1 - C_1' (d/n)^{C_2' d},
	\]
for some constants $C_1'$ and $C_2'$.
	
	\item We invoke the following Lemma~\ref{lemma-landscape-sample} to prove Claim 2. 
	\begin{lemma}\label{lemma-landscape-sample}
		Let $\{ \bm{X}_i \}_{i=1}^n$ be i.i.d.~random vectors in $\RR^{d+1}$ with $\| \bm{X}_i \|_{\psi_2} \leq 1$; $\bu \in \SSS^{d}$ be deterministic; $R>0$ be a constant. Let $f$ be defined in (\ref{eqn-test-function}) with constants $b\geq 2a \geq 4$, and
		\begin{align*}
		\hat L_\lambda(\bgamma) = \frac{1}{n} \sum_{i=1}^{n} f (\bgamma^{\top}\bX_i) + \frac{\lambda}{2} (\bgamma^\top \hat\bmu)^2
		\end{align*}
		with $\hat\bmu = \frac{1}{n}\sum_{i=1}^{n}\bX_i$ and $\lambda \geq 0$.
		Suppose that $n / d \geq e$. There exist positive constants $C_1, C_2, C$ and $N$ such that when $n > N$,
		\begin{align*}
		&\PP \bigg( \sup_{ \bgamma_1 \neq \bgamma_2  } \frac{ \| \nabla \hat{L}_\lambda(\bgamma_1) - \nabla \hat{L}_\lambda (\bgamma_2) \|_2 }{ \| \bgamma_1 - \bgamma_2 \|_2 } <  C \bigg) > 1 - C_1 e^{-C_2 n }, \\
		&\PP \bigg( \sup_{ \bgamma_1 \neq \bgamma_2  } \frac{ \| \nabla^2 \hat{L}_\lambda (\bgamma_1) - \nabla^2 \hat{L}_\lambda (\bgamma_2) \|_2 }{ \| \bgamma_1 - \bgamma_2 \|_2 } < C \max\{ 1, d \log (n/d) / \sqrt{n} \} \bigg) > 1 - C_1 (d / n)^{C_2 d}, \\
		& \PP \bigg(
		\sup_{ \| \bgamma \|_2 \leq R } |  \bu^{\top} [ \nabla^2 \hat L_\lambda ( \bgamma ) - \nabla^2  L_\lambda ( \bgamma ) ] \bu | < C \sqrt{ d \log (n/d) / n }
		\bigg) > 1 - C_1 (d / n)^{C_2 d } - C_1 e^{- C_2 n^{1/3} }.
		\end{align*}
	\end{lemma}
	\begin{proof}
		See Appendix \ref{appendix:lemma-landscape-sample}.
	\end{proof}
	From Theorem~\ref{thm:landscape-general} we know that $\bu^\top\nabla^2L_\lambda(\bgamma)\bu\leq-\eta$ if $\mathrm{dist}(\bgamma,S)\leq\delta$.  Lemma~\ref{lemma-landscape-sample} (after proper rescaling) asserts that
	\begin{equation*}
	\PP \bigg(
	\sup_{ \| \bgamma \|_2 \leq R } |  \bu^{\top} [ \nabla^2 \hat L_\lambda ( \bgamma ) - \nabla^2  L_\lambda ( \bgamma ) ] \bu | < \frac{\eta}{2}
	\bigg) > 1 - C_1 (d / n)^{C_2 d } - C_1 e^{- C_2 n^{1/3} }
	\end{equation*}
	provided that $n/d$ is sufficiently large. Then Claim 2 follows from the triangle's inequality.
	\item Claim 3 follows from Lemma \ref{lemma-landscape-sample} with proper rescaling.
\end{enumerate}

\subsection{Proof of Lemma \ref{lemma-landscape-sample-gradient}}\label{appendix:landscape-sample-gradient}
It is shown in Lemma \ref{lemma-perturbation-population} that when $b \geq 2 a \geq 4$, we have $\inf_{ x \in \RR }xf'(x) \geq - 1 $ and $\inf_{ |x| \geq 2 } f'(x) \sgn(x) \geq 6$. Using an empirical version of Lemma \ref{lemma-Hessian-Lip-0},
\begin{align*}
\nabla\hat{L}_\lambda ( \bgamma ) \geq \inf_{\bu \in \SSS^{d}} \frac{1}{n} \sum_{i=1}^{n} |\bu^{\top} \bm{X}_i| - \frac{ 13 }{\| \bgamma \|_2} , \qquad \forall \bgamma \in \RR^d .
\end{align*}
Define $S_n(\bu) = \frac{1}{n} \sum_{i=1}^{n} ( |\bu^{\top} \bm{X}_i| -\EE |\bu^{\top} \bm{X}_i| )$ for $\bu \in \SSS^{d}$. By the triangle inequality,
\begin{align*}
\hat{L}_{\lambda} ( \bgamma ) \geq \inf_{\bu \in \SSS^{d}} \EE |\bu^{\top} \bm{X}_1| - 
\sup_{\bu \in \SSS^{d} } |S_n (\bu)| - \frac{ 13 }{\| \bgamma \|_2} , \qquad \forall \bgamma \in \RR^d .
\end{align*}

According to Lemma \ref{lemma-L1-lower}, $\inf_{\bu \in \SSS^{d}} \EE |\bu^{\top} \bm{X}_1| > \varphi$ for some constant $\varphi > 0$ determined by $\sigma$. Then it suffices to prove
\begin{align}
\sup_{\bu \in \SSS^{d} } |S_n(\bu)| = O_{\PP} ( \sqrt{ d \log (n/d) / n } ;~ d \log (n/d) ).
\label{lemma-landscape-sample-gradient-0}
\end{align}
We will use Theorem 1 in \cite{wang2019some} to get there.
\begin{enumerate}
	\item Since $\| \bm{X}_i \|_{\psi_2} \leq 1$, the Hoeffding-type inequality in Proposition 5.10 of \cite{Ver10} asserts the existence of a constant $c>0$ such that
	\begin{align*}
	\PP ( |S_n (\bu)| \geq t ) \leq e \cdot e^{- c n t^2}, \qquad \forall t \geq 0.
	\end{align*}
	Then $\{ S_n(\bu) \}_{\bu \in \SSS^{d}} = O_{\PP} ( \sqrt{ d \log (n/d) / n } ;~d \log (n/d) )$.
	\item Let $\varepsilon_n = \sqrt{d / n}$. According to Lemma 5.2 in \cite{Ver10}, there exists an $\varepsilon_n$-net $\cN_n$ of $\SSS^{d}$ with cardinality at most $(1 + 2 R / \varepsilon_n)^d$. When $n/d$ is large, $\log |\cN_n| = d \log (1 + \sqrt{n/d}) \lesssim d \log (n/d)$.
	\item Define $M_n = \sup_{ \bu \in \SSS^{d}, \bv \in \SSS^{d} , \bu \neq \bv  }  \{ | S_n(\bu) - S_n(\bv) | / \| \bu - \bv \|_2 \}$. By Cauchy-Schwarz inequality,
	\begin{align*}
	\bigg| 
	\frac{1}{n} \sum_{i=1}^{n} |\bu^{\top} \bm{X}_i| - \frac{1}{n} \sum_{i=1}^{n}  |\bv^{\top} \bm{X}_i|
	\bigg| 
	& \leq
	\frac{1}{n} \sum_{i=1}^{n} |(\bu - \bv )^{\top} \bm{X}_i|
	\leq \bigg( \frac{1}{n} \sum_{i=1}^{n} |(\bu - \bv )^{\top} \bm{X}_i|^2 \bigg)^{1/2} \\
	& \leq \| \bu - \bv \|_2 \sup_{\bw \in \SSS^{d}} \bigg( \frac{1}{n} \sum_{i=1}^{n} |\bw^{\top} \bm{X}_i|^2 \bigg)^{1/2} \\
	&= \| \bu - \bv \|_2 \cdot O_{\PP} (1;~n),
	\end{align*}
	where the last equality follows from Lemma \ref{lemma-cubic-sup}. Similarly,
	\begin{align*}
	\left|
	\EE |\bu^{\top} \bm{X}_1| - \EE |\bv^{\top} \bm{X}_1|
	\right| \leq \| \bu - \bv \|_2 \| \EE (\bm{X}_1 \bm{X}_1^{\top}) \|_2 \lesssim \| \bu - \bv \|_2.
	\end{align*}
	Hence $M_n = O_{\PP} ( 1 ;~ n)$.
\end{enumerate}
Then Theorem 1 in \cite{wang2019some} yields (\ref{lemma-landscape-sample-gradient-0}).

\subsection{Proof of Lemma \ref{lemma-landscape-sample}}\label{appendix:lemma-landscape-sample}
	It follows from Example 6 in \cite{wang2019some} that $\Vert n^{-1}\sum_{i=1}^{n}\bm{X}_{i}-\bm{\mu}_{0}\Vert_{2}=O_{\mathbb{P}}(1;~n)$. As a result $\Vert n^{-1}\sum_{i=1}^{n}\bm{X}_{i}\Vert_{2}=O_{\mathbb{P}}(1;~n)$.
	This combined with Lemma \ref{lemma-Hessian-Lip-0} and Lemma \ref{lemma-cubic-sup} gives
	\begin{align*}
	&\sup_{ \bgamma_1 \neq \bgamma_2  } \frac{ \| \nabla \hat{L}_\lambda (\bgamma_1) - \nabla \hat{L}_\lambda (\bgamma_2) \|_2 }{ \| \bgamma_1 - \bgamma_2 \|_2 } = O_{\PP} ( 1;~n), \\
	&\sup_{ \bgamma_1 \neq \bgamma_2  } \frac{ | \bu^{\top} [  \nabla^2 \hat L_\lambda ( \bgamma_1 ) - \nabla^2 \hat L_\lambda ( \bgamma_2 ) ] \bu | }{ \| \bgamma_1 - \bgamma_2 \|_2 } = O_{\PP} (1 ;~n^{1/3} ), \\
	&\sup_{ \bgamma_1 \neq \bgamma_2  } \frac{ \| \nabla^2 \hat{L}_\lambda (\bgamma_1) - \nabla^2 \hat{L}_\lambda (\bgamma_2) \|_2 }{ \| \bgamma_1 - \bgamma_2 \|_2 } = O_{\PP} ( \max\{ 1 , d \log (n/d) / \sqrt{n} \};~d \log (n/d) )
	\end{align*}
	given $F_2\leq 3a^2\lesssim 1$ and $F_3\leq 6a\lesssim 1$, provided that $n/d$ is sufficiently large. It is easily seen that there exist universal constants $(c_1,c_2,N) \in (0,+\infty)^3$ and a non-decreasing function $f:~ [c_2,+\infty) \to (0,+\infty)$ with $\lim_{x \to \infty} f(x) = \infty$, such that
	\begin{align}
	&\PP \bigg( \sup_{ \bgamma_1 \neq \bgamma_2 } \frac{ \| \nabla \hat{L}_\lambda (\bgamma_1) - \nabla \hat{L}_\lambda (\bgamma_2) \|_2 }{ \| \bgamma_1 - \bgamma_2 \|_2 } \geq t  \bigg) \leq c_1 e^{-n f(t)} , \label{lemma-landscape-sample-1} \\
	&\PP \bigg( \sup_{ \bgamma_1 \neq \bgamma_2 } \frac{ | \bu^{\top} [ \nabla^2 \hat L_\lambda ( \bgamma_1 ) - \nabla^2 \hat L_\lambda ( \bgamma_2 ) ] \bu | }{ \| \bgamma_1 - \bgamma_2 \|_2 } \geq t \bigg)
	\leq c_1 e^{- n^{1/3} f(t) } , \label{lemma-landscape-sample-2} \\
	&\PP \bigg( \sup_{ \bgamma \neq \bgamma  } \frac{ \| \nabla^2 \hat{L}_\lambda (\bgamma_1) - \nabla^2 \hat{L}_\lambda (\bgamma_2) \|_2 }{ \| \bgamma_1 - \bgamma_2 \|_2 } \geq t \max\{ 1, d \log (n/d) / \sqrt{n} \} \bigg)
	\leq c_1 e^{- d \log (n/d) f(t) } = c_1 (d / n)^{d f(t)}, \label{lemma-landscape-sample-3}
	\end{align}
	as long as $n \geq N_1$ and $t \geq c_2$. We prove the first two inequalities in Lemma \ref{lemma-landscape-sample} by (\ref{lemma-landscape-sample-1}), (\ref{lemma-landscape-sample-3}) and choosing proper constants.
	
	Let
	\[
	X_n(\bgamma) = \bu^{\top} [ \nabla^2 \hat L_{\lambda} ( \bgamma ) - \nabla^2  L_{\lambda} ( \bgamma ) ] \bu = \bu^{\top} [ \nabla^2 \hat L ( \bgamma ) - \nabla^2  L ( \bgamma ) ] \bu,
	\]
	$\cS_n = B(\mathbf{0} , R)$ and $m = \log(n/d)$. We will invoke Theorem 1 in \cite{wang2019some} to control $\sup_{ \bgamma \in \cS_n} |X_n(\bgamma)|$ and prove the remaining claim. 
	\begin{enumerate}
		\item By definition, $X_n(\bgamma) = \frac{1}{n} \sum_{i=1}^{n} \{ (\bu^{\top} \bX_i)^2 f''(\bgamma^{\top} \bX_i) - \EE [(\bu^{\top} \bX_i)^2 f''(\bgamma^{\top} \bX_i)] \}$ and
		\begin{align*}
		\| (\bu^{\top} \bX_i)^2 f''(\bgamma^{\top} \bX_i) \|_{\psi_1} \leq F_2 \| (\bu^{\top} \bX_i)^2 \|_{\psi_1} \lesssim F_2 \| \bu^{\top} \bX_i \|_{\psi_2}^2 \lesssim 1.
		\end{align*}
		By the Bernstein-type inequality in Proposition 5.16 of \cite{Ver10}, there is a constant $c'$ such that
		\begin{align*}
		\PP ( |  X_n (\bgamma) | \geq t  )  \leq 2 e^{ - c' n [t^2 \wedge t ] } ,\qquad \forall t \geq 0,~ \bgamma \in \RR^d.
		\end{align*}
		When $t = s  \sqrt{ m d / n }$ for $s \geq 1$, we have $n t^2 = s^2 m d \geq s m d$. Since $n / d \geq e$, we have 
		\begin{align*}
		m = \log (n/d) = \log [1 + (n/d - 1)] \leq n/d - 1 \leq n/d,
		\end{align*}
		$n \geq md$ and $nt = s \sqrt{n m d} \geq s m d$. This gives
		\begin{align*}
		\PP ( |  X_n (\bgamma) | \geq s \sqrt{ m d / n } ) 
		\leq 2 e^{ - c' m ds }, \qquad \forall s \geq 1,~ \bgamma \in \RR^d.
		\end{align*}
		Hence $\{ X_n (\bgamma) \}_{ \bgamma \in \cS_n } = O_{\PP} ( \sqrt{ m d / n } ; ~ m d )$.
		\item Let $\varepsilon_n = 2R \sqrt{d / n}$. According to Lemma 5.2 in \cite{Ver10}, there exists an $\varepsilon_n$-net $\cN_n$ of $\cS_n$ with cardinality at most $(1 + 2 R / \varepsilon_n)^d$. Since $n/d \geq e$, $\log |\cN_n| = d \log (1 + \sqrt{n/d}) \lesssim d \log (n/d) = m d$.
		\item Define $M_n = \sup_{ \bgamma_1 \neq \bgamma_2  }  \{ | X_n(\bgamma_1) - X_n(\bgamma_2) | / \| \bgamma_1 - \bgamma_2 \|_2 \}$. Observe that by Lemma \ref{lemma-Hessian-Lip-0} and $\| {\bm{X}}_i \|_{\psi_2} \leq 1$,
		\begin{align*}
		\sup_{ \bgamma_1 \neq \bgamma_2 } \frac{ | \bu^{\top} [ \nabla^2 L\lambda ( \bgamma_1 ) - \nabla^2 L_\lambda ( \bgamma_2 ) ] \bu | }{ \| \bgamma_1 - \bgamma_2 \|_2 }
		&\leq 
		\sup_{ \bgamma_1 \neq \bgamma_2 } \frac{ \| \nabla^2 L ( \bgamma_1 ) - \nabla^2 L ( \bgamma_2 ) \|_2 }{ \| \bgamma_1 - \bgamma_2 \|_2 }\\
		&\leq F_3 \sup_{\bu \in \SSS^{d}}  \EE |\bu^{\top}{\bm{X}} |^3 \leq (\sqrt{3})^3 F_3\lesssim 1.
		\end{align*}
		From this and (\ref{lemma-landscape-sample-2}) we obtain that $M_n = O_{\PP} ( 1;~ n^{1/3})$. 
	\end{enumerate}
	
	Based on these, Theorem 1 \cite{wang2019some} implies that
	\begin{align*}
	\sup_{ \bgamma \in \cS_n} |X_n(\bgamma)| = O_{\PP} ( \sqrt{  m d / n } + \varepsilon_n ;~ m d \wedge n^{1/3} )
	= O_{\PP} ( \sqrt{  \log (n/d) d / n };~ d \log (n/d) \wedge n^{1/3} ).
	\end{align*}
	As a result, there exist absolute constants $(c_1', c_2', N_1') \in (0,+\infty)^3$ and a non-decreasing function $g:~ [c_2', +\infty) \to (0,+\infty)$ such that
	\begin{align*}
	\PP \bigg(
	\sup_{ \bgamma \in \cS_n} |X_n(\bgamma)| \geq t\sqrt{  \log (n/d) d / n }
	\bigg) 
	& \leq c_1' e^{ - (m d \wedge n^{1/3}) g(t) }
	\leq c_1' ( e^{ - md g(t) } + e^{ - n^{1/3} g(t) }) \\
	& \leq c_1' (d / n)^{d g(t) } + c_1' e^{ - n^{1/3} g(t) }, \qquad \forall n \geq N_1',~ t \geq c_2'.
	\end{align*}
	The proof is finished by taking $t = c_2'$ and re-naming some constants above.

\section{Proof of Corollary \ref{cor:approx_beta_dist1}} \label{appendix:approx_beta_dist1}

From Claim 1 in the second item of Theorem \ref{theorem-landscape-sample}, we know that $\Vert\nabla \hat{L}_1 (\bgamma)\Vert_{2}\leq\varepsilon$
implies $\mathrm{dist}(\bgamma, \{ \pm \bgamma^\star \} \cup S )<\delta$. On
the other side, since $\lambda_{\min}[\nabla^{2}\hat{L}_1 (\bgamma)]>-\eta$, we have $\bv^{\top}\nabla^{2}\hat{L}_1 (\bgamma)\bv>-\eta$
for any unit vector $\bv$. Then in view of Claim 2 of Theorem
\ref{theorem-landscape-sample}, we know that $\mathrm{dist}(\bgamma,S)>\delta$. Therefore
we arrive at $\mathrm{dist}(\bgamma, \{ \pm \bgamma^\star \} ) < \delta$. According to Theorem
\ref{thm:landscape-general}, $\nabla^{2}L_1 (\bm{\gamma}')\succeq\eta\bm{I}$ so long as $\mathrm{dist}(\bm{\gamma}',S_{1})\leq\delta$. This and $\nabla L_1 (\bm{\gamma}^{\star})=\bm{0}$ lead to
\begin{align}
\min_{s=\pm1}\left\Vert s\bgamma-\bm{\gamma}^{\star}\right\Vert _{2}&\leq\frac{1}{\eta}\left\Vert \nabla L_1\left(\bgamma\right)-\nabla L_1\left(\bm{\gamma}^{\star}\right)\right\Vert _{2}=\frac{1}{\eta}\left\Vert \nabla L_1\left(\bgamma\right)\right\Vert _{2} \notag\\
&\leq\frac{1}{\eta} \Vert \nabla \hat{L}_1 (\bgamma ) \Vert _{2}+\frac{1}{\eta} \Vert \nabla \hat{L}_1\left(\bgamma\right)-\nabla L_1 (\bgamma ) \Vert _{2}.
\label{cor:approx_beta_dist1-1}
\end{align}
All of these hold with probability exceeding $1-C_1(d/n)^{C_2d}-C_1\exp(-C_2n^{1/3})$.

The desired result is a product of (\ref{cor:approx_beta_dist1-1}) and Lemma \ref{lemma:gradient-concentration} below.
\begin{lemma}\label{lemma:gradient-concentration}
	For any constant $R>0$, there exists a constant $C>0$ such that when $n\geq Cd$ for all $n$, 
	\begin{equation}
	\sup_{\| \bgamma \|_2 \leq R }\bigl\Vert\nabla\hat{L}_1\left(\bm{\gamma}\right)-\nabla L_1\left(\bm{\gamma}\right)\bigr\Vert_{2}=O_{\mathbb{P}}\left(\sqrt{\frac{d}{n}\log\left(\frac{n}{d}\right)};~
	d\log\left(\frac{n}{d}\right)\right)
	\end{equation}
\end{lemma}
\begin{proof}
	See Appendix \ref{appendix:gradient-concentration}.
\end{proof}

\subsection{Proof of Lemma \ref{lemma:gradient-concentration}}\label{appendix:gradient-concentration}
Let $\bgamma = (\alpha, \bbeta)$, $\hat L(\bgamma) = \frac{1}{n} \sum_{i=1}^{n} f (\alpha + \bbeta^{\top} \bX_i)$, $L(\bgamma) = \EE f (\alpha + \bbeta^{\top} \bX)$, $\hat R(\bgamma) = \frac{1}{2} (\alpha + \bbeta^{\top} \hat\bmu_0)^2$ and $ R(\bgamma) = \frac{1}{2} (\alpha + \bbeta^{\top} \bmu_0)^2$.
Since $\vert f^\prime(0)\vert=0$, $\sup_{x\in\mathbb{R}}\vert f^{\prime\prime}(x)\vert=h^{\prime}(a)+(b-a)h^{\prime\prime}(a)\leq3a^{2}b\lesssim 1$ and $\Vert \bX_i\Vert_{\psi_2}\leq M\lesssim 1$, from Theorem 2 in \cite{wang2019some} we get
	\[
	\sup_{ \| \bgamma \|_2 \leq R }\bigl\Vert\nabla\hat{L}\left(\bm{\gamma}\right) - \nabla L\left(\bm{\gamma}\right)\bigr\Vert_{2}=O_{\mathbb{P}}\left(\sqrt{\frac{d}{n}\log\left(\frac{n}{d}\right)};~
	d\log\left(\frac{n}{d}\right)\right).
	\]
	Then it boils down to proving uniform convergence of $\| \nabla\hat{R}(\bgamma) - \nabla R(\bgamma) \|$. 
	Let $\bar\bX_i = (1, \bX_i)$, $\tilde\bmu_0 = (1, \frac{1}{n} \sum_{i=1}^{n} \bX_i )$ and $\bar\bmu_0 = (1, \bmu_0)$. By definition,
	\[
	\nabla\hat{R}\left(\bm{\gamma}\right)=\left(\bm{\gamma}^{\top}\tilde{\bmu}_0\right)\tilde{\bmu}_0 \qquad\text{and}\qquad\nabla R\left(\bm{\gamma}\right)=\left(\bm{\gamma}^{\top}\bar{\bmu}_0\right)\bar{\bmu}_0,
	\]
	Since $\Vert\bar{\bm{X}}_{i}-\bar{\bm{\mu}}_{0}\Vert_{\psi_{2}} \lesssim \Vert\bar{\bm{X}}_{i}\Vert_{\psi_{2}} \lesssim 1$, we know
	that $\Vert \tilde\bmu_0 - \bar{\bm{\mu}}_{0}\Vert_{\psi_{2}}\lesssim1/\sqrt{n}$.
	In view of Example 6 \cite{wang2019some} and $\| \bmu_0 \|_2 \lesssim 1$, we know that $\Vert \tilde\bmu_0 -\bm{\mu}_{0}\Vert_{2}=O_{\mathbb{P}}(\sqrt{d/n\log(n/d)};~ d\log(n/d))$
	and $\Vert \tilde\bmu_0 \Vert_{2}=O_{\mathbb{P}}(1;~ d\log(n/d))$. As a result,
	\begin{align*}
	\sup_{ \| \bgamma \|_2 \leq R }\big\Vert\nabla\hat{R}\left(\bm{\gamma}\right)-\nabla R\left(\bm{\gamma}\right)\big\Vert_2&\leq
	\sup_{ \| \bgamma \|_2 \leq R } \left\{ 
	\big\vert\bgamma^\top\left(
	\tilde{\bmu}_0-\bar{\bmu}_0
	\right)\big\vert\left\Vert \tilde{\bmu}_0 \right\Vert_2
	+\big\vert\bgamma^\top\bar{\bmu}_0\big\vert\left\Vert \tilde{\bmu}_0-\bar{\bmu}_0\right\Vert_2
	\right\}
	\\
	&\leq R \left\Vert \tilde{\bmu}_0 - \bar{\bmu}_0\right\Vert_2\left(\left\Vert \tilde{\bmu}_0 \right\Vert_2+\left\Vert\bar{\bmu}_0\right\Vert_2\right)\\
	&=O_{\mathbb{P}}\left(\sqrt{\frac{d}{n}\log\left(\frac{n}{d}\right)};~
	d\log\left(\frac{n}{d}\right)\right).
	\end{align*}
	
	\section{Proof of Theorem \ref{theorem:algorithmic}}\label{appendix:theorem-algorithmic}
	To prove Theorem \ref{theorem:algorithmic}, we invoke the convergence guarantees for perturbed gradiend descent in \cite{jin2017escape}.
	\begin{theorem}[Theorem 3 of \cite{jin2017escape}]\label{thm-jin}
		Assume that $ F (\cdot)$
		is $\ell$-smooth and $\rho$-Hessian Lipschitz. Then there exists
		an absolute constant $c_{\max}$ such that, for any
		$\delta_{\mathrm{pgd}}>0$, $\varepsilon_{\mathrm{pgd}}\leq\ell^{2}/\rho$,
		$\Delta_{\mathrm{pgd}}\geq F ( \bgamma_{\mathrm{pgd}})-\inf_{\bm{\gamma}\in\mathbb{R}^{d+1}} F (\bm{\gamma})$
		and constant $c_{\mathrm{pgd}}\leq c_{\max}$, with probability exceeding $1-\delta_{\mathrm{pgd}}$, Algorithm \ref{alg:PGD} terminates within 
		\[
		T \lesssim \frac{\ell\bigl[ F \left(  \bgamma_{\mathrm{pgd}} \right)-\inf_{\bm{\gamma}\in\mathbb{R}^{d+1}} F (\bm{\gamma})\bigr]}{\varepsilon_{\mathrm{pgd}}^{2}}\log^{4}\left(\frac{d\ell\Delta_{\mathrm{pgd}}}{\varepsilon_{\mathrm{pgd}}^{2}\delta_{\mathrm{pgd}}}\right)
		\]
		iterations and the output $\bm{\gamma}^{T}$ satisfies
		\[
		\bigl\Vert\nabla F \left(\bm{\gamma}^{T}\right)\bigr\Vert_{2}\leq\varepsilon_{\mathrm{pgd}}\qquad\text{and}\qquad\lambda_{\min}\bigl(\nabla^{2} F \left(\bm{\gamma}\right)\bigr)\geq-\sqrt{\rho\varepsilon_{\mathrm{pgd}}}.
		\]
		\end{theorem}
	
	Let $\cA$ denote this event where all of the geometric properties in Theorem \ref{theorem-landscape-sample} holds. When $\cA$ happens, $\hat{L}_{1}$ is $\ell$-smooth
	and $\rho$-Hessian Lipschitz with 
	\[
	\ell=M_{1}\qquad\text{and}\qquad\rho=M_{1}\left( 1\vee \frac{d \log (n/d) }{\sqrt{n}} \right ).
	\]
	Let $\bgamma_{\mathrm{pgd}} =\bm{0}$ and $\Delta_{\mathrm{pgd}} = 1/4$. Since $\inf_{\bgamma\in\mathbb{R}\times\mathbb{R}^d}\hat{L}_{1}\left(\bm{\gamma}\right) \geq 0$, we have
	\[
	\Delta_{\mathrm{pgd}}=\hat{L}_{1}\left(  \bgamma_{\mathrm{pgd}} \right) \geq \hat{L}_{1}\left(  \bgamma_{\mathrm{pgd}} \right) -\inf_{\bgamma\in\mathbb{R}\times\mathbb{R}^d}\hat{L}_{1}\left(\bm{\gamma}\right).
	\]
	In addition, we take $\delta^{\mathrm{pgd}}=n^{-11}$ and let 
	\[
	\varepsilon_{\mathrm{pgd}}=\sqrt{\frac{d}{n}\log\Big(\frac{n}{d}\Big)}\land\frac{\ell^{2}}{\rho}\land\frac{\eta^{2}}{\rho}\land\varepsilon.
	\]
	Here $\varepsilon$ and $\eta$ are the constants defined in Theorem
	\ref{theorem-landscape-sample}.
	
	Recall that $M_{1},\eta,\varepsilon\asymp 1$. Conditioned on the event $\cA$, Theorem \ref{thm-jin} asserts that with probability exceeding $1-n^{-10}$, Algorithm \ref{alg:PGD} with parameters $\bgamma_{\mathrm{pgd}} $, $\ell,\rho,\varepsilon_{\mathrm{pgd}},c_{\mathrm{pgd}},\delta_{\mathrm{pgd}}$, and $\Delta_{\mathrm{pgd}}$
	terminates within 
	\[
	T\lesssim \left(\frac{n}{d\log\left(n/d\right)}+\frac{d^2}{n}\log^2\Big(\frac{n}{d}\Big)\right)\log^{4}\left(nd\right)=\tilde{O}\left(\frac{n}{d}+\frac{d^2}{n}\right)
	\]
	iterations, and the output $\hat{\bm{\gamma}}$ satisfies
	\[
	\bigl\Vert\nabla\hat{L}_{1}\left(\hat{\bm{\gamma}}\right)\bigr\Vert_{2}\leq\varepsilon_{\mathrm{pgd}}\leq\sqrt{\frac{d}{n}\log\Big(\frac{n}{d}\Big)}\qquad\text{and}\qquad \lambda_{\min}\bigl(\nabla^{2}\hat{L}_{1}(\hat{\bm{\gamma}})\bigr)\geq-\sqrt{\rho\varepsilon_{\mathrm{pgd}}}\geq-\eta.
	\]
	Then the desired result follows directly from $\PP (\cA) \geq 1 - C_1 (d/n)^{C_2d} -  C_1 \exp(-C_2n^{1/3})$ in Theorem \ref{theorem-landscape-sample}.

%% file: appendix_misclassification.tex
\section{Proof of Corollary \ref{cor:misclassification}}\label{appendix:misclassification}
Throughout the proof we suppose that the high-probability event
	\[
\min_{s=\pm1}\bigl\Vert s \hat{\bm{\gamma}} -c\bgamma^{\mathrm{Bayes}}\bigr\Vert_2\lesssim \sqrt{\frac{d}{n} \log \left( \frac{n}{d} \right)}
\]
in Theorem \ref{thm:main} happens. Write $\hat\bgamma=(\hat\alpha,\hat\bbeta)$ and $\bgamma^{\star} =(\alpha^{\star},\bbeta^{\star})= c\bm{\gamma}^{\mathrm{Bayes}}$. Without loss of generality, assume that $\bm{\mu}_{0}=\bm{0}$,
$\bm{\Sigma}=\bm{I}_{d}$, $\arg\min_{s=\pm1}\Vert s\hat{\bm{\gamma}}-\bm{\gamma}^{\star}\Vert_{2}=1$ and $\hat\bbeta^{\top}\bm{\mu} > 0$. Let $F$ be the cumulative distribution
function of $Z = \be_1^{\top} \bZ$.

For any $\bgamma = (\alpha, \bbeta)$ with $\bbeta^{\top}\bm{\mu} > 0$, we use $\bX = \bmu Y + \bZ$ and the symmetry of $\bZ$ to derive that
\begin{align*}
\cR \left(\bgamma\right) & =\frac{1}{2}\mathbb{P}\left(\alpha+\bm{\beta}^{\top}\left(\bm{\mu}+\bm{Z}\right)<0\right)+\frac{1}{2}\mathbb{P}\left(\alpha+\bm{\beta}^{\top}\left(-\bm{\mu}+\bm{Z}\right)>0\right)\\
 & =\frac{1}{2}\mathbb{P}\left(\bm{\beta}^{\top}\bm{Z}<-\alpha-\bm{\beta}^{\top}\bm{\mu}\right)+\frac{1}{2}\mathbb{P}\left(\bm{\beta}^{\top}\bm{Z}>-\alpha+\bm{\beta}^{\top}\bm{\mu}\right)\\
 & =\frac{1}{2}F\left(-\alpha/\left\Vert \bm{\beta}\right\Vert _{2}-\left(\bm{\beta}/\left\Vert \bm{\beta}\right\Vert _{2}\right)^{\top}\bm{\mu}\right)+\frac{1}{2}F\left(\alpha/\left\Vert \bm{\beta}\right\Vert _{2}-\left(\bm{\beta}/\left\Vert \bm{\beta}\right\Vert _{2}\right)^{\top}\bm{\mu}\right).
\end{align*}

Define $\bm{\gamma}_{0}=(\alpha_{0},\bm{\beta}_{0})$ with $\alpha_{0}=\hat{\alpha}/\Vert\hat{\bm{\beta}}\Vert_{2}$
and $\bm{\beta}_{0}=\hat{\bm{\beta}}/\Vert\hat{\bm{\beta}}\Vert_{2}$; $\bm{\gamma}_{1}=(\alpha_{1},\bm{\beta}_{1})$ with $\alpha_{1}=0$
and $\bm{\beta}_{1}=\bm{\mu}/\Vert\bm{\mu}\Vert_{2}$.
Recall that $\bm{\gamma}^{\mathrm{Bayes}}=c(0,\bm{\mu})$ for some constant
$c>0$. We have
\[
\cR \left(\hat{\bm{\gamma}}\right)-\cR \left(\bm{\gamma}^{\mathrm{Bayes}}\right)=\underbrace{\frac{1}{2}F\big(-\alpha_{0}-\bm{\beta}_{0}^{\top}\bm{\mu}\big)-\frac{1}{2}F\big(-\alpha_{1}-\bm{\beta}_{1}^{\top}\bm{\mu}\big)}_{ E_1}+\underbrace{\frac{1}{2}F\big(\alpha_{0}-\bm{\beta}_{0}^{\top}\bm{\mu}\big)-\frac{1}{2}F\big(\alpha_{1}-\bm{\beta}_{1}^{\top}\bm{\mu}\big)}_{E_2}.
\]
Using Taylor's Theorem, $\| p' \|_{\infty} \lesssim 1$ and $\| \bmu \|_2 \lesssim 1$, one can arrive at
\begin{align*}
\left|E_1-p\big(-\alpha_{1}-\bm{\beta}_{1}^{\top}\bm{\mu}\big)\big(\alpha_{1}-\alpha_{0}+\left(\bm{\beta}_{1}-\bm{\beta}_{0}\right)^{\top}\bm{\mu}\big)\right| & \lesssim \left\Vert \bm{\gamma}_{0}-\bm{\gamma}_{1}\right\Vert _{2}^{2},\\
\left|E_2-p\big(\alpha_{1}-\bm{\beta}_{1}^{\top}\bm{\mu}\big)\big(\alpha_{0}-\alpha_{1}+\left(\bm{\beta}_{1}-\bm{\beta}_{0}\right)^{\top}\bm{\mu}\big)\right| & \lesssim \left\Vert \bm{\gamma}_{0}-\bm{\gamma}_{1}\right\Vert_{2}^{2},
\end{align*}

From $\alpha_{1}=0$, $\bm{\beta}_{1}=\bm{\mu}/\Vert\bm{\mu}\Vert_{2}$ and $\| p \|_{\infty} \lesssim 1$
we obtain that
\begin{align*}
\cR \left(\hat{\bm{\gamma}}\right) - \cR \left(\bm{\gamma}^{\mathrm{Bayes}}\right) & \lesssim 
|
p(-\bm{\beta}_{1}^{\top}\bm{\mu})[-\alpha_{0}+\left(\bm{\beta}_{1}-\bm{\beta}_{0}\right)^{\top}\bm{\mu}]+p(-\bm{\beta}_{1}^{\top}\bm{\mu})[\alpha_{0}+\left(\bm{\beta}_{1}-\bm{\beta}_{0}\right)^{\top}\bm{\mu}]
|
+\left\Vert \bm{\gamma}_{0}-\bm{\gamma}_{1}\right\Vert _{2}^{2}\nonumber \\
 & \lesssim  |
 \left(\bm{\beta}_{1}-\bm{\beta}_{0}\right)^{\top}\bm{\beta}_{1}
 |
 + \left\Vert \bm{\gamma}_{0}-\bm{\gamma}_{1}\right\Vert _{2}^{2}.
\end{align*}
Since $\bm{\beta}_{0}$ and $\bm{\beta}_{1}$ are unit vectors,
\begin{align}
&	\| \bbeta_1 - \bbeta_0 \|_2^2 = \| \bbeta_0 \|_2^2 - 2 \bbeta_0^{\top} \bbeta_1 + \| \bbeta_1 \|_2^2 = 2 ( 1 - \bbeta_0^{\top} \bbeta_1 ) = 2 (\bbeta_1 - \bbeta_0)^{\top} \bbeta_1 , \notag\\
&\cR \left(\hat{\bm{\gamma}}\right) - \cR \left(\bm{\gamma}^{\mathrm{Bayes}}\right) \lesssim 
 \Vert \bm{\beta}_{1}-\bm{\beta}_{0}\Vert _{2}^{2}+ \Vert \bm{\gamma}_{0}-\bm{\gamma}_{1}\Vert _{2}^{2}
 \lesssim \Vert \bm{\gamma}_{0}-\bm{\gamma}_{1}\Vert _{2}^{2}
 .\label{eq:E-intermediate}
\end{align}

Note that $\Vert\hat{\bm{\beta}}-\bm{\beta}^{\star}\Vert_{2}\leq\Vert\hat{\bm{\gamma}}-\bm{\gamma}^{\star}\Vert_{2}\lesssim\sqrt{d/n\log(n/d)}$
and $\Vert\bm{\beta}^{\star}\Vert_{2}\asymp1$. When $n/d$ is sufficiently
large, we have $\Vert\hat{\bm{\beta}}\Vert_{2}\asymp1$ and 
\begin{align*}
\left\Vert \bm{\beta}_{1}-\bm{\beta}_{0}\right\Vert _{2} & =\bigl\Vert\hat{\bm{\beta}}/\Vert\hat{\bm{\beta}}\Vert_{2}-\bm{\beta}^{\star}/\left\Vert \bm{\beta}^{\star}\right\Vert _{2}\bigr\Vert_{2}\lesssim\bigl\Vert\left\Vert \bm{\beta}^{\star}\right\Vert _{2}\hat{\bm{\beta}}-\Vert\hat{\bm{\beta}}\Vert_{2}\bm{\beta}^{\star}\bigr\Vert_{2}\\
 & \leq\bigl|\left\Vert \bm{\beta}^{\star}\right\Vert _{2}-\Vert\hat{\bm{\beta}}\Vert_{2}\bigr|\bigl\Vert\hat{\bm{\beta}}\bigr\Vert_{2}+\Vert\hat{\bm{\beta}}\Vert_{2}\bigl\Vert\hat{\bm{\beta}}-\bm{\beta}^{\star}\bigr\Vert_{2}\lesssim\bigl\Vert\hat{\bm{\beta}}-\bm{\beta}^{\star}\bigr\Vert_{2}.
\end{align*}
In addition, we also have $\vert\alpha_{0}-\alpha_{1}\vert=\vert\alpha_{0}\vert=\vert\hat{\alpha}\vert/\Vert\hat{\bm{\beta}}\Vert_{2}\lesssim\vert\hat{\alpha}\vert=\vert\hat{\alpha}-\alpha^{\star}\vert$.
As a result, $\Vert\bm{\gamma}_{0}-\bm{\gamma}_{1}\Vert_{2}\lesssim\vert\hat{\alpha}-\alpha^{\star}\vert+\Vert\bm{\beta}_{1}-\bm{\beta}_{0}\Vert_{2}\lesssim\Vert\hat{\bm{\gamma}}-\bm{\gamma}^{\star}\Vert_{2}$.
Plugging these bounds into (\ref{eq:E-intermediate}), we get
\[
\cR \left(\hat{\bm{\gamma}}\right)-\cR \left(\bm{\gamma}^{\star}\right)\lesssim\bigl\Vert\hat{\bm{\gamma}}-\bm{\gamma}^{\star}\bigr\Vert_{2}^{2}\lesssim\frac{d}{n}\log \left(\frac{n}{d}\right).
\]

%% file: appendix_technical.tex
\section{Technical lemmas}

\begin{lemma}\label{lemma-cubic}
	Let $\bm{X}$ be a random vector in $\RR^{d+1}$ with $\EE \| \bm{X} \|_2^3 < \infty$. Then
	\begin{align*}
	&\sup_{\bu, \bv \in \SSS^{d}} \EE ( |\bu^{\top} \bm{X} |^2 |\bv^{\top} \bm{X} | )
	= \sup_{\bu \in \SSS^{d}} \EE |\bu^{\top} \bm{X} |^3.
	\end{align*}
\end{lemma}

\begin{proof}
	It is easily seen that $\sup_{\bu, \bv \in \SSS^{d}}\EE ( |\bu^{\top} \bm{X} |^2 |\bv^{\top} \bm{X} | )
	\geq \sup_{\bu \in \SSS^{d}} \EE |\bu^{\top} \bm{X} |^3$. To prove the other direction, we first use Cauchy-Schwarz inequality to get
	\begin{align*}
	\EE ( |\bu^{\top} \bm{X} |^2 |\bv^{\top} \bm{X} | )
	& = \EE [ |\bu^{\top} \bm{X} |^{3/2} (  |\bu^{\top} \bm{X} |^{1/2} |\bv^{\top} \bm{X} | ) ]
	\leq 
	\EE^{1/2} |\bu^{\top} \bm{X} |^{3} \cdot 
	\EE^{1/2} ( |\bu^{\top} \bm{X} |\cdot |\bv^{\top} \bm{X} |^2 ).
	\end{align*}
	By taking suprema we prove the claim.
\end{proof}

\begin{lemma}\label{lemma-Hessian-Lip-0}
	Let $\bm{X}$ be a random vector in $\RR^{d+1}$ and $f \in C^2(\RR)$. Suppose that $\EE \| \bm{X} \|_2^3 < \infty$, $\sup_{x \in \RR} |f''(x)| = F_2 < \infty$ and $f''$ is $F_3$-Lipschitz. Define $\bar\bmu = \EE \bm{X}$.
	Then
	\[
	L_{\lambda} (\bgamma) = \EE f(\bgamma^{\top} \bm{X}) + \lambda (\bgamma^{\top}  \bar\bmu )^2 / 2
	\]
	exists for all $\bgamma \in \RR^{d+1}$ and $\lambda \geq 0$, and
	\begin{align*}
	&\sup_{ \bgamma_1 \neq \bgamma_2  } \frac{ \| \nabla L_{\lambda} ( \bgamma_1 ) - \nabla L_{\lambda} ( \bgamma_2 ) \|_2 }{ \| \bgamma_1 - \bgamma_2 \|_2 } \leq F_2 \sup_{\bu \in \SSS^{d}}  \EE | \bu^{\top} \bm{X} |^2  + \lambda \| \bar\bmu \|_2^2, \\
	&\sup_{ \bgamma_1 \neq \bgamma_2  } \frac{ | \bu^{\top} [ \nabla^2 L_{\lambda} ( \bgamma_1 ) - \nabla^2 L_{\lambda} ( \bgamma_2 ) ] \bu | }{ \| \bgamma_1 - \bgamma_2 \|_2 } \leq F_3 \sup_{\bv \in \SSS^{d}}  \EE [(\bu^{\top} \bm{X})^2 |\bv^{\top} \bm{X} | ], \qquad \forall \bu \in \SSS^{d-1}, \\
	&\sup_{ \bgamma_1 \neq \bgamma_2  } \frac{ \| \nabla^2 L_{\lambda} ( \bgamma_1 ) - \nabla^2 L_{\lambda} ( \bgamma_2 ) \|_2 }{ \| \bgamma_1 - \bgamma_2 \|_2 } \leq F_3 \sup_{\bu \in \SSS^{d}}  \EE |\bu^{\top} \bm{X} |^3.
	\end{align*}
	In addition, if there exist nonnegative numbers $a,b$ and $c$ such that $\inf_{ x \in \RR } x f'(x) \geq - b$ and $\inf_{ |x| \geq a } f'(x) \sgn(x) \geq c$, then
	\begin{align*}
	\| \nabla L_{\lambda} (\bgamma) \|_2 \geq c \inf_{\bu \in \SSS^{d}} \EE | \bu^{\top}\bm{X} | - \frac{ac+b}{\| \bgamma \|_2}, \qquad \forall \bgamma \neq \mathbf{0}.
	\end{align*}

\end{lemma}
\begin{proof}
	Let $L(\bgamma) = \EE f(\bgamma^{\top} \bm{X})$ and $R(\bgamma) = (\bgamma^{\top}  \bar\bmu)^2 / 2$. 
	Since $L_{\lambda} = L + \lambda R$, $\nabla^2 L(\bgamma) =  \EE [ \bm{X} \bm{X}^{\top} f''(\bgamma^{\top} \bm{X}) ]$ and $\nabla^2 R (\bgamma) =  \bar\bmu \bar\bmu^{\top}$,
	\begin{align*}
	\sup_{ \bgamma_1 \neq \bgamma_2  } \frac{ \| \nabla L_{\lambda}( \bgamma_1 ) - \nabla L_{\lambda}( \bgamma_2 ) \|_2 }{ \| \bgamma_1 - \bgamma_2 \|_2 }
	&= \sup_{ \bgamma \in \RR^{d+1} } \| \nabla^2 L_{\lambda}(\bgamma) \|_2= \sup_{ \bgamma \in \RR^{d+1} } \sup_{\bu \in \SSS^{d}} \bu^{\top}  \nabla^2 L_{\lambda}(\bgamma) \bu\\  
	&\leq 
	F_2  \sup_{\bu \in \SSS^{d}}  \EE ( \bu^{\top} \bm{X} )^2 + \lambda \| \bar\bmu \|_2^2.
	\end{align*}
	For any $\bu \in \SSS^{d}$,
	\begin{align*}
	| \bu^{\top} [ \nabla^2 L_{\lambda}( \bgamma_1 ) - \nabla^2 L_{\lambda}( \bgamma_2) ] \bu |
	&= \left|  \EE [ (\bu^{\top} \bm{X})^2 f''(\bgamma_1^{\top}\bm{X})] - \EE [ (\bu^{\top} \bm{X})^2 f''(\bgamma_2^{\top}\bm{X}) ] \right| \\
	&  \leq \EE [ (\bu^{\top} \bm{X})^2 |f''(\bgamma_1^{\top}\bm{X}) - f''(\bgamma_2^{\top}\bm{X})| ]\\
	& \leq F_3 \EE [ (\bu^{\top} \bm{X})^2 | (\bgamma_1 - \bgamma_2 )^{\top}\bm{X}| ]\\ 
	&\leq F_3 \| \bgamma_1 - \bgamma_2 \|_2 \sup_{\bv \in \SSS^{d}} \EE [ (\bu^{\top} \bm{X})^2 | \bv^{\top}\bm{X}| ].
	\end{align*}
	As a result,
	\begin{align*}
	\sup_{ \bgamma_1 \neq \bgamma_2  } \frac{ \| \nabla^2 L_{\lambda} ( \bgamma_1 ) - \nabla^2 L_{\lambda} ( \bgamma_2 ) \|_2 }{ \| \bgamma_1 - \bgamma_2 \|_2 }
	& = \sup_{\bgamma_1 \neq \bgamma_2} \frac{ \sup_{\bu \in \SSS^{d}} | \bu^{\top} [ \nabla^2 L_{\lambda}( \bgamma_1 ) - \nabla^2 L_{\lambda}( \bgamma_2 ) ] \bu | }{
		\| \bgamma_1 - \bgamma_2\|_2
	}\\
	&= \sup_{\bu \in \SSS^{d}}  \sup_{\bgamma_1 \neq \bgamma_2} \frac{ | \bu^{\top} [ \nabla^2 L_{\lambda}( \bgamma_1 ) - \nabla^2 L_{\lambda}( \bgamma_2 ) ] \bu | }{
		\| \bgamma_1 - \bgamma_2\|_2
	} 
	\\ &
	\leq \sup_{\bu \in \SSS^{d}} \{ F_3 \sup_{\bv \in \SSS^{d}} \EE [ (\bu^{\top} \bm{X})^2 | \bv^{\top}\bm{X}| ] \}
	= F_3 \sup_{\bu \in \SSS^{d}} \EE |\bu^{\top} \bm{X} |^3,
	\end{align*}
	where the last equality follows from Lemma \ref{lemma-cubic}.
	
	We finally come to the lower bound on $\| \nabla L_\lambda(\bgamma) \|_2$. Note that $\| \nabla L_\lambda (\bgamma) \|_2 \| \bgamma \|_2 \geq  \langle \bgamma , \nabla L_\lambda (\bgamma) \rangle$, $\nabla L(\bgamma) = \EE [ \bm{X}  f'(\bm{X}^{\top} \bgamma) ]$ and $\nabla R(\bgamma)=(\bgamma^\top\bar{\bmu})\bar{\bmu}$. The condition $\inf_{ |x| \geq a } f'(x) \sgn(x) \geq c$ implies that $xf'(x) \geq c |x|$ when $|x| \geq a$. By this and $\inf_{ x \in \RR } x f'(x) \geq - b$,
	\begin{align*}
	\langle \bgamma , \nabla L (\bgamma) \rangle 
	& = \EE [ \bm{X}^{\top} \bgamma f'(\bm{X}^{\top} \bgamma) ]
	= \EE [ \bm{X}^{\top} \bgamma f'(\bm{X}^{\top} \bgamma) \mathbf{1}_{ \{ |\bm{X}^{\top} \bgamma | \geq a \} } ] + \EE [ \bm{X}^{\top} \bgamma f'(\bm{X}^{\top} \bgamma) \mathbf{1}_{ \{ |\bm{X}^{\top} \bgamma | < a \} } ] \notag \\
	&\geq c \EE ( |\bm{X}^{\top} \bgamma | \mathbf{1}_{ \{ |\bm{X}^{\top} \bgamma | \geq a \} } ) - b
	= c \EE |\bm{X}^{\top} \bgamma | - c \EE ( |\bm{X}^{\top} \bgamma | \mathbf{1}_{ \{ |\bm{X}^{\top} \bgamma | < a \} } ) - b \notag\\
	&\geq c \EE |\bm{X}^{\top} \bgamma | - (ac+b)
	\geq \| \bgamma \|_2 c \inf_{\bu \in \SSS^{d}}  \EE | \bu^{\top}\bm{X} | - (ac+b).
	\end{align*}
	In addition, we also have $\langle\bgamma,\nabla R(\bgamma)\rangle=(\bgamma^\top\bar{\bmu})^2\geq0$. Then the lower bound directly follows.
\end{proof}

\begin{lemma}\label{lemma-L1-lower}
	There exists a continuous function $\varphi:~ (0,+\infty)^2 \to (0,+\infty)$ that is non-increasing in the first argument and non-decreasing in the second argument, such that for any nonzero sub-Gaussian random variable $X$, $\EE |X| \geq \varphi ( \| X \|_{\psi_2}, \EE X^2 ) $.
\end{lemma}
\begin{proof}
	For any $t > 0$,
	\begin{align*}
	\EE |X| \geq \EE (|X| \mathbf{1}_{ \{ |X| \leq t \} } ) \leq t^{-1} \EE (X^2 \mathbf{1}_{ \{ |X| \leq t \} } )
	= t^{-1} [ \EE X^2 - \EE (X^2 \mathbf{1}_{ \{ |X| > t \} } ) ].
	\end{align*}
	By Cauchy-Schwarz inequality and the sub-Gaussian property \citep{Ver10}, there exist constants $C_1,C_2>0$ such that
	\begin{align*}
	\EE (X^2 \mathbf{1}_{ \{ |X| > t \} } ) \leq \EE^{1/2} X^4 \cdot \PP^{1/2} ( |X| > t )
	\leq C_1 \| X \|_{\psi_2}^2 e^{ - C_2 t^2 / \| X \|_{\psi_2}^2 }.
	\end{align*}
	By taking $\varphi ( \| X \|_{\psi_2}, \EE X^2 ) = \sup_{t > 0} t^{-1} ( \EE X^2 - C_1 \| X \|_{\psi_2}^2 e^{ - C_2 t^2 / \| X \|_{\psi_2}^2 } ) $ we finish the proof, as the required monotonicity is obvious.
\end{proof}

\begin{lemma}\label{lemma-power-truncation}
	Let $\{ X_{ni} \}_{n\geq 1, i \in [n]}$ be an array of random variables where for any $n$, $\{ X_{ni} \}_{i=1}^n$ are i.i.d.~sub-Gaussian random variables with $\| X_{n1} \|_{\psi_2} \leq 1$. Fix some constant $a \geq 2$, define $S_{n} = \frac{1}{n} \sum_{i=1}^{n} |X_{ni}|^a$ and let $\{ r_n \}_{n=1}^{\infty}$ be a deterministic sequence satisfying $\log n \leq r_n \leq n$. We have
	\begin{align*}
	& S_n - \EE |X_{n1}|^a = O_{\PP} (  r_n^{(a-1)/2} / \sqrt{n}  ;~ r_n ), \\
	& S_n = O_{\PP} ( \max\{ 1 ,  r_n^{(a-1)/2} / \sqrt{n} \} ;~ r_n ).
	\end{align*}
\end{lemma}
\begin{proof}
	Define $R_{nt} = t \sqrt{r_n}$ and $S_{nt} = \frac{1}{n} \sum_{i=1}^{n} |X_{ni}|^a \mathbf{1}_{ \{ |X_{ni}| \leq R_{nt} \} }$ for $n,t \geq 1$. For any $p \geq 1$, we have $2p \geq 2 > 1$ and $(2p)^{-1/2} \EE^{1/(2p)} |X_{ni}|^{2p} \leq \| X_{ni} \|_{\psi_2} \leq 1$. Hence
	\begin{align*}
	\EE (|X_{ni}|^a \mathbf{1}_{ \{ |X_{ni}| \leq R_{nt} \} })^p & = \EE ( |X_{ni}|^{ap} \mathbf{1}_{ \{ |X_{ni}| \leq R_{nt} \} } ) =  \EE ( |X_{ni}|^{2p} |X_{ni}|^{(a-2)p} \mathbf{1}_{ \{ |X_{ni}| \leq R_{nt} \} } )
	\\ &
	\leq \EE |X_{ni}|^{2p} R_{nt}^{(a-2)p}
	\leq  [ (2p)^{1/2}  \| X_{ni} \|_{\psi_2} ]^{2p} R_{nt}^{(a-2)p} \leq ( 2 p  R_{nt}^{a-2} )^p
	\end{align*}
	and $ \| |X_{ni}|^a \mathbf{1}_{ \{ |X_{ni}| \leq R_{nt} \} } \|_{\psi_1} \leq 2 R_{nt}^{a-2}$.
	By the Bernstein-type inequality in Proposition 5.16 of \cite{Ver10}, there exists a constant $c$ such that
	\begin{align}
	\PP ( | S_{nt} - \EE S_{nt} | \geq s ) \leq 2 \exp\bigg[ - c n \bigg( \frac{s^2}{ R_{nt}^{2(a-2)} } \wedge \frac{s}{ R_{nt}^{a-2} } \bigg) \bigg]
	,\qquad \forall t \geq 0,~ s \geq 0.
	\label{eqn-lemma-power-truncation}
	\end{align}
	
	Take $t \geq 1$ and $s = t^{a-1} r_n^{(a-1)/2} / \sqrt{n}$. We have
	\begin{align*}
	& \frac{s}{ R_{nt}^{a-2} } = \frac{ t^{a-1} r_n^{(a-1)/2} / \sqrt{n} }{ t^{a-2} r_n^{(a-2) / 2} } =  t \sqrt{r_n / n} , \\
	&  \frac{s^2}{ R_{nt}^{2(a-2)} } \wedge \frac{s}{ R_{nt}^{a-2} } = \frac{t^2 r_n}{n} \wedge \frac{t \sqrt{r_n}}{\sqrt{n}} \geq \frac{t r_n}{n},
	\end{align*}
	where the last inequality is due to $r_n / n \leq 1 \leq t$. By (\ref{eqn-lemma-power-truncation}),
	\begin{align}
	\PP ( | S_{nt} - \EE S_{nt} | \geq t^{a-1} r_n^{(a-1)/2} / \sqrt{n} ) \leq 2 e^{-c r_n t},\qquad \forall t \geq 1.
	\label{eqn-lemma-power-truncation-1}
	\end{align}
	By Cauchy-Schwarz inequality and $\| X_{n1} \|_{\psi_2} \leq 1$, there exist $C_1, C_2> 0$ such that
	\begin{align*}
	0 \leq \EE S_{n} - \EE S_{nt} = \EE ( |X_{n1}|^a \mathbf{1}_{ \{ |X_{n1}| > t \sqrt{r_n} \} } )
	\leq \EE^{1/2} |X_{n1}|^{2a} \cdot \PP^{1/2} ( |X_{n1}| > t \sqrt{r_n})
	\leq C_1 e^{-C_2 t^2 r_n}
	\end{align*}
	holds for all $t \geq 0$. Since $r_n \geq \log n$, there exists a constant $C>0$ such that $C_1 e^{-C_2 t^2 r_n} \leq t^{a-1} r_n^{(a-1)/2} / \sqrt{n} $ as long as $t \geq C$. Hence (\ref{eqn-lemma-power-truncation-1}) forces
	\begin{align*}
	\PP ( | S_{nt} - \EE S_{n} | \geq 2 t^{a-1} r_n^{(a-1)/2} / \sqrt{n} )
	& \leq \PP ( | S_{nt} - \EE S_{nt} | + |\EE S_{nt} - \EE S_n| \geq  2 t^{a-1} r_n^{(a-1)/2} / \sqrt{n} ) \\
	& \leq \PP ( | S_{nt} - \EE S_{nt} | \geq t^{a-1} r_n^{(a-1)/2} / \sqrt{n} ) \leq 2 e^{-c r_n t},\qquad \forall t \geq C.
	\end{align*}
	
	Note that
	\begin{align}
	&  \PP ( |S_n - \EE S_n| \geq 2 t^{a-1}  r_n^{(a-1)/2} / \sqrt{n}  )\\
	&\quad\leq \PP ( |S_n - \EE S_n| \geq 2 t^{a-1} r_n^{(a-1)/2} / \sqrt{n}  ,~S_n = S_{nt} ) + \PP (S_n \neq S_{nt}) \notag \\
	&\quad\leq \PP ( | S_{nt} - \EE S_n | \ge 2q t^{a-1}  r_n^{(a-1)/2} / \sqrt{n}  ) +  \PP (S_n \neq S_{nt}) \notag \\
	&\quad\leq 2 e^{-c r_n t} + \PP \left( \max_{i \in [n]} |X_{ni}| > t \sqrt{r_n} \right) , \qquad \forall t \geq C.
	\label{eqn-lemma-power-truncation-2}
	\end{align}
	Since $ \| X_{ni} \|_{\psi_2} \leq 1$, there exist constants $C_1' ,C_2' > 0$ such that
	\begin{align*}
	\PP ( |X_{ni} | \geq t ) \leq C_1' e^{-C_2' t^2},\qquad \forall n \geq 1,~ i \in [n], ~ t \geq 0.
	\end{align*}
	By union bounds,
	\begin{align*}
	\PP \left( \max_{i \in [n]} |X_{ni}| > t \sqrt{r_n} \right)
	\leq n C_1' e^{-C_2' t^2 r_n}
	= C_1' e^{\log n - C_2' t^2 r_n}, \qquad \forall t \geq 0.
	\end{align*}
	When $t \geq \sqrt{2 / C_2'}$, we have $C_2' t^2 r_n \geq 2 r_n \geq 2 \log n$ and thus $\log n - C_2' t^2 r_n \leq -C_2' t^2 r_n / 2$. Then (\ref{eqn-lemma-power-truncation-2}) leads to
	\begin{align*}
	&  \PP ( |S_n - \EE S_n| \geq 2 t^{a-1}  r_n^{(a-1)/2} / \sqrt{n}  ) \leq 2 e^{-c r_n t} + C_1' e^{-C_2' r_n t^2 / 2} , \qquad \forall t \geq C \vee \sqrt{2 / C_2'}.
	\end{align*}
	This shows $S_n - \EE |X_{n1}|^a = S_n - \EE S_n = O_{\PP} ( r_n^{(a-1)/2} / \sqrt{n} ; ~ r_n )$. The proof is finished by $\EE |X_{n1}|^a \lesssim 1$.
\end{proof}

\begin{lemma}\label{lemma-cubic-sup}
	Suppose that $\{  \bm{X}_i \}_{i=1}^n \subseteq \RR^{d+1}$ are independent random vectors, $\max_{i \in [n]} \| \bm{X}_i \|_{\psi_2} \leq 1$ and $n \geq m d \geq \log n $ for some $m \geq 1$. We have
	\begin{align*}
	&\sup_{\bu \in \SSS^{d}} \frac{1}{n} \sum_{i=1}^{n} | \bu^{\top} \bm{X}_i |^2
	=  O_{\PP} (1;~n),  \\
	& \sup_{\bu \in \SSS^{d}}  \frac{1}{n} \sum_{i=1}^{n} (\bv^{\top} \bm{X}_i)^2 |\bu^{\top} \bm{X}_i| = O_{\PP} (1;~n^{1/3}), \qquad \forall \bv \in \SSS^{d}, \\
	&\sup_{\bu \in \SSS^{d}} \frac{1}{n} \sum_{i=1}^{n} |\bu^{\top} \bm{X}_i|^3 = O_{\PP} \left(  \max\{1,~ md/\sqrt{n} \};~ md\right). 
	\end{align*}
\end{lemma}

\begin{proof}
	From $2^{-1/2} \EE^{1/2} (\bu^{\top} \bm{X})^2 \leq \| \bu^{\top} \bm{X} \|_{\psi_2} \leq 1 $, $\forall \bu \in \SSS^{d}$ we get $\EE (\bm{X} \bm{X}^{\top}) \preceq 2 \bI$. Since $n \geq d+1$, Remark 5.40 in \cite{Ver10} asserts that
	\[
	\sup_{\bu \in \SSS^{d}} \frac{1}{n} \sum_{i=1}^{n} | \bu^{\top} \bm{X}_i |^2 = \bigg\|  \frac{1}{n} \sum_{i=1}^{n} \bm{X}_i \bm{X}_i^{\top} \bigg\|_2 \leq \bigg\|  \frac{1}{n} \sum_{i=1}^{n} \bm{X}_i \bm{X}_i^{\top} - \EE (\bm{X} \bm{X}^{\top}) \bigg\|_2 + \| \EE (\bm{X} \bm{X}^{\top}) \|_2 = O_{\PP} (1;~n).
	\]
	
	For any $\bu, \bv \in \SSS^{d}$, the Cauchy-Schwarz inequality forces
	\begin{align*}
	&\frac{1}{n} \sum_{i=1}^{n} (\bv^{\top} \bm{X}_i)^2 |\bu^{\top} \bm{X}_i| \leq \bigg(  \frac{1}{n} \sum_{i=1}^{n} (\bv^{\top} \bm{X}_i)^4 \bigg)^{1/2} \bigg(  \frac{1}{n} \sum_{i=1}^{n} (\bu^{\top} \bm{X}_i)^2 \bigg)^{1/2}, \\
	&\sup_{\bu \in \SSS^{d}}  \frac{1}{n} \sum_{i=1}^{n} (\bv^{\top} \bm{X}_i)^2 |\bu^{\top} \bm{X}_i| \leq 
	\bigg(  \frac{1}{n} \sum_{i=1}^{n} (\bv^{\top} \bm{X}_i)^4 \bigg)^{1/2} O_{\PP} (1;~n).
	\end{align*}
	Since $\{ \bv^{\top} \bm{X}_i \}_{i=1}^n$ are i.i.d.~sub-Gaussian random variables and $\|  \bv^{\top} \bm{X}_i  \|_{\psi_2} \leq 1$, Lemma \ref{lemma-power-truncation} with $a = 4$ and $r_n = n^{1/3}$ yields
	$ \frac{1}{n} \sum_{i=1}^{n} (\bv^{\top} \bm{X}_i)^4 = O_{\PP} ( 1 ;~n^{1/3} )$. Hence $\sup_{\bu \in \SSS^{d}}  \frac{1}{n} \sum_{i=1}^{n} (\bv^{\top} \bm{X}_i)^2 |\bu^{\top} \bm{X}_i| = O_{\PP} (1 ;~ n^{1/3} )$.

	To prove the last equation in Lemma \ref{lemma-cubic-sup}, define $\bZ_i = \bm{X}_i - \EE \bar{\bX}_i$. From $\| \bZ_i \|_{\psi_2} = \| \bm{X}_i - \EE\bm{X}_i \|_{\psi_2} \leq 2 \| \bm{X}_i \|_{\psi_2} \leq 2$ we get $\sup_{\bu \in \SSS^{d}} \frac{1}{n} \sum_{i=1}^{n} | \bu^{\top} \bZ_i |^2 = O_{\PP} ( 1 ;~n )$. For $\bu \in \SSS^{d}$,
	\begin{align*}
	|\bu^{\top} \bm{X}_i|^3 & = |\bu^{\top} \bZ_i |^3 + ( |\bu^{\top} \bm{X}_i| - |\bu^{\top} \bZ_i| ) (
	|\bu^{\top} \bm{X}_i|^2 +  |\bu^{\top} \bm{X}_i| \cdot |\bu^{\top} \bZ_i| + |\bu^{\top} \bZ_i|^2
	) \\
	&\leq  |\bu^{\top} \bZ_i |^3 +  |\bu^{\top} ( \bm{X}_i - \bZ_i ) | (
	|\bu^{\top} \bm{X}_i|^2 +  |\bu^{\top} \bm{X}_i| \cdot |\bu^{\top} \bZ_i| + |\bu^{\top} \bZ_i|^2
	)  \\
	&\leq  |\bu^{\top} \bZ_i |^3 +  |\bu^{\top} \EE\bar{\bX}_i |  \cdot \frac{3}{2} (
	|\bu^{\top} \bm{X}_i|^2 +  |\bu^{\top} \bZ_i|^2
	) 
	\leq |\bu^{\top} \bZ_i |^3 + \frac{3}{2} (
	|\bu^{\top} \bm{X}_i|^2 +  |\bu^{\top} \bZ_i|^2
	),
	\end{align*}
	where the last inequality is due to $ |\bu^{\top} \EE\bar{\bX}_i| \leq \| \EE\bar{\bX}_i \|_2 \leq \| \bm{X}_i \|_{\psi_2} \leq 1$. Hence
	\begin{align}
	\sup_{\bu \in \SSS^{d}} \frac{1}{n} \sum_{i=1}^{n} |\bu^{\top} \bm{X}_i|^3
	\leq \sup_{\bu \in \SSS^{d}} \frac{1}{n} \sum_{i=1}^{n} |\bu^{\top} \bZ_i|^3
	+ O_{\PP} ( 1 ; ~n).
	\label{eqn-lemma-cubic-sup-0}
	\end{align}
	
	Define $S(\bu) = \frac{1}{n} \sum_{i=1}^{n} |\bu^{\top} \bZ_i|^3$ for $\bu \in \SSS^{d}$. We will invoke Theorem 1 in \cite{wang2019some} to control $\sup_{\bu \in \SSS^{d}} S(\bu)$. 
	\begin{enumerate}
		\item For any $\bu \in \SSS^{d}$, $\{ \bu^{\top} \bZ_i \}_{i=1}^n$ are i.i.d.~and $\| \bu^{\top} \bZ_i \|_{\psi_2} \leq 1$. Lemma \ref{lemma-power-truncation} with $a = 3$ and $r_n = md$ yields
		\begin{align*}
		\{ S(\bu) \}_{\bu \in \SSS^{d}} = O_{\PP} ( \max\{ 1 , md/\sqrt{n} \} ;~ md ).
		\end{align*}
		\item According to Lemma 5.2 in \cite{Ver10}, for $\varepsilon = 1/6$ there exists an $\varepsilon$-net $\cN$ of $\SSS^{d}$ with cardinality at most $(1 + 2 / \varepsilon)^d = 13^d$. Hence $\log |\cN| \lesssim md$. 
		\item For any $x,y\in\RR$, we have $\left| |x| - |y| \right| \leq |x-y|$, $2|xy| \leq x^2 + y^2$ and
		\begin{align*}
		\left| |x|^3 - |y|^3 \right| \leq \left| |x| - |y| \right| (x^2 + |xy| + y^2) \leq \frac{3}{2} |x-y|(x^2 + y^2).
		\end{align*}
		Hence for any $\bu, \bv \in \SSS^{d}$,
		\begin{align*}
		| S(\bu) - S(\bv) | &
		\leq  \frac{1}{n} \sum_{i=1}^{n} \left| |\bu^{\top} \bZ_i|^3 - |\bv^{\top} \bZ_i|^3 \right| 
		\leq \frac{3}{2} \cdot \frac{1}{n} \sum_{i=1}^{n} |(\bu - \bv)^{\top} \bZ_i| (
		|\bu^{\top} \bZ_i|^2 + |\bv^{\top} \bZ_i|^2 )  \\
		&\leq 3  \| \bu - \bv \|_2 \sup_{\bw_1, \bw_2 \in \SSS^{d}} \frac{1}{n} \sum_{i=1}^{n}  |\bw_1^{\top} \bZ_i |\cdot |\bw_2^{\top} \bZ_i |^2 
		= \frac{1}{2 \varepsilon} \| \bu - \bv \|_2 \sup_{\bw \in \SSS^{d}} S(\bw).
		\end{align*}
		where the last inequality follows from $\varepsilon = 1/6$ and Lemma \ref{lemma-cubic}. 
	\end{enumerate}
	Theorem 1 in \cite{wang2019some} then asserts that $\sup_{\bu \in \SSS^{d}} S(\bu) = O_{\PP} ( \max\{ 1, md /\sqrt{n} \};~md)$. We finish the proof using (\ref{eqn-lemma-cubic-sup-0}).
\end{proof}